\documentclass[11pt]{article}

\usepackage{amsmath,amssymb,amsthm}
\usepackage{mathtools}

\usepackage[noend]{algorithmic}
\usepackage[ruled,vlined]{algorithm2e}
\usepackage{placeins}
\usepackage{setspace}
\usepackage{url}
\usepackage{makeidx}
\usepackage[shortlabels]{enumitem}

\usepackage{mathrsfs}
\DeclareSymbolFont{rsfs}{U}{rsfs}{m}{n}
\DeclareSymbolFontAlphabet{\mathscrsfs}{rsfs}

\def\E{{\mathbb E}}

\def\hE{\hat{\mathbb E}}
\def\hR{\hat{R}}

\def\bfone{{\boldsymbol 1}}

\def\Ball{{\sf B}}

\def\btheta{{\boldsymbol \theta}}
\def\id{{\boldsymbol I}}
\def\bg{{\boldsymbol g}}
\def\bm{{\boldsymbol m}}
\def\bx{{\boldsymbol x}}
\def\by{{\boldsymbol y}}
\def\bD{{\boldsymbol D}}
\def\bX{{\boldsymbol X}}
\def\bY{{\boldsymbol Y}}
\def\bW{{\boldsymbol W}}
\def\bA{{\boldsymbol A}}

\def\bB{{\boldsymbol B}}

\def\normal{{\sf N}}

\def\<{{\langle}}
\def\>{{\rangle}}

\def\sML{\mbox{\tiny\rm ML}}

\def\cL{{\mathcal L}}

\def\eps{{\varepsilon}}

\def\bfone{{\boldsymbol 1}}

\newcommand{\Tr}{\mathrm{Tr}}

\newcommand{\argmin}{\mbox{argmin}}

\newcommand{\beq}{\begin{equation}}
\newcommand{\eeq}{\end{equation}}

\newcommand{\bff}{\mathrm{\bf f}}

\newcommand{\be}{\mathrm{\bf e}}
\newcommand{\bh}{\mathrm{\bf h}}
\newcommand{\br}{\mathrm{\bf r}}

\newcommand{\bu}{\mathrm{\bf u}}
\newcommand{\bv}{\mathrm{\bf v}}

\newcommand{\bz}{\mathrm{\bf z}}
\newcommand{\bZ}{\mathrm{\bf Z}}
\newcommand{\bC}{\mathrm{\bf C}}

\newcommand{\bF}{\mathrm{\bf F}}
\newcommand{\bK}{\mathrm{\bf K}}

\newcommand{\bM}{\mathrm{\bf M}}
\newcommand{\bH}{\mathrm{\bf H}}

\newcommand{\bV}{\mathrm{\bf V}}
\newcommand{\bQ}{\mathrm{\bf Q}}
\newcommand{\bR}{\mathrm{\bf R}}

\newcommand{\bb}{\text{\boldmath $b$}}

\newcommand{\bbeta}{\text{\boldmath $\beta$}}
\newcommand{\bdelta}{\text{\boldmath $\delta$}}
\newcommand{\bDelta}{\text{\boldmath $\Delta$}}

\newcommand{\bOmega}{\text{\boldmath $\Omega$}}

\newcommand{\bPhi}{\text{\boldmath $\Phi$}}

\newcommand{\cA}{\mathcal{A}}

\def\normal{{\sf N}}
\def\id{{\boldsymbol I}}

\def\de{{\rm d}}

\def\eps{\varepsilon}

\def\bx{{\boldsymbol x}}

\def\bfone{{\boldsymbol 1}}

\usepackage[top=1in, bottom=1.25in, left=0.8in, right=0.8in]{geometry}
\usepackage[utf8]{inputenc}

\usepackage{graphicx}
\usepackage{multirow}
\usepackage{bbm}
\usepackage{booktabs}
\usepackage{bm}
\usepackage{etoc}
\usepackage{float}
\usepackage{epsfig}
\usepackage{caption}
\usepackage{subcaption}
\usepackage[dvipsnames,svgnames,table]{xcolor}
\usepackage{epstopdf}
\usepackage{color}
\usepackage{adjustbox}

\usepackage[
  pagebackref,
  colorlinks=true,
  pdfpagemode=UseNone,
  citecolor=OliveGreen,
  linkcolor=BrickRed,
  urlcolor=BrickRed,
  pdfstartview=FitH
]{hyperref}

\usepackage{xr}
\usepackage{cancel}
\usepackage{calc}
\usepackage{scalerel,stackengine}
\usepackage[numbers,sort&compress]{natbib}
\stackMath
\newcommand\reallywidehat[1]{%
\savestack{\tmpbox}{\stretchto{%
  \scaleto{%
    \scalerel*[\widthof{\ensuremath{#1}}]{\kern.1pt\mathchar"0362\kern.1pt}%
    {\rule{0ex}{\textheight}}%
  }{\textheight}%
}{2.4ex}}%
\stackon[-6.9pt]{#1}{\tmpbox}%
}

\numberwithin{equation}{section}

\newtheoremstyle{myexample}
    {\topsep}
    {\topsep}
    {\rm}
    {}
    {\bf}
    {.}
    {.5em}
    {}

\newtheoremstyle{myremark}
    {\topsep}
    {\topsep}
    {\rm}
    {}
    {\bf}
    {.}
    {.5em}
    {}

\newtheorem{claim}{Claim}[section]
\newtheorem{lemma}[claim]{Lemma}

\newtheorem{assumption}{Assumption}

\newtheorem{theorem}{Theorem}
\newtheorem{proposition}[claim]{Proposition}
\newtheorem{corollary}[claim]{Corollary}

\theoremstyle{myremark}
\newtheorem{remark}{Remark}[section]

\theoremstyle{myexample}
\newtheorem{example}[remark]{Example}

\definecolor{darkgreen}{rgb}{0.0, 0.5, 0.0}

\def\de{{\rm d}}
\def\prob{{\mathbb P}}
\def\<{\langle}
\def\>{\rangle}
\def\top{\intercal}
\def\cL{{\mathcal L}}

\def\hE{{\widehat{\mathbb E}}}
\def\bpsi{\boldsymbol{\psi}}
\def\hR{\hat{R}}
\def\bff{{\boldsymbol f}}
\def\bPhi{{\boldsymbol \Phi}}
\def\tbPsi{{\tilde{\boldsymbol \Psi}}}
\def\hbh{\hat{\boldsymbol h}}
\def\hbm{\hat{\boldsymbol m}}
\def\bA{\boldsymbol{A}}
\def\bB{\boldsymbol{B}}
\def\bC{\boldsymbol{C}}
\def\bD{\boldsymbol{D}}
\def\bF{\mathbf{F}}
\def\bG{\boldsymbol{G}}
\def\bH{\boldsymbol{H}}
\def\bJ{\boldsymbol{J}}
\def\bL{\boldsymbol{L}}
\def\bQ{\boldsymbol{Q}}
\def\bS{\boldsymbol{S}}
\def\bV{\boldsymbol{V}}
\def\bW{\boldsymbol{W}}
\def\bX{\boldsymbol{X}}
\def\bY{\boldsymbol{Y}}
\def\bZ{\boldsymbol{Z}}
\def\bb{\boldsymbol{b}}
\def\bg{\boldsymbol{g}}
\def\bh{\boldsymbol{h}}
\def\bm{\boldsymbol{m}}
\def\bu{\boldsymbol{u}}
\def\bv{\boldsymbol{v}}
\def\bx{\boldsymbol{x}}
\def\by{\boldsymbol{y}}
\def\bz{\boldsymbol{z}}
\def\hbtheta{\hat{\boldsymbol \theta}}
\def\normal{{\mathsf{N}}}
\def\btheta{\boldsymbol{\theta}}
\def\bDelta{\boldsymbol{\Delta}}
\def\bOmega{\boldsymbol{\Omega}}
\def\bdelta{\boldsymbol{\delta}}
\def\hq{\hat{q}}
\def\id{{\boldsymbol I}}
\def\rP{{\rm P}}
\def\interior{{\rm int}}
\def\strain{\mbox{\tiny\rm train}}
\def\stest{\mbox{\tiny\rm test}}

\def\ssynth{\mbox{\tiny\rm synth}}
\def\sdiff{\mbox{\tiny\rm diff}}

\def\bd{\boldsymbol{d}}
\def\hP{\hat{P}}
\def\TV{{\rm TV}}
\def\cA{\mathcal{A}}
\def\oY{\overline{Y}}
\def\oX{\overline{X}}
\def\br{{\boldsymbol r}}
\def\bphi{{\boldsymbol \phi}}
\def\reals{{\mathbb R}}
\newcommand{\contdiff}{\text{DiffLTM}}
\newcommand{\archscore}[2]{#1\,{\tiny\normalfont\,(#2)}}

\title{DiffLTM: Continuous Diffusions Large\\
Tabular Models for Conditional Generation}

\author{
Alankrita Bhatt\textsuperscript{1}\thanks{Authors are listed in alphabetical order.}\quad
Mukur Gupta\textsuperscript{1}\quad
Andrea Montanari\textsuperscript{1,2}\quad
Pulkit Tandon\textsuperscript{1}\thanks{Current affiliation: Apple.}\\[0.75em]
\small\textsuperscript{1}Granica Computing Inc.\\
\small\textsuperscript{2}Stanford University
}

\date{}

\begin{document}

\maketitle

\begin{abstract}
Conditional generation is a key capability to be supported by Large Tabular Models, enabling a number of applications. Accurate conditional generation has to contend with two key challenges: $(i)$~Data comprise a mixture of numerical and categorical entries, the latter often with high cardinality and heavily unbalanced proportions; $(ii)$~Conditional generation becomes increasingly difficult for rare sub-populations. The best models in this class use a hybrid of continuous and discrete diffusions. In this paper, we design a continuous diffusion model with strong generation quality. An important element is the use of cross-entropy loss for categoricals, for which we provide a first theoretical justification. 
We introduce metrics that focus on rare subpopulations.  Further, we generalize conditional generation techniques
to enable new types of conditioning. Finally, we show that our model (trained on labor data) can capture classical insights from econometrics.

Code is available at \url{https://github.com/granica-ai/DiffLTM}.
\end{abstract}

\tableofcontents

\section{Introduction} \label{sec:intro}

The term Large Tabular Model (LTM) was coined by \cite{van2024position}
to denote generative models for tabular data with cross-dataset modeling capabilities,
enabling a variety of tasks. Among the key challenges outlined in \cite{van2024position},
we highlight three:
\begin{itemize}
\item[$(i)$] Tabular data present a mixture of categorical and numerical data, 
    and categoricals are often high-cardinality or heavily skewed.
\item[$(ii)$] Standard evaluation metrics (e.g., FID, precision-recall) are often hard to interpret or misleading in the tabular setting, and synthetic samples are hard to interpret.
\item[$(iii)$] Conditional generation is arguably the most important capability for such models, as it enables a number of applications: forecasting; scenario analysis; imputation, and so on. Accurate conditional generation is particularly important and challenging for infrequently occurring conditions.
\end{itemize}

In this paper, we address these challenges. As a result, we design a generative tabular model 
with state of the art conditional generation behavior (Figure~\ref{fig:cfg_acslabour_cond_shape_all}). We train the model on sizable census 
datasets (up to $7$M rows), and compare it with alternative approaches 
using newly designed faithfulness metrics. In the rest of this introduction we summarize
some of our innovations, contrasting them with earlier literature and providing pointers to the rest of the paper.

\begin{figure}[t]
    \centering
    \includegraphics[width=0.9\textwidth]{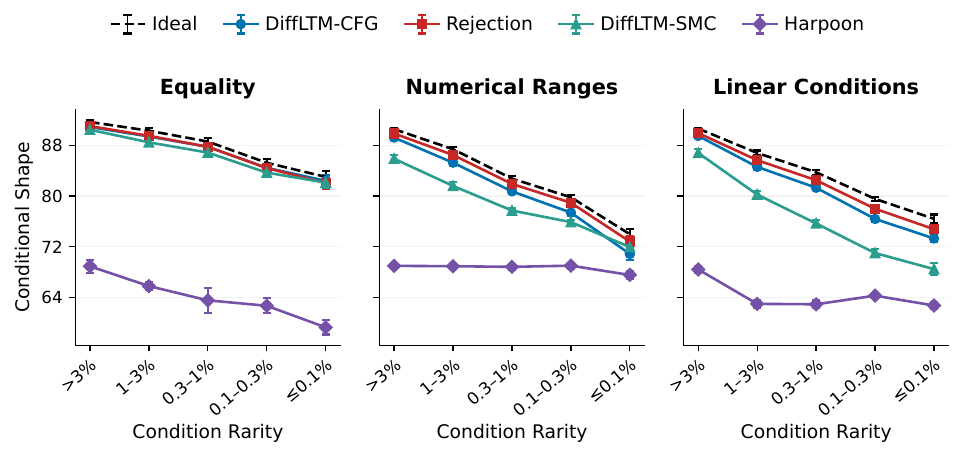}
    \caption{
    Conditional Shape score for conditional generation of ACS Labour data ($n_{\strain} \approx 4.8$M). For each condition $\mathcal{C}$, scores are evaluated on matched synthetic and test samples of size $n_{\ssynth}=n_{\stest}=\min(n_{\stest}^{\mathcal{C}},100{,}000)$ where $n_{\stest}^{\mathcal{C}}$ denotes the number of rows in the test set that satisfy $\mathcal{C}$. 
    \emph{\contdiff{}-CFG}  achieves accuracy very close to \emph{Rejection} sampling 
    (ideal conditioning using samples from \contdiff{}) and to \emph{Ideal}, outperforming
     \emph{\contdiff{}-SMC} and \emph{Harpoon}  (Section~\ref{sec:condgen};  Appendix~\ref{app:CondGenEval}).
    }
    \label{fig:cfg_acslabour_cond_shape_all}
\end{figure}

\paragraph{Metrics.} An important concern with tabular data is that humans 
do not have good intuition of when the generated data
`look good.' This is in sharp 
contrast with images or language models. We address this challenge in two ways.
\emph{First,} we define interpretable fidelity metrics that emphasize the 
accuracy on rare subpopulations (see Section~\ref{subsec:FidelityMetrics}). 
While in hindsight this is a natural idea, previous work uses metrics that are dominated by majority 
classes \citep{shi2025tabdiff, mueller2025continuous}.

\emph{Second,} we argue that an
LTM must correctly reproduce \emph{non-obvious statistical patterns} in 
the target population. To implement this idea, we consider two major
discoveries in econometrics: counterfactual analysis using 
synthetic controls
and Mincer equation on returns to education. 
In a sense, checking such patterns is the analogue of the `6-fingers test' for image generation.
We demonstrate that our LTM reproduces these patterns (see Figure~\ref{fig:AZ_PA_SC_result}).

\begin{figure}[t]
    \centering
    \includegraphics[width=0.9\textwidth]{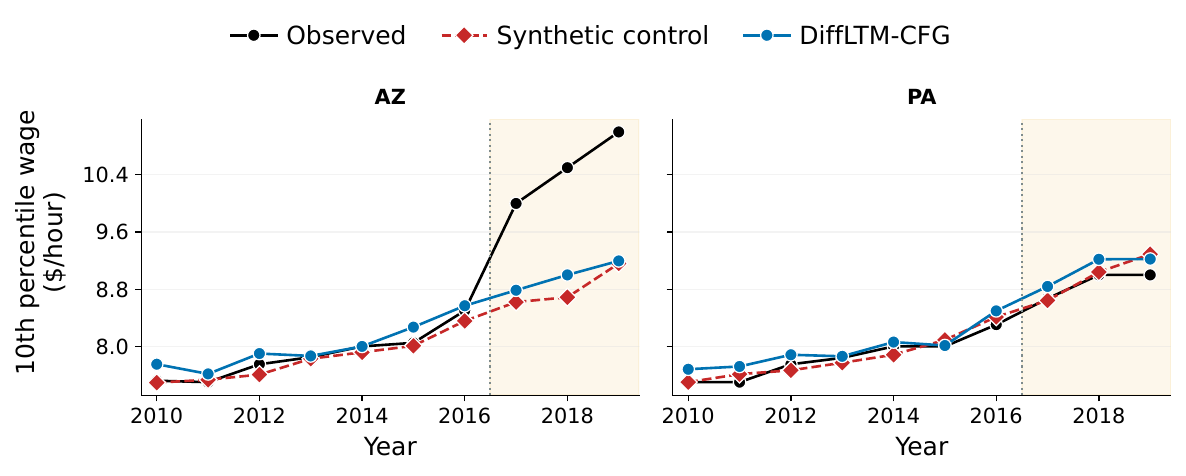}
    \caption{ Evolution of 10th percentile hourly wage in Arizona and Pennsylvania: 
    observed wages display the effect of the Arizona minimum wage increase post-2016.
    Comparison of counterfactuals based on synthetic controls and conditional generation.
      Trained without each state's post-2016 rows (shaded), \emph{\contdiff{}-CFG} predicts the held-out Pennsylvania wages 
      (placebo, right), and in Arizona (left) produces a no-policy counterfactual well below the observed post-2017 wages. (Section~\ref{subsec:Econ}; Appendix~\ref{app:AZMinwageDetails}).
    }
    \label{fig:AZ_PA_SC_result}
\end{figure}

\paragraph{Architecture.} We develop a diffusion LTM
(DiffLTM). Given the mixture of categorical and numerical data in tabular datasets, an important choice is whether to use discrete 
or continuous diffusion, or a hybrid. State of
the art methods such as TabDiff \citep{shi2025tabdiff} are hybrid,
while examples of purely continuous methods are \cite{mueller2025continuous, mueller2026cascaded, si2025tabrep}.

Inspired by recent results on theoretical and empirical superiority of continuous diffusions in the canonical setting of random constraint satisfaction problems \cite{BGKM2026} we demonstrate that excellent generative performance can be achieved by a
purely continuous model, provided embeddings and loss function are chosen carefully. 
We compare our model with continuous diffusion baselines  and demonstrate state of the art unconditional generation performance (Table \ref{tab:unconditional-baselines-vs-learnt-ce}). 

We devote significant attention to the choice of the loss for categoricals. 
A subset of earlier applied works  \citep{mueller2025continuous, mueller2026cascaded, eijkelboom2024variational},
has chosen cross-entropy (CE) loss even in conjunction with categoricals embeddings and continuous diffusions.
The motivation for this choice is not clear, and theoretical arguments appear to support score matching (a.k.a. MSE)
loss for continuous diffusions \citep{chewi2026ddpm}. We present empirical evidence that CE with learned embeddings
is preferable (Table~\ref{tab:unconditional-architectures}).

Further, we show mathematically that --under a simple 
misspecification scenario-- cross-entropy is indeed the optimal choice (Theorem \ref{thm:Overall-Limit-Gen-Informal}).
We also prove that continuous diffusions result in a bounded inefficiency in a simple discrete variable generation task 
(Theorem \ref{thm:Bounded-Suboptimality}).   

\paragraph{Conditional generation.} We study three types of conditions: $(i)$ \emph{equality conditions}: a subset of columns take specified values;  $(ii)$ \emph{range conditions}: numerical column $i$ 
lies in a range $[L_i,U_i]$;
$(iii)$ \emph{linear conditions}: A linear combination of the numerical columns $x_i$ satisfies  $L \leq \sum_i w_i x_i \leq U$.
A number of important applications can be modeled using conditions of 
these three types. We describe how to adapt classifier free guidance
(CFG) \citep{ho2022classifier} to enable these conditions, and show state-of-the art (and near-optimal) performances. 
We also adapt inference-time interventions (in particular 
Sequential Monte Carlo (SMC) \citep{doucet2001introduction}) and compare the results with CFG, showing that both 
of these approaches significantly outperform baselines in tabular conditional generation \citep{shankar2026harpoon} (Figure~\ref{fig:cfg_acslabour_cond_shape_all}).

\subsection{Further related work}

The literature on generative models for tabular data is rapidly
growing. We provide a few pointers beyond the ones 
mentioned above. 

Evaluation   is well recognized to be an outstanding challenge 
because of the heterogeneity of tabular data
 \citep{nafis2025critical}. Metrics based on statistical accuracy
\citep{alaa2022faithful}, downstream utility of synthetic data 
\citep{xu2019modeling},
latent causal relationships \citep{jiang2025well} 
have been developed.

A variety of techniques and ideas have been
brought to bear to develop generative models for tabular data.
Among others, generative adversarial networks \citep{xu2019modeling},
graph neural networks \citep{liu2023goggle},
based on autoregressive models \cite{borisov2022language},
latent space diffusions  \citep{zhang2024mixed},
variational flow matching \cite{guzman2025exponential}.

Finally, a small number of recent papers develop models for
relational databases beyond the single-table setting
\cite{hudovernik2025reldiff,ketata2026joint,scassola2026graph}.

\section{Metrics} \label{sec:Metrics}

\begin{table}[t]
    \centering
    \small
    \caption{Comparison of \contdiff{}, TabRep~\citep{si2025tabrep}, TabCascade~\citep{mueller2026cascaded} and TabDiff~\citep{shi2025tabdiff} on unconditional generation. Higher values indicate better performance for all metrics except the DCR Ratio 
    (mean (SD) over three training seeds).
    If multiple methods are within $0.3\%$ of the best mean, all are underlined and none is bold; otherwise the best is bold and the second-best underlined.
    }    
    \label{tab:unconditional-baselines-vs-learnt-ce}
    \begin{adjustbox}{width=\linewidth}
    \begin{tabular}{llrrrrrr}
    \hline
    Dataset & Method & Shape$\times 100$ & Trend$\times 100$ & CShape$\times 100$ & CShapeR$\times 100$ & MLE$\times 100$ & DCR ratio \\
    \hline
    \multirow{4}{*}{ACSLabour} & TabRep & \archscore{93.52}{0.12} & \archscore{87.03}{0.16} & \archscore{88.55}{0.13} & \archscore{87.10}{0.15} & \archscore{98.46}{0.06} & \archscore{1.001}{0.000} \\
     & TabCascade & \archscore{\underline{99.03}}{0.05} & \archscore{\underline{97.61}}{0.06} & \archscore{\underline{94.25}}{0.31} & \archscore{\underline{92.33}}{0.54} & \archscore{\underline{98.72}}{0.15} & \archscore{1.000}{0.000} \\
     & TabDiff & \archscore{98.64}{0.25} & \archscore{96.00}{0.36} & \archscore{92.78}{0.26} & \archscore{90.81}{0.20} & \archscore{97.29}{2.97} & \archscore{1.000}{0.000} \\
     & \contdiff & \archscore{\underline{99.04}}{0.04} & \archscore{\underline{97.66}}{0.04} & \archscore{\textbf{95.37}}{0.04} & \archscore{\textbf{93.93}}{0.03} & \archscore{\textbf{99.41}}{0.05} & \archscore{1.000}{0.000} \\
    \hline
    \multirow{4}{*}{ACSPubCov} & TabRep & \archscore{99.15}{0.04} & \archscore{98.51}{0.05} & \archscore{98.51}{0.08} & \archscore{97.51}{0.11} & \archscore{99.01}{0.05} & \archscore{1.002}{0.001} \\
     & TabCascade & \archscore{98.02}{0.13} & \archscore{96.25}{0.27} & \archscore{93.09}{0.57} & \archscore{82.69}{4.03} & \archscore{92.98}{0.20} & \archscore{0.997}{0.001} \\
     & TabDiff & \archscore{\underline{99.74}}{0.03} & \archscore{\textbf{99.44}}{0.04} & \archscore{\underline{99.00}}{0.04} & \archscore{\underline{97.93}}{0.08} & \archscore{\underline{99.63}}{0.02} & \archscore{1.001}{0.000} \\
     & \contdiff & \archscore{\underline{99.48}}{0.04} & \archscore{\underline{99.09}}{0.06} & \archscore{\underline{98.98}}{0.06} & \archscore{\underline{98.03}}{0.19} & \archscore{\underline{99.58}}{0.06} & \archscore{1.001}{0.001} \\
    \hline
    \multirow{4}{*}{ACSMobility} & TabRep & \archscore{99.08}{0.11} & \archscore{98.51}{0.22} & \archscore{97.91}{0.13} & \archscore{96.70}{0.11} & \archscore{99.10}{0.01} & \archscore{1.084}{0.161} \\
     & TabCascade & \archscore{98.71}{0.02} & \archscore{97.77}{0.10} & \archscore{96.40}{0.30} & \archscore{93.92}{0.76} & \archscore{97.87}{0.35} & \archscore{0.839}{0.026} \\
     & TabDiff & \archscore{\underline{99.62}}{0.05} & \archscore{\textbf{99.29}}{0.10} & \archscore{\underline{98.35}}{0.00} & \archscore{\underline{97.02}}{0.02} & \archscore{\underline{99.60}}{0.02} & \archscore{1.035}{0.035} \\
     & \contdiff & \archscore{\underline{99.39}}{0.12} & \archscore{\underline{98.97}}{0.12} & \archscore{\underline{98.20}}{0.07} & \archscore{\underline{96.88}}{0.22} & \archscore{\underline{99.48}}{0.05} & \archscore{1.090}{0.055} \\
    \hline
    \end{tabular}
    \end{adjustbox}
    \label{tab:unconditional-comparisons}
\end{table}

\subsection{Datasets} \label{subsec:Datasets}

All datasets considered in this work are derived from the American Community Survey Public Use Microdata Sample (ACS PUMS), using
Folktables~\citep{ding2021retiring}. We use two datasets from this suite: ACSPubCov (originally constructed to build predictive models for public health-insurance coverage), and ACSMobility
(for predictive models of residential mobility.)

In addition, we construct \emph{ACSLabour}, a dataset 
which retains continuous log wage, as well as a broad set of labor and demographic
 columns, including high-cardinality occupation, industry, geographic, and birthplace variables.
 (See Appendix~\ref{app:Datadetails} for detailed descriptions.)

\subsection{Fidelity metrics} \label{subsec:FidelityMetrics}

We assess the quality of generated data by comparing two samples:
$n^{\stest}$ hold-out rows from the target table;
$n^{\ssynth}$ rows generated by our model (often $n^{\stest}=n^{\ssynth}$).
All fidelity metrics can be thought of as statistical estimators
of (pseudo-)distances between the corresponding distributions 
$P^{\stest}$, $P^{\ssynth}$. Typically, the `plug-in' estimator that 
replaces $P^{\stest}$, $P^{\ssynth}$ by their empirical versions 
$\hP^{\stest}$, $\hP^{\ssynth}$ is used. 
We provide the definitions of the metrics used below.
Unless otherwise stated, the normal range of these metrics is $[0,1]$,
with $1$ corresponding to the maximum possible value. 
We use \texttt{ltm-eval} package\footnote{\url{https://github.com/granica-ai/ltm-eval}} for 
fidelity metrics.

{\bf Shape.} Shape~\citep{sdmetrics} measures fidelity of 
univariate marginal distributions, averaged uniformly over the 
column index. 
For numerical column $j$, let 
$S_j=1-\mathrm{KS}(\hP_j^{\stest},\hP_j^{\ssynth})$, where $\mathrm{KS}$ 
denotes the Kolmogorov--Smirnov distance and $\hP_j^{\stest}, \hP_j^{\ssynth}$
are the empirical marginals of column $j$ in the test or synthetic 
data\footnote{We view $\mathrm{KS}$ as acting on laws: $\mathrm{KS}(P,Q):=
\sup_{x\in \reals} |P((-\infty,x]) - Q((-\infty,x])|$.}. 
For categorical column $j$, let $S_j=1-\TV(\hP_j^{\stest},
\hP_j^{\ssynth})$, where $\TV$ is the total variation distance. We report
\begin{align}\label{eq:ShapeDefn}
\mathrm{Shape}:=\frac{1}{m}\sum_{j=1}^m S_j.
\end{align}

{\bf Trend.} Trend~\citep{sdmetrics} measures  the fidelity of pairwise marginals 
(for categorical columns) or pairwise correlations (for numerical ones), averaged uniformly over column indices. 
For each pair of columns $(j,\ell)$, if both columns are numerical, let $T_{j\ell}=1-\frac{1}{2}|\widehat{\rho}_{j\ell}^{\mathrm{real}}-\widehat{\rho}_{j\ell}^{\mathrm{syn}}|$, where $\widehat{\rho}_{j\ell}$ is the empirical Pearson correlation. Otherwise, let $T_{j\ell}=1-\TV(\widehat{p}_{j\ell}^{\mathrm{real}},\widehat{p}_{j\ell}^{\mathrm{syn}})$, where $\widehat{p}_{j\ell}$ is the empirical joint distribution of the two columns, with numerical columns discretized when necessary. We report
\begin{align}\label{eq:TrendDefn}
\mathrm{Trend}:=\frac{2}{m(m-1)}\sum_{j<\ell}T_{j\ell}.
\end{align}

{\bf Conditional Shape.} We introduce a new metric named \emph{conditional shape}, 
which is the mean shape \emph{within subpopulations}. For  fixed 
categorical conditioning column $j^*$ and categorical value $a$ 
(the case of a numerical column is analogous), let $\hP^{\stest}_{j|j^*,a}$ and $\hP^{\ssynth}_{j|j^*,a}$ be the empirical
distributions of column $j$ (respectively, real and synthetic)
for rows with $x_{j^*}=a$. The shape score of column $j$ is
$S_{j\mid x_j^*=a}=1-\TV(\hP^{\stest}_{j|j^*,a},\hP^{\ssynth}_{j|j^*,a})$ 
and we define the shape score of subpopulation $(j^*,a)$ by averaging over 
all columns except $j^*$:
\begin{align} \label{eq:SliceShapeDef}
S(j^*,a):=\frac{1}{m-1}\sum_{j\neq j^*}S_{j\mid j^*,a}.
\end{align}
If $\mathcal{A}_j$ is the set of values appearing in  column $j$, we 
report\footnote{We sometimes exclude values whose frequency is below a small threshold.}
\begin{align}\label{eq:CShapeDefn}
\mathrm{CShape}:=\frac{1}{m}\sum_{j=1}^m\frac{1}{|\mathcal{A}_j|}\sum_{a\in\mathcal{A}_j}S(j,a).
\end{align}
Note that subpopulations receives equal weight regardless of their frequency,
hence emphasizing the importance of rare subpopulations with respect to standard Shape.

{\bf Conditional Shape (Rare).}  For a fixed $\tau\in (0,1)$, we denote by 
 $\mathcal{A}_j^{\mathrm{rare}}\subseteq \mathcal{A}_j$
 the subset of values that appear in column $j$ with frequency below $\tau$
  (by default $\tau=0.05$). We define
\begin{align}\label{eq:CShapeRareDefn}
 \mathrm{CShape}_{\mathrm{rare}}:=\frac{1}{m}\sum_{j=1}^m\frac{1}{|\mathcal{A}_j^{\mathrm{rare}}|}\sum_{a\in\mathcal{A}_j^{\mathrm{rare}}}S(j,a).
\end{align}
This metric emphasizes even more accuracy on rare sub-populations.
We note that Conditional Trend and Conditional Trend (Rare) can be defined 
analogously to the above.

{\bf MLE Ratio.} \citep{zhang2024mixed} We train an XGBoost model on synthetic data and evaluate 
it on held-out real data (\emph{Train Synthetic, Test Real}; TSTR).
We divide this score by that of an identically configured XGBoost model
 trained on real data and evaluated on the same held-out set 
 (\emph{Train Real, Test Real}; TRTR). 
 Then, 
\begin{align}\label{eq:MLERatioDefn}
\mathrm{MLE\ Ratio}:=\frac{\mathrm{Score}_{\mathrm{TSTR}}}{\mathrm{Score}_{\mathrm{TRTR}}}.
\end{align}
As score, we use AUC for classification and $R^2$ for regression; values near one 
indicate that the synthetic data preserve downstream predictive utility.

{\bf DCR median ratio.} \citep{zhang2024mixed} Let $\mathcal{D}_{\mathrm{train}}$ and $\mathcal{D}_{\mathrm{test}}$ denote the real training and held-out datasets, 
respectively, and let $\mathcal{D}_{\mathrm{syn}}$ denote the synthetic dataset. 
Let $d(\boldsymbol{z},\mathcal{D}_{\mathrm{train}}) := \min_{\boldsymbol{z'} \in \mathcal{D}_{\mathrm{train}}} d(\boldsymbol{z}, \boldsymbol{z'})$, where $d(\cdot,\cdot)$ is the Euclidean distance between rows $\boldsymbol{z}$ and $\boldsymbol{z'}$ after categorical columns are one-hot encoded. 
We let 
\begin{align}\label{eq:DCRRatioDefn}
\mathrm{DCR\ Ratio}:=
\frac{\operatorname{median}_{\boldsymbol{z}\in\mathcal{D}_{\ssynth}}d(\boldsymbol{z},\mathcal{D}_{\mathrm{train}})}
{\operatorname{median}_{\boldsymbol{z}\in\mathcal{D}_{\stest}}d(\boldsymbol{z},\mathcal{D}_{\mathrm{train}})} .
\end{align}
Values substantially below one are indicative of memorization.

\paragraph{Results: unconditional generation} 
Table \ref{tab:unconditional-comparisons} compares \contdiff{} 
on multiple metrics defined above with three state-of-the-art 
methods\footnote{In order to perform a fair comparison,
we tuned learning rate and increased the number of epochs.}:
TabRep \citep{si2025tabrep}, TabCascade \cite{mueller2026cascaded}, 
and TabDiff \cite{shi2025tabdiff}.  

TabRep and TabCascade are close comparison points since they are based on continuous diffusions. \contdiff{} outperforms them on most metrics and all datasets.
TabDiff is based on a continuous and discrete diffusions (respectively,
for numerical and categorical columns). TabDiff and \contdiff{} have comparable performance 
in two out of three datasets, with an advantage for \contdiff{} on ACSLabour. 

We note that \contdiff{} has an MLP-only architecture (as compared with the use of Transformers in TabDiff), leading to an approximate $5\times$ speed up in inference and training (more details in Appendix~\ref{app:uncondBaselinesDetails}). An additional advantage of entirely continuous diffusions is that there is evidence that  they are better parallelizable \cite{chen2026parallelism}.

We point out that 
earlier benchmarking of these models was carried out on significantly 
smaller datasets than our study. All datasets considered in 
\cite{si2025tabrep,shi2025tabdiff} have less than 100K rows; and
all datasets but one in \cite{mueller2026cascaded} have less than 102K
rows, with one exception comprising 540K rows. In contrast our study
focuses on moderate-size datasets with $\sim 5$M rows.

\section{Diffusion architecture and unconditional generation} \label{sec:diffusion_arch}

\begin{table}[t]
\centering
\small
\caption{Comparison across architectural variants, including MSE versus CE 
loss, representation (Repr) versus logits parameterization, and the effect 
of learned embeddings. Higher values indicate better performance for all metrics 
(mean (SD) over three training seeds) 
If multiple approaches are within $0.3\%$ of the best mean, all are underlined and none is bold; otherwise the best is bold and the second-best underlined.
}    
\label{tab:unconditional-architectures}
\begin{tabular}{llrrrr}
\hline
Dataset & Architecture & Shape & Trend & CShape & MLE \\
\hline
\multirow{4}{*}{ACSLabour} & MSE+Repr & \archscore{89.39}{0.18} & \archscore{79.61}{0.31} & \archscore{72.36}{0.18} & \archscore{91.22}{1.21} \\
 & MSE+Logits & \archscore{93.70}{0.48} & \archscore{87.40}{0.80} & \archscore{81.40}{0.91} & \archscore{97.33}{0.43} \\
 & CE+Logits & \archscore{\underline{98.72}}{0.06} & \archscore{\underline{96.87}}{0.10} & \archscore{\underline{94.87}}{0.07} & \archscore{\underline{99.10}}{0.02} \\
 & CE+Logits+Learnt & \archscore{\textbf{99.04}}{0.04} & \archscore{\textbf{97.66}}{0.04} & \archscore{\textbf{95.37}}{0.04} & \archscore{\textbf{99.41}}{0.05} \\
\hline
\multirow{4}{*}{ACSPubCov} & MSE+Repr & \archscore{98.61}{0.35} & \archscore{97.32}{0.63} & \archscore{95.56}{1.73} & \archscore{98.65}{0.69} \\
 & MSE+Logits & \archscore{98.84}{0.10} & \archscore{97.80}{0.15} & \archscore{97.01}{1.06} & \archscore{99.17}{0.28} \\
 & CE+Logits & \archscore{\underline{99.70}}{0.02} & \archscore{\textbf{99.42}}{0.03} & \archscore{\underline{99.13}}{0.05} & \archscore{\underline{99.55}}{0.02} \\
 & CE+Logits+Learnt & \archscore{\underline{99.48}}{0.04} & \archscore{\underline{99.09}}{0.06} & \archscore{\underline{98.98}}{0.06} & \archscore{\underline{99.58}}{0.06} \\
\hline
\multirow{4}{*}{ACSMobility} & MSE+Repr & \archscore{\underline{99.20}}{0.28} & \archscore{98.55}{0.41} & \archscore{94.59}{1.68} & \archscore{98.91}{0.26} \\
 & MSE+Logits & \archscore{\underline{99.31}}{0.21} & \archscore{\underline{98.75}}{0.34} & \archscore{96.46}{1.16} & \archscore{\underline{99.25}}{0.17} \\
 & CE+Logits & \archscore{\underline{99.48}}{0.14} & \archscore{\underline{98.99}}{0.22} & \archscore{\underline{98.08}}{0.36} & \archscore{99.12}{0.18} \\
 & CE+Logits+Learnt & \archscore{\underline{99.39}}{0.12} & \archscore{\underline{98.97}}{0.12} & \archscore{\underline{98.20}}{0.07} & \archscore{\underline{99.48}}{0.05} \\
\hline
\end{tabular}
\end{table}

In this section, we outline the diffusion architecture, emphasizing 
treatment of categoricals. We refer to Appendices \ref{app:uncondGenDetails} 
and \ref{app:CondGenDetails} for precise details regarding 
in particular the denoiser architecture.  

\paragraph{Continuous embedding.} Consider a table with (without loss of generality) the first $m_{\text{num}}$ columns being numerical, and the latter $m_{\text{cat}}$ columns being categorical. Let categorical column $j$ take values in $\{1,\ldots,C_j\}$. For each such column $j$, we associate an embedding
\[
\phi_j : \{1,\ldots,C_j\} \to \mathbb{R}^{d_j}.
\]
A widely used choice is the \emph{one-hot} embedding 
$\phi_j(i) = e_i \in \reals^{C_j}$. Another choice is 
\emph{learned embeddings} where $\phi_j(i) \in \reals^{d_{j}}$ 
is a vector of learnt parameters. We enforce the normalization
$\|\phi_j(i)\|=\sqrt{d_{j}}$. 
For a data point $\boldsymbol{x}=(x_1,\ldots,x_{m_{\text{num}}+m_{\text{cat}}})$, 
we thus have a continuous representation $\br \in \reals^D$ 
\begin{align} \label{eq:ContRepDefn}
\br = \left(x_1,\ldots,x_{m_{\mathrm{num}}},\phi_1(x_{m_{\mathrm{num}}+1}),\ldots,\phi_{m_{\mathrm{cat}}}(x_m)\right),
\end{align}
where $D = m_{\mathrm{num}} + \sum_{j=1}^{m_{\mathrm{cat}}} d_j$. 

{\bf Diffusion noising process.}
For each training example, we sample $t \sim \mathrm{Unif}[0,1]$, 
which corresponds to a ``noise level''. This level is then mapped to a variance 
via a monotonically increasing noise schedule $\sigma: [0,1] \to 
[\sigma_{\min},\sigma_{\max}]$. In our experiments, we use the 
\emph{power-mean schedule}
\[
\sigma_t = \left(\sigma_{\min}^{1/\rho} + t\left(\sigma_{\max}^{1/\rho} - \sigma_{\min}^{1/\rho}\right)\right)^\rho.
\]
of~\cite{karras2022elucidating}. The forward corruption occurs 
in the continuous representation space, so that for a row's representation $\boldsymbol{r} \in \mathbb{R}^D$ and $\boldsymbol{z} \sim \mathcal{N}(\boldsymbol{0}, I_D)$ the noising process is $\boldsymbol{r}_t := \boldsymbol{r} + \sigma_t\boldsymbol{z}$. 

\paragraph{Denoiser output head.}
The denoiser $f_{\theta}$ receives the noisy representation $\boldsymbol{r}_t$ of target row $\boldsymbol{x}$ and a sinusoidal embedding of the diffusion time $t$. Numerical columns are always predicted directly; for categorical columns, we consider two parameterizations:
\begin{enumerate}
\item \emph{Representation.} The denoiser predicts the clean continuous representation,
\[
f_\theta(r_t,t) = \left(\widehat{\boldsymbol{r}}_{\mathrm{num}},\widehat{\boldsymbol{r}}_{\mathrm{cat},1},\ldots,\widehat{\boldsymbol{r}}_{\mathrm{cat},m_{\mathrm{cat}}}\right),
\]
where $\widehat{\boldsymbol{r}}_{\mathrm{num}} = (\widehat{x}_1,\dotsc,\widehat{x}_{m_{\mathrm{num}}}) \in\mathbb{R}^{m_{\mathrm{num}}}$ and $\widehat{\boldsymbol{r}}_{\mathrm{cat},j}\in\mathbb{R}^{d_j}$. For categorical column $j$, the target is its clean embedding $\phi_j(x_j)$.

\item \emph{Logit.}
Alternatively, the denoiser predicts one vector of logits per categorical column,
\[
f_\theta(\boldsymbol{r}_t,t) = \left(\widehat{\boldsymbol{r}}_{\mathrm{num}},\boldsymbol{z}_1,\ldots,\boldsymbol{z}_{m_{\mathrm{cat}}}\right), \qquad \boldsymbol{z}_j \in \mathbb{R}^{C_j}.
\]
The logits induce the posterior mean embedding
\begin{align} \label{eq:EmbeddingDefLogitMode}
\boldsymbol{\widehat{r}}_{\mathrm{cat},j} = \sum_{c=1}^{C_j} \widehat{p}_{j,c}\phi_j(c)\, ,\qquad
\widehat{p}_{j,c} = \frac{\exp(z_{j,c})}{\sum_{c'=1}^{C_j}\exp(z_{j,c'})}.
\end{align}
\end{enumerate}

The denoiser is a fully-connected MLP with $L$ hidden layers of width
  $d_{\mathrm{h}}$ and \texttt{SiLU} activation. 

\paragraph{Loss function and training.}
For numerical columns, we use mean-squared error:
\[
\ell_{\mathrm{num}} := \frac{1}{m_{\mathrm{num}}}
\|\widehat{\boldsymbol{r}}_{\text{num}}-\boldsymbol{r}_{\text{num}}\|_2^2.
\]
For categorical columns, we consider two choices, square loss in embedding space, or cross-entropy loss
on logits:
    \begin{align} \label{eq:MSELossCatDef}
    \ell_{\mathrm{cat}}^{\mathrm{MSE}} := \frac{1}{m_{\mathrm{cat}}}
    \sum_{j=1}^{m_{\mathrm{cat}}}
    \frac{\left\|\widehat{\boldsymbol{r}}_{\mathrm{cat},j}-\phi_j(x_j)\right\|_2^2}{d_j}
    \qquad
        \ell_{\mathrm{cat}}^{\mathrm{CE}} = \frac{1}{m_{\mathrm{cat}}}\sum_{j=1}^{m_{\mathrm{cat}}}\log \frac{1}{\widehat{p}_{j, x_j}}.
    \end{align}

The total loss function is then $\ell = \ell_{\mathrm{num}} + \ell_{\mathrm{cat}}$. The optimizer then backpropagates the loss through the network and, if applicable, the learnt embedding lookup tables $E_j$.

{\bf Generation.}
Sampling is carried out entirely in the continuous representation space: obtain decreasing noise levels $\sigma_{t_0}>\sigma_{t_1}>\cdots>\sigma_{t_T}$ via the noise schedule and discretization of $[0,1]$, and initialize $\boldsymbol{r}^{(0)} \sim \mathcal{N}(\boldsymbol{0},\sigma_{t_0}^2 I_D)$. 
Using the denoiser prediction, we construct the clean representation $\widehat{\boldsymbol{r}}^{(i)} := \widehat{\boldsymbol{r}}(\boldsymbol{r}^{(i)},t_i)$ and perform the Euler update~\citep{karras2022elucidating}
\begin{align} \label{eq:EulerUpdateDef}
\boldsymbol{r}^{(i+1)} = \boldsymbol{r}^{(i)} + (\sigma_{t_{i+1}}-\sigma_{t_i})\frac{\boldsymbol{r}^{(i)}-\widehat{\boldsymbol{r}}^{(i)}}{\sigma_{t_i}}, \qquad i=0,\ldots,T-1.
\end{align}
Numerical coordinates are returned directly after the final step. For categorical column $j$, representation mode returns $\widehat{x}_j = \arg\min_{c\in\{1,\ldots,C_j\}}\|\widehat{\boldsymbol{r}}_{\mathrm{cat},j}-\phi_j(c)\|_2$, while logit mode returns $\widehat{x}_j = \arg\max_{c\in\{1,\ldots,C_j\}}\widehat{p}_{j,c}$.

{\bf Empirical results.} 
We note that the parameterization and the loss function together 
create three possible architecture variants: representation with MSE loss, logits with MSE loss, and logits with CE loss. The empirical results in Table~\ref{tab:unconditional-architectures} demonstrate the superiority of logits parametrization with CE loss; furthermore, they demonstrate that learnt embeddings enhance the performance for this model vis-a-vis the one-hot embedding.

\section{Conditional generation} \label{sec:condgen}

We study conditional generation using \contdiff{} with two approaches:
Classifier-free guidance (CFG) \citep{ho2022classifier},  and Sequential Monte Carlo
(SMC) \citep{doucet2001introduction}. We extend CFG to range
and linear conditions, and compare them with baselines in Figure~\ref{fig:cfg_acslabour_cond_shape_all}.

\subsection{Train-time intervention: Classifier-free guidance} \label{subsec:CFG}

{\bf Denoiser.}
In addition to $\boldsymbol{r}_t$, the conditional denoiser takes as input the condition $\mathcal{C}$ and outputs $f_\theta(\boldsymbol{r}_t,\mathcal{C},t)$ to generate rows satisfying condition $\mathcal{C}$. 
As anticipated in the introduction, we consider three types of conditions: $(i)$~\emph{equality conditions}: for a categorical column $i$, these
take the form $x_i=a$; $(ii)$~\emph{range conditions}, 
for a numerical column $i$, $x_i\in [L_i, U_i]$; and $(iii)$ \emph{linear conditions}, which constrain a linear function of the numerical columns $\sum_i w_i x_i \in[L, U]$. Multiple of these conditions can be imposed simultaneously.

The denoiser takes as input a tensor that specifies,
for each column, the noisy values as well as a specification
of the various conditions that involve that column 
(e.g. the weight $w_i$ for a linear constraint).
For the case of linear conditions, we also input the value
of the linear function $\sum_i w_i r_{t,i}$ at the current state.

Each condition type is accompanied by a Boolean mask, indicating whether the corresponding condition is present. Whenever an equality condition is specified for column $i$, we replace the noisy state $r_{t,i}$ with a special token $\tau_i$. For numerical columns, $\tau_i$ is a learned scalar, whereas for categorical columns it is a learned embedding. Beyond these  modification to the input representation, the remainder of the architecture is identical to that used for unconditional generation.

{\bf Training.} During CFG training~\citep{ho2022classifier}, 
we need to sample a condition $\mathcal{C}$.
We proceed as follows: we draw a masking probability $p_{\mathrm{mask}} \sim \mathrm{Unif}(0,1)$ and independently select each column for conditioning with probability $p_{\mathrm{mask}}$. Selected categorical columns are always assigned equality conditions. For each selected numerical column, the equality condition is replaced by a range condition with probability $p_{\mathrm{rng}}$. To construct a linear condition, we sample the number of participating numerical columns, $m_{\mathrm{lin}}$,
 uniformly in $\{1,\ldots,m_{\mathrm{num}}\}$ and then select $m_{\mathrm{lin}}$ numerical columns. With probability $p_{\mathrm{lin}}$, we impose a linear condition over these selected columns and remove any equality or range conditions previously assigned to them.


Equality conditions are constructed by setting the conditioned value of each selected column $i$ to its value $x_i$ in the corresponding training row. For numerical range conditions, we must ensure that the optimal conditional denoiser targets the appropriate truncated distribution, $p_{\mathrm{target}}(\boldsymbol{x}) \propto p(\boldsymbol{x})\,\mathbbm{1}\{L_i < x_i < U_i\}$. Specifically, we first sample an interval width $\Delta \sim \mathrm{LogUnif}(\Delta_{\min},\Delta_{\max})$, then sample the lower bound as $L_i \sim \mathrm{Unif}(x_i-\Delta,x_i)$, and set $U_i=L_i+\Delta$. For a linear condition involving $m_{\mathrm{lin}}$ numerical columns, we sample the corresponding weight vector as $\boldsymbol{w}\sim\mathcal{N}(\boldsymbol{0},I_{m_{\mathrm{lin}}})$. We then construct the bounds $L_b$ and $U_b$ using the same interval-width sampling procedure, centered around the value of the linear function $\boldsymbol{w}^{\top}\boldsymbol{x}$. Finally, we normalize the linear condition by $\|\boldsymbol{w}\|_2$, scaling both the weights $w_i$ and the bounds $L_b$ and $U_b$ accordingly.

\subsection{Inference-time methods} \label{subsec:InfTimeMethods}

We experiment with two inference-time methods: rejection sampling and sequential Monte Carlo (SMC). Rejection sampling is unbiased but highly inefficient for rare conditions, so we use it as a reference (\emph{Rejection} in Figure~\ref{fig:cfg_acslabour_cond_shape_all}).

SMC is a class of methods widely used across statistics and signal processing to perform posterior sampling in sequential systems \citep{doucet2001introduction}, which was recently used in conjunction with diffusion models \citep{wu2023practical}. At a high level, SMC entails initializing $K$ reverse-diffusion trajectories and reweighting them in accordance with how likely each trajectory is to eventually produce a synthetic row satisfying the requested condition (see Algorithm~\ref{alg:smc} in Appendix~\ref{app:pseudocode} for details)---the main algorithmic question that then arises is how to approximate this probability. To this end, we make three approximations: $(i)$ columns are conditionally independent given the current trajectory state; $(ii)$ categorical columns follow the distribution induced by the predicted denoiser logits; and 
$(iii)$~numerical columns follow a Gaussian distribution centered at the denoiser prediction.
We note that choosing logit mode enables assumption $(ii)$, adding to the previously outlined 
reasons supporting this choice. SMC enables conditioning on categorical predicates as well as arbitrary numerical conditions, a capability not offered by prior work such as \emph{Harpoon} \citep{shankar2026harpoon}.

\section{Reproducing econometric patterns via conditional generation} \label{subsec:Econ}

\begin{figure}[t]
    \centering
    \includegraphics[width=0.99\textwidth]{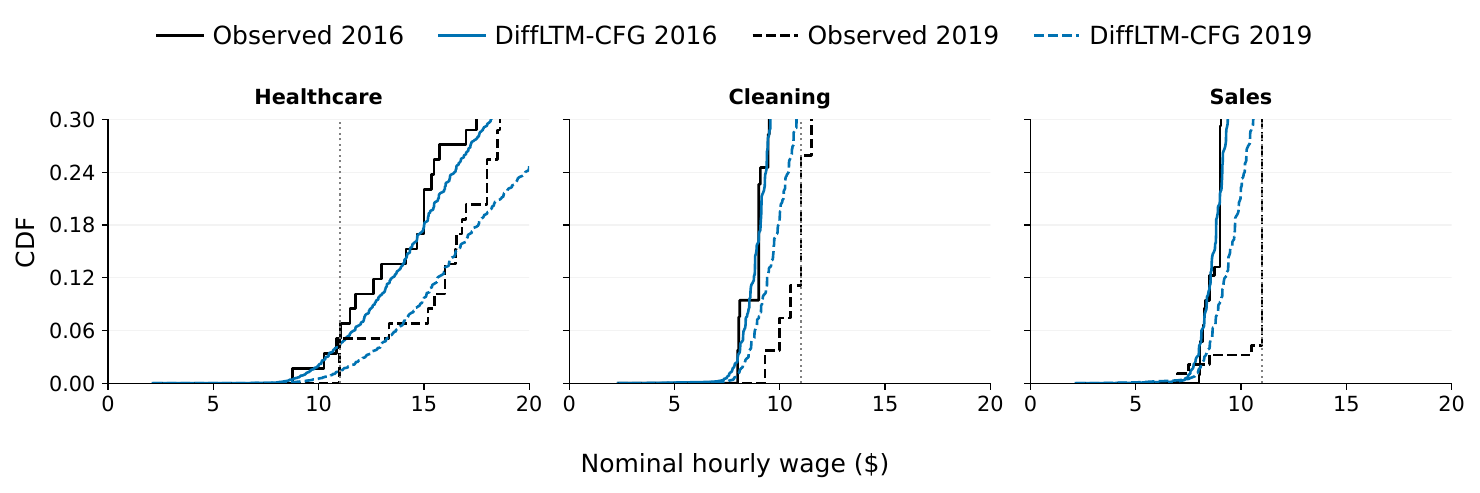}
    \caption{
    The counterfactual effect of Proposition~206 tracks each occupation's exposure to the minimum wage. The gap between dashed curves (i.e. true 2019 values and \contdiff{} predicted counterfactual) is large for cleaning and sales workers, whose observed wages pile up at \$11, and negligible for healthcare workers, whose wages already sat above it (see Appendix~\ref{app:AZMinwageIndiv} for details).
    }
    \label{fig:AZ_SC_indiv_occ}
\end{figure}

 We use two important econometric ideas to test 
whether our model can capture subtle behavioral patterns 
in the data:
counterfactual analysis and Mincer equation on returns on education.
  Results presented in this section use the same denoiser architecture 
  as in previous ones, with additional adaptive layer norm (AdaLN) 
  \cite{peebles2023scalable} in the reparametrization experiment in Figure~\ref{fig:sc_reparam_median}. Using or not AdaLN does not have measurable impact on experiments reported in previous sections, but it yields a significant improvement in the reparametrization robustness analysis. 

\subsection{Counterfactual analysis} 

An entity $s_*$
receives a policy intervention at time period $t_*$.
The intervention might have an effect on a certain outcome $\oY(s_*,t)$
at times $t\ge t_*$. 
What would the outcome 
have been if the intervention had not happened?

As a case study, we will consider the effect
of minimum wage raise. Proposition 206, approved by Arizona 
voters in November 2016, 
raised the state minimum wage from $\$ 8.05$ to $\$ 10.00$ on
 January 1 2017, followed by increases to $\$ 10.50$ in 2018, 
 $\$11.00$ in 2019, and $\$12.00$ in 2020 \citep{ballotpedia_prop206}. 
What would have been wages in Arizona had these measures not passed?

Synthetic control~\citep{abadie2010synthetic, abadie2015comparative, abadie2021using} finds non-negative weights 
$w_s$ summing to one such that $\oY(s_*,t)\approx\sum_{s\neq s_*} w_s\,\oY(s,t)$ for 
$t<t_*$. Assuming the treatment does not affect other units, the same 
combination serves as the counterfactual for $t\ge t_*$. 
In running example $s_*$ is Arizona and $\oY(s,t)$ can be any statistic of
wages in state $s$ at time $t$. In particular, Figure~\ref{fig:AZ_PA_SC_result} 
reports the 10th-percentile wage in Arizona and Pennsylvania.

We propose to view this problem as 
\emph{conditional generation for an unseen subpopulation}. We view each record as a sample from a joint distribution over unit $S$ (e.g., state), time $T$, covariates and outcome $Y$. The training data contain the treated unit $s_*$ only before $t_*$, and the post-intervention period $t\ge t_*$ only for the other units, so the combination $(S=s_*, T\ge t_*)$ is never observed (Arizona in 2017--2019, shaded in Figure~\ref{fig:AZ_PA_SC_result}). We train \contdiff{} on these data and generate the missing subpopulation by conditioning on $S=s_*$ and $T\ge t_*$. This has two advantages: $(i)$~unlike a single summary statistic per time step, it estimates the full counterfactual distribution, including individual-level effects; and $(ii)$~it is robust to reparametrizations of the outcome.

\paragraph{Data and method.} In order to investigate the effect of 
Proposition 206, we utilize wage records from the 2010--2019 Current
 Population Survey (CPS) data, made available by the NBER 
 \citep{nber_morg}.  After filters (see Appendix~\ref{app:CPSDataset}), the dataset comprises
  489,225 rows corresponding to employment records from all 
  50 states and DC. As a baseline, we construct a counterfactual 
  following the methodology of \cite{abadie2010synthetic}.
The resulting synthetic Arizona places 47.7\% of its weight on 
Virginia, 25.2\% on Texas, 14.6\% on North Dakota, 11.2\% on Wyoming, and 1.3\% on South Carolina. 
Our conditional generation approach consists in training \contdiff{} after removing
Arizona data from 2017 and after; see Appendix~\ref{app:AZMinwageDetails} for details.

\paragraph{Results.}
Figure~\ref{fig:AZ_PA_SC_result} we compare observed 10th-percentile 
wages in Arizona and Pennsylvania with
with the counterfactual prediction
provided by synthetic control and  conditional generation using \contdiff{}.

We observe that Arizona's observed 
tenth-percentile nominal hourly wage follows the statutory 
minimum almost exactly. 
As for the counterfactual, obviously there is no way of determining 
`ground truth.'  However, the conditional generation approach is in rough
agreement with synthetic control and provides a natural prediction.
 We repeat the procedure for Pennsylvania, which remained at
 the federal minimum throughout. Consistent with the design,
  the observed post-2016 wages match very closely with the one 
  predicted by \contdiff{}. We observe that, in this case, our prediction 
is somewhat more accurate than the one from synthetic control.

\paragraph{Subpopulation-level effects.}
Synthetic control produces one counterfactual summary
 statistic per year. Conditional generation instead samples 
 from the \emph{joint} counterfactual distribution 
 and produces records at the individual level.
  By further conditioning  on demographic/occupational 
  we can assess which subpopulations are more affected. As an example, in Fig.~\ref{fig:AZ_SC_indiv_occ} we stratify based on occupation (see Appendix~\ref{app:AZMinwageIndiv} for full analysis and Fig.~\ref{fig:AZ_SC_indiv_edu} for stratification based on education.)

 \begin{figure}[t]
    \centering
    \includegraphics[width=0.99\textwidth]{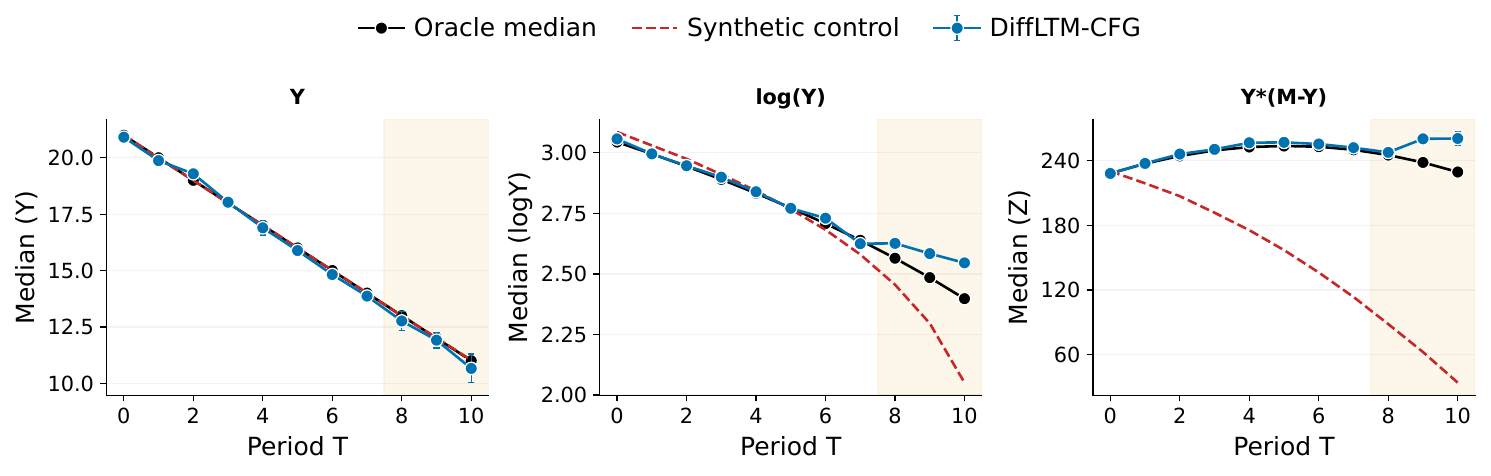}
    \caption{Counterfactual analysis using synthetic control and conditional generation
    (based on \contdiff{}) for data generated according to a structural equations model.
    The prediction for $f(y)=y$, $\log y$, $y(\max y-y)$.
    }
    \label{fig:sc_reparam_median}
\end{figure}

\paragraph{Robustness to reparameterization.}
Synthetic control estimates heavily rely on the assumption of 
a linear relationship between outcomes $Y(s,t)$ for different units.
In particular, this assumption is fragile under reparametrizations
of the underlying variable. Consider ---as in application example--- the case in which 
$\oY(s,t)$ is the $\alpha$ quantile of
a variable of interest $Y$ conditional on $s,t$,
and define $X = f(Y)$ for a nonlinear function $f$, 
and the corresponding quantile $\oX(s,t)$.
If the linear relationship $\oY(s_*,t)\approx \sum_{s\neq s_*}
w^Y_{s} \oY(s,t)$ holds approximately then, generically,
$\oX(s_*,t)\approx \sum_{s\neq s_*} w^X_{s} \oX(s,t)$ will not hold. 

Figure \ref{fig:sc_reparam_median} demonstrates in a synthetic 
example that conditional generation is robust to reparametrizations, while
synthetic control is not.  In this example, data take the form $(S,T,Y)$, with $S$ the 
unit, $T$ the time, and $Y$ the outcome: shaded 
area corresponds to the held-out data $T\ge 8$, $S=6$.
We attempt to predict $\mathrm{median}(f(Y))$ in the held-out domain.
Both synthetic control and \contdiff{}-CFG match the ground truth for  $f(y)=y$.
Under nonlinear $f$, synthetic control's linear bias yields inaccurate predictions,
 while \contdiff{}-CFG stays close to the ground truth (see also Appendix~\ref{app:SCReparamDetails}).

\vspace{0.2cm}

We conclude by emphasizing that --although promising-- 
the above  is a preliminary study
of the use of conditional generation for counterfactual analysis.
Turning this into a reliable methodology is deferred to future work.

\subsection{Mincer returns to education} \label{app:Mincer}

A simple first test our generative model is to fit a linear
regression on (real) test data and synthetic data and ask whether the 
two linear models yield similar coefficients.
For the ACSLabour dataset in particular,
a natural test is to reproduce the celebrated Mincer earnings equation~\citep{mincer1974schooling} 
(see also~\citep{lemieux2006mincer}). We use the quadratic specification
\[
\log(W) = \beta_0+\beta_1X+\beta_2X^2+\beta_3S+\epsilon,
\]
where $W$ is wage, $X$ is potential years of experience, $S$ is
years of schooling, and $\epsilon$ is exogenous noise. We obtain $S$ by mapping the ACS education categories to
completed years of schooling and define potential experience as $X=\texttt{AGEP}-S-6$, 
assuming that schooling begins at age six and that work begins immediately
after schooling. Rows for which $X<0$ are excluded.

\begin{figure}[t]
      \centering
      \includegraphics[width=0.5\linewidth]{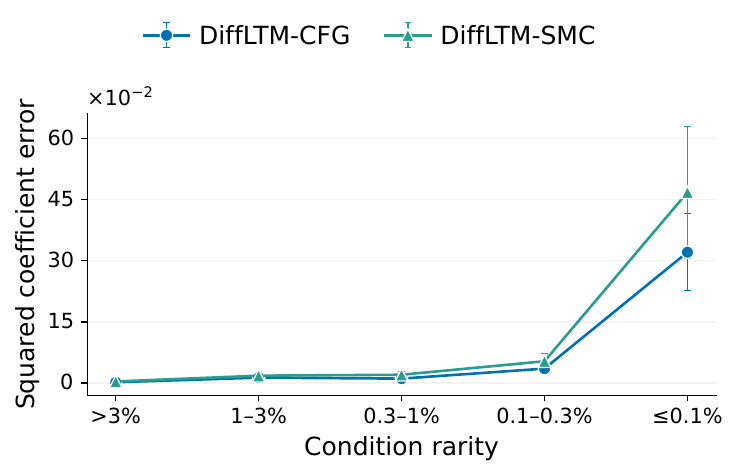}
      \caption{Squared error between Mincer coefficients fitted on conditionally generated sub-populations and real data slices. 
      Here we used $n^{\ssynth}=n^{\stest} \approx \min\left(n_{\stest}^{\mathcal{C}},100{,}000\right)$, where $n_{\stest}^{\mathcal{C}}$ denotes the number of eligible test rows satisfying $\mathcal{C}$.
      }
      \label{fig:mincer_coeff_mse_v_rarity}
\end{figure}

  \begin{figure}[ht]
      \centering
      \includegraphics[width=\linewidth]{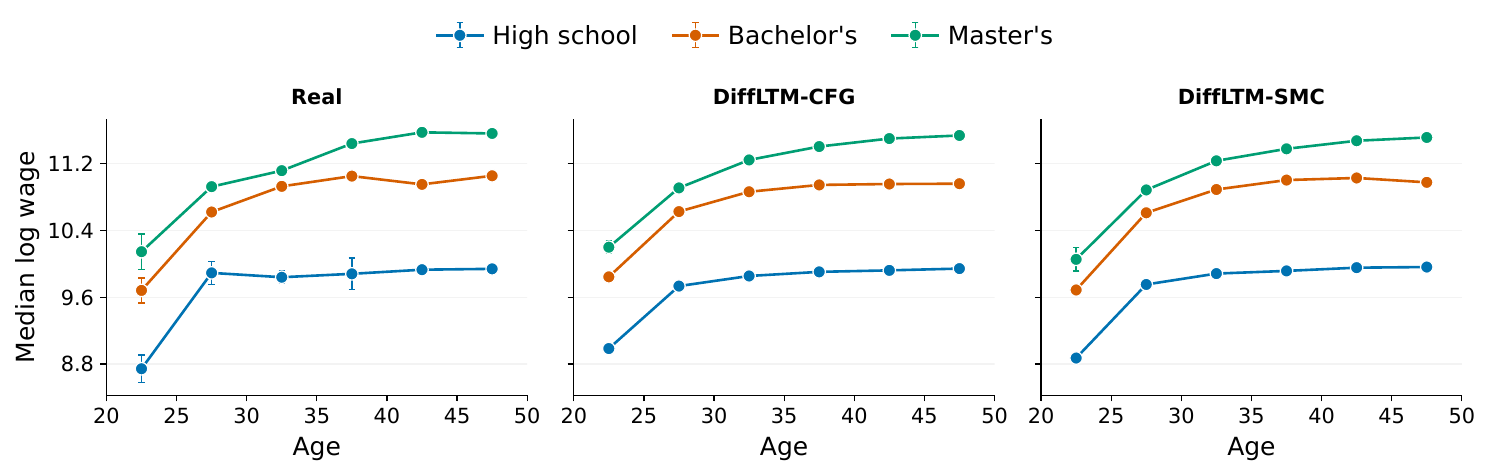}
      \caption{Returns to education (China-born, Female, Private Industry). For China-born female private-sector workers, both methods reproduce 
      the education ordering and the concave age--earnings profile of the real data. 
      }
      \label{fig:mincer_return_to_edu}
  \end{figure}

We write $\bbeta=(\beta_0,\beta_1,\beta_2,\beta_3)^\top$ for the vector of coefficients,
and fit it by ordinary least squares (OLS), both on real (test) data and on synthetic data (generated using \contdiff{}). As a first test, we consider unconditional generation and fit the Mincer equation on the (subsampled) test data and on (uncondtionally generated) synthetic data with $n_{\stest} = n_{\ssynth} \approx 100,000$. Averaging the fitted coefficients across three seeds yields
\begin{align*}
  \overline{\bbeta}_{\mathrm{real}} &= \begin{bmatrix} 7.64071 & 0.08891 & -0.001439 & 0.12709 \end{bmatrix}\\
  \overline{\bbeta}_{\mathrm{syn}} &= \begin{bmatrix} 7.70948 & 0.08832 & -0.001430 & 0.12171 \end{bmatrix}.
\end{align*}
The two linear models are very close to each other, confirming that \contdiff{} can reproduce the Mincer equation.
Note in particular the
positive schooling coefficient as well as the positive linear
and negative quadratic experience coefficients: these reproduce the known results on 
$(i)$~higher education translating to higher earnings at each level, and $(ii)$~the concave
wage-experience profile.

Next, to carry out a more stringent test, we consider subpopulations defined by specific conditions.
For each such subpopulation, we first compute OLS on  the corresponding real data  from the held-out test set to obtain
$\beta_{\mathrm{real}}$. We then generate synthetic rows under the 
same conditions using each of CFG and SMC and fit OLS to obtain $\beta_{\mathrm{CFG}}$ and $\beta_{\mathrm{SMC}}$. 
Similarly to the fidelity vs rarity experiment (Figure~\ref{fig:cfg_acslabour_cond_shape_all}), for each condition $\mathcal{C}$ OLS is fitted on matched eligible synthetic and test samples of size $n_{\ssynth}=n_{\stest}=\min\left(n_{\stest}^{\mathcal{C}},100{,}000\right)$\footnote{Eligibility filtering, i.e. removing rows with $X < 0$ may change $n_{\ssynth}$ but this change is minor when it occurs.}, where $n_{\stest}^{\mathcal{C}}$ denotes the number of eligible test rows satisfying $\mathcal{C}$ ($152-100,000$ rows per slice across the final experiment).
Finally, we compute $\|\beta_{\mathrm{real}}-\beta_{\mathrm{syn}}\|_2^2$ 
for each method and plot it as a function of subpopulation rarity in Figure~\ref{fig:mincer_coeff_mse_v_rarity}.
The considered slices identical to the equality slices from Figure~\ref{fig:cfg_acslabour_cond_shape_all}, 
with the  exclusion of those defined using the wage, schooling, or age columns.

Conditional generation should reproduce statistical relationships that
differ substantially between a specific subpopulation and the full population. One such example is
the subpopulation reporting no English proficiency ($\approx 0.95\%$ of the dataset), where the fitted coefficients on real and (conditionally generated) synthetic data averaged across three seeds are given by
\begin{align*}
\overline{\bbeta}_{\mathrm{real}}&=\begin{bmatrix}9.15633&0.03147&-0.000418&0.01478\end{bmatrix}\quad \\
\overline{\bbeta}_{\mathrm{CFG}}&=\begin{bmatrix}9.16444&0.03115&-0.000414&0.01384\end{bmatrix}\quad  \\
\overline{\bbeta}_{\mathrm{SMC}}&=\begin{bmatrix}9.19665&0.03356&-0.000457&0.01182\end{bmatrix}
\end{align*}
The schooling coefficient falls from $\approx 0.13$ in the full
dataset to $\approx 0.01$ in this subpopulation, while the initial experience
coefficient falls from $\approx 0.09$ to $\approx 0.03$; suggesting that the return on education and experience for this group is 
substantially lower than that for the full population. Both CFG and SMC successfully recover this rare subgroup-specific
conclusion.

As a further example, we examine the implied returns-to-education curves for China-born female private-sector workers. We condition on this subpopulation, in turn, on high-school, bachelor's, and master's education, while jointly generating age and log wage. We then plot median log wage across five-year age bins to test whether the synthetic data preserve both the ordering of earnings by education and the concave age-earnings profiles observed in the real data; as illustrated in Figure~\ref{fig:mincer_return_to_edu}, this curve is successfully reproduced by both CFG and SMC.


\section{Asymptotic theory} \label{sec:Theory}

\subsection{Optimality of CE loss for continuous diffusions}
\label{subsec:Optimality-CE-Loss}

The empirical superiority of CE loss over MSE loss (cf. Section~\ref{sec:diffusion_arch}) 
is surprising for several reasons: $(i)$~At infinite sample size, 
minimizing MSE loss in denoising diffusions is equivalent to minimizing the KL divergence between 
the generated and target distribution 
\citep{song2021maximum,chen2022sampling}; $(ii)$~It was proven by \cite{chewi2026ddpm}
that ---in a parametric setting--- MSE loss results in the minimum 
parameter estimation error, 
under the classical asymptotics in which the sample size
$n$ diverges with the number of parameters fixed.

Our main theoretical result is that these conclusions should be revised 
if the denoisers learnt at different noise levels are
not tied in the Bayes-optimal way (as is usually the case for neural denoisers).
In other words  a somewhat minimal form of misspecification
leads to MSE becoming suboptimal.

For this analysis we consider a simple categorical variable
taking $k$ values, with probabilities $q(\sigma)$,
 $\sigma\in\{0,1,\dots,q-1\}$. This variable is represented via 
 embeddings $\bx(\sigma)\in\reals^d$, $\|\bx(\sigma)\|=1$. We learn
 a general diffusion indexed by the continuous time variable 
 which increases from $t=0$ (pure noise), to $t=T$ (sample from the target distribution):
 \begin{align}
 \bX_t =a(t)\bX+\tau_0(t) \bB_{s(t)}\, .\label{eq:general-diffusion-0}
 \end{align} 
 Here $\bX\sim \sum_{\sigma} q(\sigma)\delta_{\bx(\sigma)}$ is distributed
 according to the target, independent of the Brownian motion $\bB$, $a(\,\cdot\,)$ increases from $0$ to $1$,
 while $\tau_0$ goes from $\tau_0(0)$ to $\tau_0(T)=0$.
 As usual, the process \eqref{eq:general-diffusion-0} is equivalently
 defined by an SDE of the form 
\begin{align}\label{eq:general-diffusion-SDE}
  \de \bX_t = \bff(\bX_t,t)\,\de t + \sqrt{g(t)} \de \bW_t,
  \end{align}
  where $\bW_t$ is a standard Brownian motion and
  \begin{align}
  \bff(\bx,t) &= \tau_0(t) \left(\frac{\de }{\de t} \frac{a(t)}{\tau_0(t)}\right)\bm(\bx,t)+
  \frac{\dot{\tau}_0(t)}{\tau_0(t)}\bx\, ,\;\;\;\;\;\;\;
  g(t) = \tau_0(t)^2\dot{s}(t)\, ,\\
  \bm(\bx,t) &:= \E[\bX |  \bX_t = \bx] \, .
  \end{align}
  The one-time marginals of $\bX_t$ are the same as the ones of the probability flow ODE:
\begin{align}
\de \bZ_t &= \bh(\bZ_t,t)\,\de t \, ,\label{eq:ProbFlow-Gen}\\
\bh(\bz,t) &= \bff(\bz,t) - \frac{1}{2} g(t) \nabla_{\bz} \log p(\bz,t) \, .
\end{align}
 As standard in practice, we use probability flow for generation.

 We are given i.i.d. data triples
$\{(t_i,\sigma_i,\bY_i):i\le n\}$.  Their common law is specified by
$\sigma_i\in\{0,1,\dots,k-1\}$, $\prob(\sigma_i=\sigma)=q(\sigma)$,
$t_i\sim \lambda(t)\cdot\de t$, independently of $\sigma_i$, where $\lambda$ is a
probability density on $[0,T]$, and
\begin{align}
\left.\bY_i\right|_{(t_i,\sigma_i)}
\sim \normal\big(\bx(\sigma_i)\,a(t_i),\tau(t_i)^2\id_d\big),
 \quad \quad i\le n\, . 
\end{align}
We learn a generative model of the form \eqref{eq:ProbFlow-Gen}
where $\bm$ is replaced by 
\begin{align}
\hbm(\bx;\hbtheta(t))
&= \sum_{\sigma=0}^{k-1}\hq_{\sigma}(\bx;\hbtheta(t))\, \bx(\sigma)\, ,\\
\hq_{0}(\bx;\btheta) & = \frac{1}{z(\bx,t,\btheta)}\, ,\;\;\;\;
\hq_{\sigma}(\bx;\btheta)  = \frac{1}{z(\bx,\btheta)}\, 
\exp\big\{\theta_{\sigma,0}+\<\btheta^{+}_{\sigma},\bx\>\big\}\, ,\;\;\;
\sigma\in\{1,\dots,k-1\}\, .
\label{eq:conditional-distribution-Gen}
\end{align}
and $z(\bx,\btheta)  = 1+\sum_{\sigma=1}^{k-1}
\exp\{\theta_{\sigma,0}+\<\btheta^{+}_{\sigma},\bx\>\}$ is a normalization
constant.

Crucially,  the parameters 
$\btheta(t)=(\btheta_1(t),\dots,\btheta_{k-1}(t))$
are allowed to vary with $t$, as is usually the case for 
denoising diffusions, where a suitable encoding of $t$ is 
used as input to the denoiser.

In our theoretical analysis, we consider a simplified setting 
whereby $\btheta(t)$ is learned via the
following weighted empirical risk minimization (ERM):
\begin{align}\label{eq:SmoothERM}
\hbtheta_b(t)
:= \argmin_{\btheta}\hR_{b,t}(\btheta), \;\;\;\;\;\;
\hR_{b,t}(\btheta) := \frac{1}{n}\sum_{i=1}^{n}
w_b(t_i-t)\, \ell_{\sigma_i}(\bY_{i};\btheta) \, .
\end{align}
where $w_b(t) = w_1(t/b)$ is an unnormalized weight function and $b$ is a bandwidth parameter
and $t\mapsto K(t)\subseteq E$ specifies
the time-dependent optimization domain. 

Denote by $\hq^{\sdiff}_n(\sigma)$  the probability of generating class $\sigma$
via the probability flow ODE, using the estimated parameters $\hbtheta_b(t)$.
We prove a central limit theory asymptotics for 
$\sqrt{n}(\hbtheta_b(t)-\btheta^*(t))$ (where $\btheta^*(t)$
are the `optimal' parameters) and, as a consequence, for 
$\sqrt{n}(\hq^{\sdiff}_n-q)$. Further, we prove that the limiting variance of the 
error is minimized for CE loss.
\begin{theorem}[Informal]
\label{thm:Overall-Limit-Gen-Informal}
Under regularity assumptions, further assume that, as $n\to\infty$,  
\begin{align}\label{eq:UndersmoothingFirst}
n^{1/4}b_n\to 0 \qquad \mbox{and} \qquad \frac{\sqrt n\,b_n}{\log(1/b_n)}\to\infty\, .
\end{align}
Then the probability distribution of the of generated 
samples $\hq^{\sdiff}_n$ satisfies
\begin{align}\label{eq:CLT-Gen}
\sqrt n(\hq^{\sdiff}_n-q)&\Rightarrow\normal_k(0,\bV)\, ,
\end{align}
where the covariance matrix $\bV\in\reals^{k\times k}$ is given explicitly in terms of the 
loss function $\ell$.

Further $\bV$ is minimized (with respect to the positive
semidefinite order) for CE loss.
\end{theorem}
A formal statement, with full specification of the assumptions, is given in 
Appendix~\ref{app:Asymptotics}.
We point out that omitted assumptions are analogous 
to the classical ones in low-dimensional
asymptotic statistics 
\cite[Chapter 5]{van2000asymptotic}.
In contrast, the undersmoothing
condition \eqref{eq:UndersmoothingFirst} is 
important and corresponds to a regime in which
the $t$-dependence of $\hbtheta(t)$ is only weakly penalized
and hence $\hbtheta(t)$ is asymptotically unbiased.
It is possible to generalize our results to 
other regimes (in which $\hbtheta(t)$ is 
smoother and hence somewhat biased) but calculations 
are  less transparent, and we omit them.

While there is some subtleties in the actual proof,
the intuition for this results can be described succinctly.
Let $\btheta^*(t)$ be the `ideal' parameters in the logistic model
\eqref{eq:conditional-distribution-Gen}. Explicitly: 
\begin{align} 
\theta^*_{\sigma,0}(t) :=\log \frac{r(0)}{r(\sigma)}\, ,
\;\;\;\;\;\;\;\;
\btheta^{*,+}_{\sigma}(t)
:=\frac{a(t)}{\tau(t)^2}\{\bx(\sigma)-\bx(0)\} \, .
\end{align}
For large $n$, the estimated parameters $\hbtheta_b(t)$ are close to the ideal ones,
and hence a first order Taylor approximation holds
for the sampling probabilities of the model :
\begin{align}
\hq^{\sdiff}_n(\sigma)-q(\sigma) =
\int_0^T\<\bL_{\sigma}(t),\hbtheta_b(t)-\btheta^*(t)\> +
O\big(\|\hbtheta^b-\btheta^*\|^2\big)\, ,
\end{align}
where $\bL_{\sigma}(t)$ is a certain deterministic matrix, see Proposition 
\ref{propo:Lin-Effect} for a formal statement.

For large $n$, classical asymptotics implies that $\hbtheta^b-\btheta^*$
is approximately normal. Further, under the undersmoothing 
condition is also unbiased and `delta-correlated.' Roughly speaking,
letting $\bDelta(t):= \hbtheta^b(t)-\btheta^*(t)$,
\begin{align}
\E[\bDelta(t)]\approx 0\, ,\;\;\;\;\;\;\;
\E\big[\bDelta(t)\otimes\bDelta(s)]\approx \frac{1}{n\lambda(t)}\bQ^*(t;\ell)\, \delta_b(t-s)\, ,
\end{align}
where $\delta_b(t)$ integrates to one, and vanishes for $|t|> 2b$.
(See Proposition \ref{propo:AsympGen} for a mathematically precise claim.)
Further $\bQ^*(t;\ell)$ is the classical variance of the M-estimator of $\btheta^*(t)$
with loss function  $\ell$ under the mixture Gaussian model
\citep{van2000asymptotic}.

Putting together the above suggests that Eq.~\eqref{eq:CLT-Gen}
should hold with (as $n\to\infty$, $b\to 0$)
\begin{align}
    V_{\sigma,\tau} = \int_0^T\frac{1}{\lambda(t)}\<\bL_{\sigma}(t),
    \bQ^*(t;\ell)[\bL_{\tau}(t)]\>\, \de t\, .
\end{align}
By classical optimality of the maximum likelihood estimator
\citep{van2000asymptotic},   
$\bQ^*(t;\ell)$ is minimized by CE loss 
at any given $t$, and hence CE loss is overall optimal.

This conclusion is  is contrast
to earlier work \citep{chewi2026ddpm} where the denoisers 
at different $t$ are carefully tied.

\subsection{Comparing continuous to discrete sampling}
\label{sec:discrete-continuous-theory} 

Theorem \ref{thm:Overall-Limit-Gen-Informal} establishes
optimality of CE loss under the stated setting: 
continuous diffusion sampling and under-smoothed denoiser
estimation, cf. Eqs. \eqref{eq:conditional-distribution-Gen} and
  \eqref{eq:SmoothERM}.
  It is natural to ask whether the asymptotic variance 
  of Eq.~\eqref{eq:CLT-Gen} (for the optimal CE loss) matches 
  the minimum variance of any generative method. 
In this case, an optimal procedure is to estimate the probabilities
via maximum likelihood, and then generating 
$\sigma\sim \hq^{\sML}_n(\,\cdot\,)$. 
It is elementary that
\begin{align}\label{eq:CLT-MLE}
  \sqrt n(\hq^{\sML}_n-q)\Rightarrow\normal_k(0,\bV^{\sML})\, ,
  \qquad\qquad V^{\sML}_{\sigma,\tau}:=q(\sigma)\bfone_{\sigma=\tau}
  -q(\sigma)q(\tau)\, .
\end{align}
Using the explicit expression 
of $\bV$ in Appendix \ref{app:Asymptotics}, it is relatively 
simple to give examples in which $\bV\succ \bV^{\sML}$ strictly,
already for $k=2$. 

We ask whether the suboptimality
is bounded, i.e. there is a constant factor $c(k)>0$
and a loss function $\ell$ such that $\bV\preceq c(k)\bV^{\sML}$ 
for any distribution $q$ on $\{0,1,\dots,k-1\}$. 

This question is answered in the affirmative in the case $k=2$
(in this case we relabel the two classes as $\sigma\in\{+1,-1\}$.)
Recall that the MLE variance is $v^{\sML}(q) = q(1)(1-q(1))$.
\begin{theorem}[Bounded relative inefficiency]
\label{thm:Bounded-Suboptimality}
Consider a continuous diffusion model learned according to the setting of Section \ref{subsec:Optimality-CE-Loss},
for $k=2$,  $\sigma\in\{+1,-1\}$, $d=1$, and embeddings $x(-1)=-1$, $x(1)=1$, with
probabilities $q(1)$ and $q(-1)=1-q(1)$.
Under the same technical assumptions on the loss $\ell$ as for 
Theorem \ref{thm:Overall-Limit-Gen-Informal}, as well as the undersmoothing
assumption \eqref{eq:UndersmoothingFirst},
we have, for the optimal choice of the time-sampling density $\lambda$,
$\sqrt{n}(\hq^{\sdiff}_n(1)-q(1))\Rightarrow\normal(0,v(\ell,q))$  (where we made explicit the dependence of the variance
on the loss and the probability distribution $q$). 

Relabel the two classes as $\sigma\in\{-1,+1\}$. Consider a loss of the
form $\ell_{\sigma}(Y;\btheta)=\varphi_{\sigma}(\theta_0+\theta_1Y)$,
define $\pi_\sigma(u):=(1+\sigma\tanh u)/2$ and
$h_{\ell}(u):=\sum_{\sigma\in\{-1,+1\}}\pi_\sigma(u)
\varphi_{\sigma}''(u)$, and assume
\begin{align}
  \frac{1}{L}\log h_\ell(Lz)&\longrightarrow -\alpha|z|\, .\label{eq:LossAtInfty}
\end{align}
In the boundary case $\alpha=2$, we additionally assume the tail regularity
condition
\begin{align}
c_h:=\lim_{u\to-\infty}e^{-2u}h_\ell(u)\in(0,\infty).
\label{eq:LossAtInftyBoundary}
\end{align}
Then we have the following:
\begin{align}
\label{eq:AlphaLarge} 2<\alpha\le 4 &  \Rightarrow  \sup_{q(1)\in(0,1)}\frac{v(\ell,q)}{v^{\sML}(q)} =\infty\, ,\\
\label{eq:Alpha2} \alpha = 2 &  \Rightarrow  
\sup_{q(1)\in(0,1)}\frac{v(\ell,q)}{v^{\sML}(q)}<\infty\, .
\end{align}
\end{theorem}
The general regularity assumptions of Theorem \ref{thm:Overall-Limit-Gen-Informal}
(which also apply to Theorem \ref{thm:Bounded-Suboptimality})
are stated in Section \ref{sec:N_Asymp-Gen}, while the proof
of this theorem is given in Appendix \ref{sec:proof-discrete-continuous-theory}.

We note that CE loss satisfies Eq.~\eqref{eq:LossAtInfty}, with
$\alpha=2$, whence the `asymptotic relative inefficiency'
$v(\ell,q)/v^{\sML}(q)$  is bounded. In contrast, the MSE
loss satisfies the assumption with $\alpha = 4$, and therefore 
its inefficiency is unbounded.

\section{Discussion and limitations}

We presented the design of a generative model for tabular data, based on continuous diffusions. The model has state of the art 
capabilities in generation and conditional generation. An important
methodological element is the use of learned embeddings and CE loss for categoricals, for which we provide a mathematical rationale. Finally, we trained our model on datasets that are one order of magnitude larger than previous published work, and demonstrated the ability to
recover econometric patterns, providing complementary information to earlier validations.

An important missing element is cross-datasets modeling.
In preliminary work, we trained a cross-datasets version of our 
architecture, but defer this to a future report. 
An important limitation to our econometrics application is that we did not develop rigorous statistical guard-rails for our approach. 
We consider it an important avenue for future work.

\section*{AI use statement}
We used AI assistants to polish writing, support literature review, 
and assist with code development. All references included in
 the paper were manually verified by the authors, and all 
 AI-generated code was reviewed and tested for correctness.
  We did not use AI tools to write code for the core methodological c
  omponents of the work. We additionally used generative
   AI tools for drafting the proofs in Appendix~\ref{app:Asymptotics}. The proofs were reviewed and edited by the authors. We take responsibility for the final content of this work, including text, claims or artifacts produced with the aid of generative AI.

\section*{Reproducibility statement}
We use either publicly available datasets or provide complete details of the datasets constructed in this work in Appendix~\ref{app:Datadetails}. Model architectures and hyperparameter settings are documented in Appendix~\ref{app:uncondGenDetails} and Appendix~\ref{app:CondGenDetails}. All experiments are repeated across multiple random seeds, and we report means together with standard errors (or standard deviation) throughout the paper.

\bibliographystyle{amsalpha}
\newcommand{\etalchar}[1]{$^{#1}$}
\providecommand{\bysame}{\leavevmode\hbox to3em{\hrulefill}\thinspace}
\providecommand{\MR}{\relax\ifhmode\unskip\space\fi MR }
\providecommand{\MRhref}[2]{%
  \href{http://www.ams.org/mathscinet-getitem?mr=#1}{#2}
}
\providecommand{\href}[2]{#2}

\newpage
\appendix

\section{Datasets details} \label{app:DataMetrics}
 \label{app:Datadetails}

{\bf Construction of ACSLabour.} We first provide a detailed description of the construction of the ACSLabour dataset. As mentioned in Section~\ref{subsec:Datasets}, all datasets used herein are derived from the ACS PUMS. The ACS is a survey administered by the U.S. Census Bureau, while PUMS provides de-identified person-level survey records containing demographic, geographic, educational, employment, and income information.
Folktables~\citep{ding2021retiring} constructs standardized machine-learning datasets from ACS PUMS by specifying a population filter, a set of retained columns, and a prediction target and two of our utilized datasets (ACSPubCov and ACSMobility) are derived from this work, concatenated over the $5$-year survey period $2014-2018$. 

In addition to these, we construct \emph{ACSLabour}, a dataset derived from the ACS PUMS which proves particularly useful for econometric experiments considered in this work. The existing folktables dataset ACSIncome discretizes income into a binary prediction target whereas ACSLabour retains continuous log wage, and is designed to build predictive models for this outcome. This choice enables analyses such as regressions on the continuous log wage (see Appendix~\ref{app:Mincer}). We also retain a broader set of labor and demographic columns, including high-cardinality occupation, industry, geographic, and birthplace variables a detailed list of which is provided in Table~\ref{table:ACSLabourDescription}. ACSLabour therefore is both a regression benchmark and a stress test of whether the generative model can recover rare, semantically meaningful subpopulations within a large table.

\begin{table*}[ht]
\centering
\footnotesize
\begin{tabular}{llp{0.68\textwidth}}
\hline
Column & Type & Description \\
\hline
\texttt{AGEP}      & Numerical   & Age in completed years. \\
\texttt{WKHP}      & Numerical   & Usual hours worked per week during the previous 12 months. \\
\texttt{log\_wage} & Numerical   & Regression target: natural logarithm of annual wage and salary income in constant 2018 dollars, computed as $\log(\texttt{WAGP}\times\texttt{ADJINC}/10^6)$. \\
\texttt{ST}        & Categorical & State or territory of residence. \\
\texttt{PUMA}      & Categorical & Public Use Microdata Area of residence; codes are unique only within state. \\
\texttt{SEX}       & Categorical & Sex. \\
\texttt{RAC1P}     & Categorical & Recoded detailed race. \\
\texttt{HISP}      & Categorical & Detailed Hispanic-origin recode. \\
\texttt{SCHL}      & Categorical & Highest level of educational attainment. \\
\texttt{POBP}      & Categorical & Place of birth. \\
\texttt{CIT}       & Categorical & Citizenship status. \\
\texttt{ENG}       & Categorical & Self-reported ability to speak English. \\
\texttt{OCCP}      & Categorical & Occupation code. \\
\texttt{INDP}      & Categorical & Industry code. \\
\texttt{COW}       & Categorical & Class of worker. \\
\texttt{MAR}       & Categorical & Marital status. \\
\texttt{HICOV}     & Categorical & Health-insurance coverage status. \\
\texttt{WKW}       & Categorical & Interval-coded number of weeks worked during the previous 12 months. \\
\hline
\end{tabular}
\caption{Modeled columns in ACSLabour. Categorical variables retain their original ACS PUMS codes; code definitions and semantic labels are provided in the ACS 2014--2018 PUMS Data Dictionary~\cite{census_pums_dictionary_2014_2018}.}
\label{table:ACSLabourDescription}
\end{table*}
From the full $2014-2018$ PUMS dataset, for ACSLabour we retain only those rows satisfying: 
\begin{itemize}
    \item $\texttt{AGEP} \geq 16$;
    \item $\texttt{ESR} \in \{1,2\}$, corresponding to civilians employed either at work or temporarily absent from work;
    \item \texttt{WAGP} is observed and strictly positive; and
    \item $\texttt{WKHP} > 0$.
\end{itemize}
{\bf Datasets summary.} Table~\ref{table:ACSDataSummary} summarizes the size, column types, and designated prediction task for each of the three datasets utilized in our experiments.
\begin{table}[ht]
\centering
\small
\begin{tabular}{lrrrll}
\hline
Dataset & Rows & Numerical & Categorical & Target & Task \\
\hline
ACSLabour & $6{,}892{,}317$ & $3$ & $15$ & \texttt{log\_wage} & Regression \\
ACSPubCov \citep{ding2021retiring} & $5{,}916{,}565$ & $2$ & $18$ & \texttt{PUBCOV} & Classification \\
ACSMobility \citep{ding2021retiring} & $2{,}041{,}325$ & $4$ & $17$ & \texttt{MIG} & Classification \\
\hline
\end{tabular}
\caption{Summary of the ACS-derived datasets}
\label{table:ACSDataSummary}
\end{table}

{\bf Preprocessing.} All preprocessing transformations are fit on the training partition and then applied unchanged to the validation and test partitions. Numerical columns are quantile transformed, while categorical columns are integer-encoded. We use a 70/10/20\% train/validation/test split, stratified by the target for classification datasets and unstratified for regression datasets.
\section{Unconditional generation: Experiment details} \label{app:uncondGenDetails}

In this section, we provide a full description of the experiment setup leading to the results reported in Tables \ref{tab:unconditional-comparisons} and \ref{tab:unconditional-architectures}. We recall that the model obtained by the architecture/training choices of $(i)$~learnt embeddings, $(ii)$~logits mode and $(iii)$~CE loss is utilized in Table \ref{tab:unconditional-comparisons} to compare against baselines. 
\subsection{Architecture parameters} \label{app:denoiserArch}

When used, the learned embedding matrix is implemented using PyTorch's \texttt{nn.Embedding}. In all such experiments we set the embedding dimension to $d_{\mathrm{emb}}=16$. Numerical columns are transformed using a quantile transformation. The concatenated noisy representation $\boldsymbol{r}_t$ is first linearly projected to a hidden width of $d_h=1024$, which is subsequently processed by a shared MLP comprising $L=5$ fully connected hidden layers of width $d_h=1024$ and using SiLU activations. The diffusion time $t$ is encoded using a $1024$-dimensional sinusoidal embedding, which is further processed by two fully connected layers of width $1024$ with SiLU activations. This time embedding is added to the projected input representation before it is passed through the shared MLP. On the ACSLabour dataset, the unconditional model contains approximately $9.84$M parameters, including the learned embedding tables. 

\subsection{Training} \label{app:uncondTrainDetails}

All the \contdiff{} variants are trained for $2000$ epochs with a batch size $8192$ and AdamW ($\beta _{1}=0.9$ and $\beta _{2}=0.999$). We use an initial learning rate of $3 \times10^{-4}$ with a cosine decay schedule. We use the noise schedule defined above with $(\sigma_{min}, \sigma_{max}, \rho) = (0.002, 80, 7)$. Unconditional and CFG-based conditional generation use $100$ Euler steps. Following earlier work~\citep{song2021score}, we use exponential moving average(EMA) of denoiser parameters with decay of $0.997$ for evaluation. The final checkpoint is selected using lowest validation loss.


All experiments were run on a single node with 192 CPU cores and 8 NVIDIA A100 80 GB GPUs, using Ubuntu 22.04 and CUDA 13.0. We train each model using three different random seeds and use bfloat16 (bf16) precision for all training runs. Each individual run, including both training and evaluation, takes less than 12 GPU-hours.

\subsection{Baselines implementation} \label{app:uncondBaselinesDetails}
We benchmark \contdiff{} against three recent state-of-the-art generative models for tabular data. Two of these use continuous diffusion (TabRep~\citep{si2025tabrep} and TabCascade~\citep{mueller2026cascaded}), while TabDiff~\citep{shi2025tabdiff} is based on mixed diffusion.

To ensure a fair comparison, we use the official implementations released by the authors:
\begin{itemize}
    \item For TabRep, we use the flow-matching variant, TabRep-Flow, available at \url{https://github.com/jacobyhsi/TabRep}.
    \item For TabCascade, we use the implementation available at \url{https://github.com/muellermarkus/tabcascade}.
    \item For TabDiff, we use the implementation from \url{https://github.com/minkaixu/tabdiff}. We observe that, with the default recommended learning rate, training becomes unstable after a few optimization steps, with gradient norms increasing sharply and failing to recover. We therefore use a lower learning rate of $3\times10^{-4}$ together with gradient-norm clipping at $1.0$, which yields stable training and substantially better performance. We retain the original denoiser architecture; in contrast, some prior comparisons modify TabDiff by removing its Transformer layers.    
\end{itemize}

Importantly, we use the exact same training, validation, and test splits as for \contdiff{}. We train each model end-to-end using its official implementation and generate a synthetic dataset containing $n_{\text{train}}$ rows, matching the size of the real training set. Regardless of the generation method, all synthetic datasets are evaluated using the same standardized evaluation setup described in the subsequent Appendix \ref{app:MetricsConfIntervals}. 

We note that the default training budgets used in the official implementations are substantially smaller, potentially due to their evaluation on smaller datasets. To match the training budget, we therefore train all models for $2000$ epochs with batch size of $8192$ and select the checkpoint with the lowest validation loss.

{\bf Training time comparison.} To quantify the training-efficiency gains of \contdiff{} over TabDiff, we conduct an isolated profiling experiment on the ACSLabour dataset. For both models, we measure the wall-clock time over 100 gradient steps after 20 warm-up steps, using a batch size of 8,192 on the same NVIDIA A100-SXM4-80GB GPU. Averaged over five runs, TabDiff requires $78.03$ ms per gradient step, compared with $34.98$ ms for \contdiff{}. Using BF16, as in all of our experiments in this paper, further reduces the \contdiff{} step time to $14.40$ ms, leading to a $5.4\times$ speedup over TabDiff.

\subsection{Metrics computation} \label{app:MetricsConfIntervals}

Upon training the model, we generate $n_{\mathrm{train}}$ (i.e. size of the training data) synthetic rows for each method. To compute the fidelity metrics (Shape, Trend, and Conditional Shape), we subsample this to $n_{\mathrm{test}}$ synthetic rows and compare them against the held-out real test set. 

The MLE metrics involve randomness in fitting the downstream prediction model; we account for this by repeating both TRTR and TSTR evaluations $10$ times. More precisely, in each round we sample $n_{\mathrm{test}}$ rows without replacement from the real training set for TRTR and from the synthetic pool for TSTR, fit a fresh XGBoost model, and evaluate it on the full real test set. We separately average the $10$ TRTR and $10$ TSTR scores and compute their ratio to obtain the final MLE score. 

For all metrics in Table~\ref{tab:unconditional-comparisons} and Table~\ref{tab:unconditional-architectures}, we report the mean and standard deviation across three training seeds.
\section{Conditional generation: Experiment details} \label{app:CondGenDetails}
We now provide a description of the experimental setup from Figure \ref{fig:cfg_acslabour_cond_shape_all}, as well as all the additional results on conditional generation in Appendix \ref{app:AdditionalResults}. 
\subsection{CFG} \label{app:CondGenCFG}
The CFG denoiser architecture is described in Figure~\ref{fig:CFG_arch_diagram}. 
The training procedure is identical to that reported in Appendix~\ref{app:uncondTrainDetails} with the following choices specific to the CFG variant of \contdiff{}. We use CFG dropout probability $p_{\mathrm{dropout}}=0.15$, $p_{\mathrm{rng}}=0.5$, and $p_{\mathrm{lin}}=0.5$. For numerical range conditions, we use width limits $(w_{\min},w_{\max})=(0.02,1.0)$, while for linear conditions we use $(\Delta_{\min},\Delta_{\max})=(0.02,1.0)$. These limits are specified as fractions of the corresponding variable or linear-function span, rather than as absolute widths. For generation, the CFG conditioning strength of $\omega = 1$ gave the best performance.

\begin{figure}[ht]
    \centering
    \includegraphics[width=0.95\textwidth]{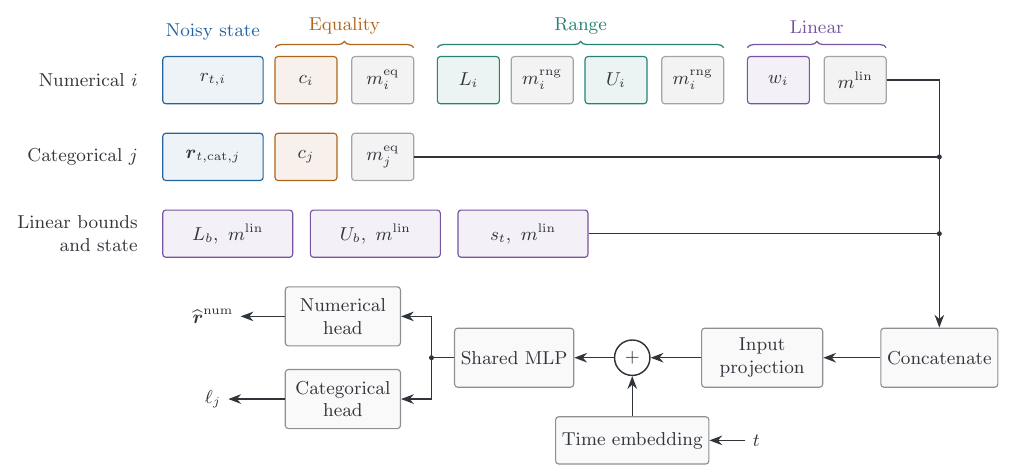}
    \caption{
    The \contdiff{} conditional denoiser differs from the unconditional one only at the input. For each column, the noisy state is concatenated with the parameters of every condition involving that column (equality value, range bounds, or linear weight), each paired with a mask indicating whether the condition is present. A linear condition additionally contributes its bounds and its value $s_t$ at the current noisy state. The concatenated input is then projected and processed by the same time-conditioned MLP and output heads as the unconditional model, and setting all masks to zero recovers the unconditional model (Section~\ref{subsec:CFG}; Appendix~\ref{app:CondGenCFG}).
    }
    \label{fig:CFG_arch_diagram}
\end{figure}

\subsection{Inference-time methods} \label{app:CondGenInfTime}

Rejection sampling and SMC can be used in conjunction with any unconditionally trained model---in our experiments, we utilize the same checkpoint used to obtain the results in Table~\ref{tab:unconditional-comparisons}, i.e. with CE loss in logit parameterization and using the learnt embeddings. The architecture and training details for this model are provided in Appendix~\ref{app:uncondGenDetails}. 

The pseudocode for SMC is provided in Algorithm~\ref{alg:smc}. If $n_{\mathrm{syn}}$ synthetic rows are requested, we generate $\lceil n_{\mathrm{syn}}/4 \rceil$ rows at a time using SMC and subsequently concatenate all four sets of synthetically generated rows. To generate each set of $m = \lceil n_{\mathrm{syn}}/4 \rceil$ rows, we initialize $m \cdot \max\left(1, \big\lceil\frac{8192}{m}\big\rceil\right)$ trajectories and use an effective sample size (ESS) threshold of $\alpha=0.95$ throughout our experiments. For conditioning on arbitrary numerical conditions (Table \ref{table:ArbitNumCondSMC}) we need to compute the probability that the final sample produced by trajectory will satisfy the given condition: we compute this via a Monte Carlo estimate; drawing $256$ antithetic samples from a Gaussian centered at the denoiser prediction and computing the fraction of draws that satisfy the requested condition (recall Section~\ref{subsec:InfTimeMethods}). 

For rejection sampling, for each condition $\mathcal{C}$, we retain the first $\min(100{,}000,n_{\mathrm{test}}^{\mathcal{C}})$ generated rows that satisfy the condition, where $n_{\mathrm{test}}^{\mathcal{C}}$ denotes the number of test-set rows satisfying $\mathcal{C}$. We use a maximum budget of 10 million unconditionally generated samples for rejection sampling. If a condition exhausts this budget before producing the required samples, we report its score as missing.
\subsection{Evaluation methodology} \label{app:CondGenEval}

 We use a fixed set of $50$ slice conditions $\mathcal{C}$ for each task (equality, numerical range, and linear conditioning) with each slice containing at least $100$ rows in the test set. The same $50$ slices are used for evaluating all methods. To construct an equality condition, we first sample $k\in{1,2,3,4}$, then select $k$ categorical columns uniformly without replacement and assign their conditioned values using a uniformly sampled row. For numerical range conditions, we sample $k\in{1,2,3,4}$ numerical columns uniformly without replacement. For each selected numerical column $i$, we sample the interval in quantile space by first drawing a quantile width $q_i\sim\text{LogUniform}(0.002,0.8)$ and then a lower quantile bound $a_i\sim\mathrm{Unif}(0,1-q_i)$. Set the upper quantile bound to $a_i+q_i$. These quantile bounds are then mapped back to the original numerical space to get $(L_i,H_i)$. The bounds for integer columns are rounded to the integers values. For linear conditions, we use the same distribution as during training. For all three tasks, we repeatedly sample candidate conditions until the resulting slice falls within the desired rarity range (like $1-3\%$) in the training set and contains at least $100$ rows in the test set.

For each slice condition $\mathcal{C}$, we evaluate $\min(100{,}000,n_{\mathrm{test}}^{\mathcal{C}})$ synthetic rows against the same number of real rows from the test set.  When computing the fidelity metrics, we retain all generated synthetic rows, including those that do not satisfy $\mathcal{C}$, for every generation method. To compute the empirical ideal score for $\mathcal{C}$, we compare $n_{\mathrm{eval}}^{\mathcal{C}}$ real test rows against an equal number of rows from the real training set. The rarity of $\mathcal{C}$ is defined by its frequency in the training set. For each condition $\mathcal{C}$, we compute the fidelity metrics using the three checkpoints obtained from the three independent training seeds. Each metric for each slice is averaged across these three runs. In the final fidelity plots, conditions are grouped into rarity buckets (e.g., $1$--$3\%$), and we report the mean and standard error of the scores across the slice conditions within each bucket.
\subsection{Baseline implementation} \label{app:HarpoonImplementation}

For a fair comparison, we apply Harpoon~\citep{shankar2026harpoon} guidance to the one-hot embedding, MSE-loss, representation-mode variant of \contdiff{} as the authors chose this base model in their work to be used for Harpoon. This allows controlled comparison of conditional generation capability in which the denoiser architecture and diffusion process are held fixed, with results presented in Figure~\ref{fig:cfg_acslabour_cond_shape_all} and Appendix~\ref{app:AdditionalResults}. We use the constant guidance scale of $0.2$, both as suggested by the authors and validated by us through a guidance-scale sweep on the ACSLabour dataset. We additionally extend Harpoon to support linear conditioning: each linear constraint is converted into an equivalent normalized interval in raw-data coordinates. At every diffusion step, it penalizes the denoiser prediction’s distance outside this interval and follows the loss's gradient as additional guidance.

While our implementation of Harpoon is faithful to the authors' design of base model, in light of the results in Table~\ref{tab:unconditional-architectures} one may ask whether Harpoon, which is essentially diffusion posterior sampling (DPS) \cite{chung2022diffusion} with a double-exponential likelihood (after forming the real representation of each row via the one-hot embedding) might be competitive with our methods with the strongest base architecture viz. learnt embedding, logits parameterization, and CE loss. We implement this version of Harpoon for conditional generation on the ACSLabour dataset, and plot the Shape vs. rarity in Figure~\ref{fig:cfg_acslabour_cond_shape_all_logitHarpoon}. As we can see, while a superior base model does improve performance vis a vis the previous version of Harpoon it still remains significantly below both SMC and CFG. We observe that this conclusion continues to hold across various datasets and metrics.
\begin{figure}
    \centering
    \includegraphics[width=0.9\textwidth]{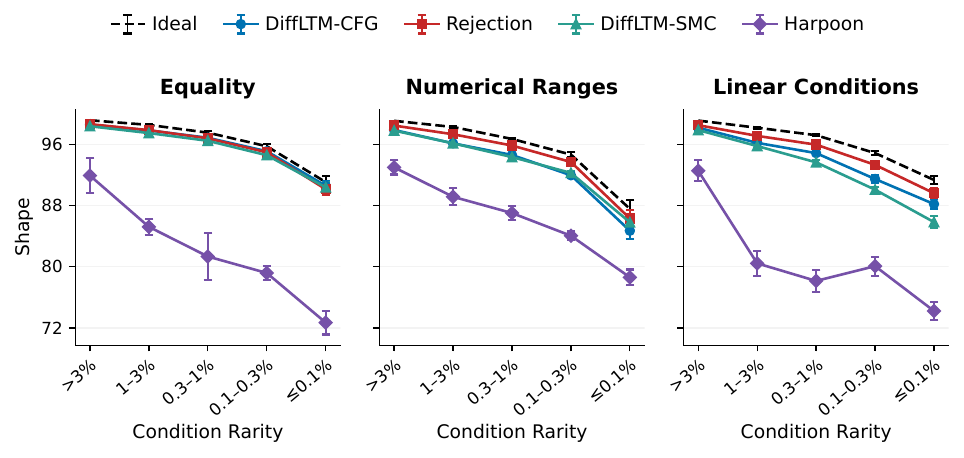}
    \caption{
     Shape score for conditional generation of ACS Labour data ($n_{\strain} \approx 4.8$M) with Harpoon implemented on same base model as CFG and SMC. The generation and evaluation protocol is otherwise identical to that followed in Figure~\ref{fig:cfg_acslabour_cond_shape_all}. }
    \label{fig:cfg_acslabour_cond_shape_all_logitHarpoon}
\end{figure}

\section{Arizona proposition 206 experiment details} \label{app:AZMinwageDetails}

We provide details on the experimental setup to study effects of Proposition 206 (cf. Figures~\ref{fig:AZ_PA_SC_result} and~\ref{fig:AZ_SC_indiv_edu}), as well as present additional results.  
\subsection{Dataset} \label{app:CPSDataset}

We utilize wage records from the 2010--2019 Current Population Survey (CPS) data, made available by the NBER \citep{nber_morg}. We retain employed private-sector  workers aged 16--64 with self reported hourly wages between $\$2.00$ and $\$99.98$. The resulting dataset contains 489,225 rows corresponding to observations from all 50 states and DC. The retained columns are state, year, age, sex, education, race, broad occupation, and $\log$(hourly wage), with wages normalized to constant 2019 units using CPI-U-RS, a provided inflation measure.

\subsection{Synthetic control counterfactual} \label{app:SCCounterfactual}

Our statistic of interest is the tenth-percentile hourly wage each year; this is the region where the effect of a minimum wage policy shock should be most visible. The donor pool $\mathcal{D}$ consists of the 21 states that remained at the federal $\$7.25$ minimum wage throughout 2010--2019: AL, GA, ID, IN, IA, KS, KY, LA, MS, NH, NC, ND, OK, PA, SC, TN, TX, UT, VA, WI, and WY. To pick the synthetic control weights, we construct the following $14$ features (all measured pre-2017) for each donor state and AZ: 
\begin{itemize}
    \item Fraction of workers aged 16--24
    \item Fraction of workers with at most a high school education 
    \item Fraction of workers with at least a bachelor degree
    \item Fraction of workers in the following occupations:
    \begin{itemize}
        \item Sales
        \item Cleaning or personal care 
        \item Food preparation and serving 
    \end{itemize}
    \item Mean $\log($hourly wage)
    \item $10$th percentile wage in each pre-intervention year (2010--2016).
\end{itemize}
The population features chosen in addition to the pre-intervention $10$th percentile value are those that make the individual more likely to be affected by a minimum wage increase. 

Let $\boldsymbol{x}_{\mathrm{AZ}}\in\mathbb{R}^{14}$ denote Arizona's obtained feature
vector and $\boldsymbol{x}_j$ the corresponding feature vector for donor state $j$ (all standardized). We choose the weights by minimizing the mean-squared distance:
\begin{align} \label{eq:SCWeightsOpt}
\widehat{\boldsymbol{w}} =\arg\min_{\boldsymbol{w}}
\Big\|
\boldsymbol{x}_{\mathrm{AZ}}-\sum_{j \in \mathcal{D}} w_j\boldsymbol{x}_j
\Big\|_2^2
\quad\text{subject to}\quad
w_j\geq 0,\qquad \sum_{j \in \mathcal{D}} w_j=1.
\end{align}
The post-intervention synthetic-control outcome $Y$ is then
\[
\widehat{Y}_{\mathrm{AZ},t} = \sum_{j \in \mathcal{D}} \widehat{w}_jY_{j,t}.
\]
The exact analogous procedure is used to construct the synthetic-control counterfactual for PA. 
\subsection{\contdiff{} counterfactual} \label{app:AZMinwageLTMCounterfactual}

From the dataset, we remove all rows ($2,539$ in total) satisfying 
\[
    \textsc{State}=\mathrm{AZ}
    \quad\text{and}\quad
    \textsc{Year}\in\{2017,2018,2019\}
\]
and train a CFG-enabled \contdiff{} model with the architecture/training choices described in Appendix \ref{app:CondGenCFG}. 

At inference, we provide condition $\textsc{State}=\mathrm{AZ}$ and the requested post-treatment year and sample all remaining worker attributes and wages jointly. We generate $100,000$ synthetic rows for the slice corresponding to each year and calculate the 10th percentile value of the hourly wage (converted into nominal \$). 
The exact analogous procedure is used to construct the \contdiff{} counterfactual for PA. For each held-out state we train one model using seed 500.
\subsection{Subpopulation-level analyses} \label{app:AZMinwageIndiv}

Synthetic control produces one counterfactual summary statistic per year (in the above case, the $10$th percentile hourly wage value). Conditional generation instead samples from the \emph{joint} counterfactual distribution and produces individual-level records. By further conditioning on other columns corresponding to demographic/occupational profiles and producing counterfactual distributions \emph{per subpopulation}, we can assess which subpopulations are more affected by the policy.

As an example, we stratify based on education level (shown in Figure~\ref{fig:AZ_SC_indiv_edu}) and broad occupational
description, since we expect certain education and occupation profiles to be
more affected by a minimum-wage increase. Using the model trained with all
post-intervention Arizona rows excised, we conditionally generate $100{,}000$
rows for each of the following slices:
\begin{itemize}
    \item $\textsc{State}=\textsc{AZ}$, $\textsc{Year}\in\{2016,2019\}$, and
    $\textsc{Edu}\in\{\text{high school},\text{bachelor's}\}$;
    \item $\textsc{State}=\textsc{AZ}$, $\textsc{Year}\in\{2016,2019\}$, and
    $\textsc{Occ}\in\{\text{healthcare},\text{cleaning},\text{sales}\}$.
\end{itemize}
The 2016 slices provide a calibration value because data available during training, whereas the 2019 slices correspond
to the held-out, post-intervention period.
We consider two metrics of policy exposure. First, we calculate the
Kolmogorov--Smirnov (KS) distance between the real and synthetic wage
distributions in each year and report its change from 2016 to 2019 for group $g$:
\[
\Delta_{\mathrm{KS}}(g)
=
\mathrm{KS}\!\left(F^{\mathrm{real}}_{g,2019},
                   F^{\mathrm{syn}}_{g,2019}\right)
-
\mathrm{KS}\!\left(F^{\mathrm{real}}_{g,2016},
                   F^{\mathrm{syn}}_{g,2016}\right).
\]
If the 2019 distance is similar to its 2016 calibration value, this suggests
that the policy shock for that subpopulation is small relative to the model's
baseline approximation error. Second, because Arizona's minimum wage was
$\$11.00$ on January 1, 2019, we calculate the \emph{sub-$\$11$ reduction}:
\[
\Delta_{<11}(g) = \mathsf{P}_{\mathrm{syn}}(\text{hourly wage}<\$11\mid g,2019) - \mathsf{P}_{\mathrm{real}}(\text{hourly wage}<\$11\mid g,2019).
\]
This measures the fraction of the counterfactual population predicted to earn
below $\$11$ that is absent from the observed post-policy distribution.

\begin{figure}[h]
    \centering
    \includegraphics[width=0.75\textwidth]{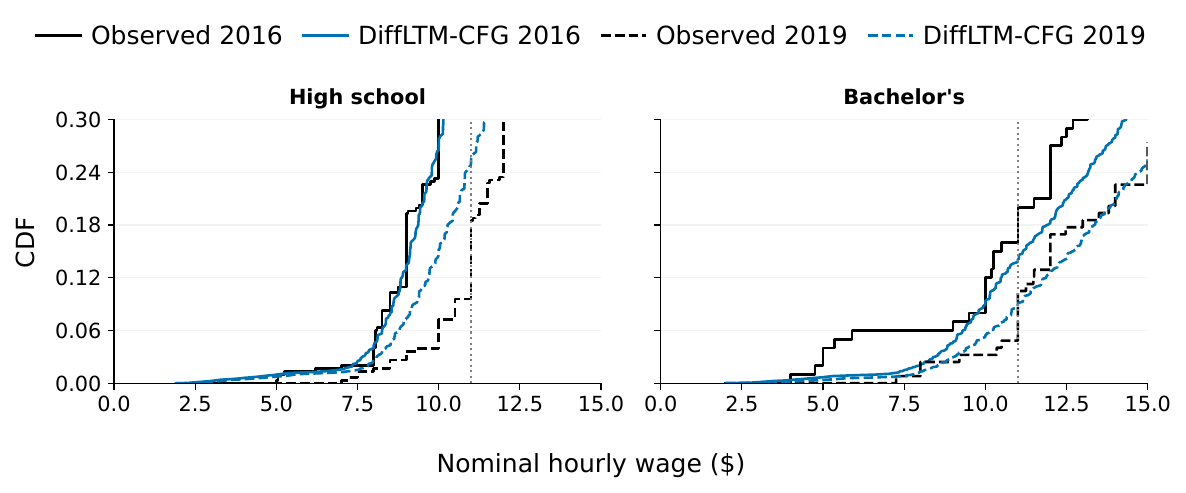}
    \caption{
    \contdiff{} generates the full counterfactual wage distribution for any subgroup. Arizona wage CDFs (lowest 30\%); 2016 seen in training, 2019 held out. Dotted: \$11 minimum set by Proposition~206. The policy effect, the gap between dashed curves, is large for high-school workers and small for bachelor's-degree workers relative to their 2016 gap (Section~\ref{subsec:Econ}; Appendix~\ref{app:AZMinwageIndiv}). 
    }
    \label{fig:AZ_SC_indiv_edu}
\end{figure}

The KS changes behave as expected (see Figures~\ref{fig:AZ_SC_indiv_edu} and~\ref{fig:AZ_SC_indiv_occ}). The distance increases by $0.128$ for
workers with a high-school education, compared with decreasing by $0.031$ for workers with a bachelor's degree. Across occupations, it increases by $0.085$ for
healthcare workers, $0.089$ for cleaning workers and $0.262$ for sales workers. The corresponding sub-\$11 reductions are $20.5$ percentage points
for high-school workers and $7.3$ points for workers with a bachelor's degree.
Across occupations, they are $1.7$ points for healthcare, $11.1$ points for
cleaning, and $36.6$ points for sales.
We note that the observed data also exhibit substantial mass at exactly \$11,
particularly among the more exposed groups.

\section{Reparametrization in synthetic control details} \label{app:SCReparamDetails}

As alluded to in Section~\ref{subsec:Econ} and illustrated in Figure~\ref{fig:sc_reparam_median}, synthetic control can be very brittle, in particular showing sharp performance degradation when the target variable is non-linearly transformed. We now precisely describe the experimental setup underlying this illustration.
\subsection{Data distribution} \label{app:SCReparamData}

To study this question in a controlled setting, we construct a synthetic dataset with state $S\in\{1,\ldots,11\}$, time period $T\in\{0,\ldots,10\}$, and exogenous noise $U\sim\mathrm{Unif}[-\sqrt{3},\sqrt3]$. The outcome is generated as
\[ 
Y=\mu_S+\mu_T+\sigma U, 
\]
where
\[ 
\boldsymbol{\mu}_T=\begin{bmatrix}11&10&\cdots&1\end{bmatrix} 
\]
\[ \boldsymbol{\mu}_S=\begin{bmatrix}0&0.02&\cdots&0.08&10&19.92&19.94&\cdots&20.0\end{bmatrix}
\]
\[ 
\sigma = 0.5.
\]

We draw $1{,}200$ independent rows for each $(S,T)$ pair (by sampling $1,200$ corresponding $U$ i.i.d.), giving a $4$-column table with columns $(Y,S,T,U)$. State $6$ is subject to an intervention during periods $T\in\{8,9,10\}$.
The construction is chosen so that state $6$, for which $\mu_6=10$, lies between two clustered groups of donor states.  The task is to learn a \contdiff{}-CFG model on this table of four columns and then conditionally generate $Y$ for the unseen slices which is $S=6$ and $T\in\{8,9,10\}$. 

To study robustness to reparameterization, we transform the outcome \(Y\) using a function \(f(Y)\). The training data are then given by the table \((f(Y),S,T,U)\). The three experiments in Figure~\ref{fig:sc_reparam_median} correspond to three choices of \(f\): (i) \(f(y)=y\), (ii) \(f(y)=\log y\), and (iii) \(f(y)=y(\max(y)-y)\). The third setting is included to evaluate robustness under a non-monotone reparameterization, where \(\max(y)\) denotes the maximum value of \(Y\) in the training set. Figure~\ref{fig:sc_reparam_median} compares the median of the distribution generated by our model with the median of the observed data and the median of the distribution generated by synthetic control.

\subsection{Training and generation} \label{app:SCReparamTrain}
Along with learned embeddings and logit parameterization, we use cross-entropy (CE) loss for the categorical features (\(S\) and \(T\)) and mean-squared error (MSE) loss for the numerical features (\(Y\) and \(U\)). Given the smaller dataset sizes, we use a smaller model for these experiments with embedding dimension \(d_{\mathrm{emb}}=4\), hidden width \(d_h=256\), and depth \(L=1\), resulting in approximately \(350{,}000\) trainable parameters. Unlike the other experiments, we standardize the numerical columns rather than applying a quantile transformation. The model is trained for \(1{,}000\) epochs with a batch size of \(4096\), and the final checkpoint is selected based on the minimum EMA validation loss. Following prior work~\citep{peebles2023scalable}, we additionally use MLP blocks with adaptive layer normalization (AdaLN)~\citep{perez2018film}. Specifically, the time embedding is used to predict per-layer scale and shift parameters, which are applied to the layer-normalized features before each linear layer. All other training and architectural details follow Appendix~\ref{app:uncondGenDetails}.

For each choice of \(f(Y)\), we repeat the experiment across three training seeds. For each trained model, we condition on \(S=6\) and each \(T\in\{0,\ldots,10\}\), generating \(100{,}000\) samples per \((S,T)\) condition. In Figure~\ref{fig:sc_reparam_median}, we report the mean of the generated medians across the three training seeds, with error bars denoting the corresponding standard error.

\section{Additional empirical results}
\label{app:AdditionalResults}

We present results on fidelity vs rarity across fidelity metrics and datasets---see Table~\ref{table:FidelityvRarityAdditional}. The protocol used for selecting slices, training the models, conditional generation, and metrics evaluation is exactly that described in Appendix~\ref{app:CondGenDetails} to obtain Figure~\ref{fig:cfg_acslabour_cond_shape_all}. We remark that our conclusion of our methods outperforming the \cite{shankar2026harpoon} baseline is very robust across metrics and datasets.

\begin{table}[ht]
\centering
\small
\begin{tabular}{lll}
\hline
Fidelity Metric & Dataset & Figure \\
\hline
Shape & ACSLabour & Figure~\ref{fig:cfg_acslabour_shape} \\
Trend & ACSLabour & Figure~\ref{fig:cfg_acslabour_trend} \\
Conditional Shape (Rare) & ACSLabour & Figure~\ref{fig:cfg_acslabour_cond_shape_rare} \\
Shape & ACSPubCov & Figure~\ref{fig:cfg_acspubcov_shape} \\
Trend & ACSPubCov & Figure~\ref{fig:cfg_acspubcov_trend} \\
Conditional Shape & ACSPubCov & Figure~\ref{fig:cfg_acspubcov_cond_shape_all} \\
Conditional Shape (Rare) & ACSPubCov & Figure~\ref{fig:cfg_acspubcov_cond_shape_rare} \\
Shape & ACSMobility & Figure~\ref{fig:cfg_acsmobility_shape} \\
Trend & ACSMobility & Figure~\ref{fig:cfg_acsmobility_trend} \\
Conditional Shape & ACSMobility & Figure~\ref{fig:cfg_acsmobility_cond_shape_all} \\
Conditional Shape (Rare) & ACSMobility & Figure~\ref{fig:cfg_acsmobility_cond_shape_rare} \\
\hline
\end{tabular}
\caption{Additional fidelity vs.\ rarity results across fidelity metrics and datasets.}
\label{table:FidelityvRarityAdditional}
\end{table}

\begin{figure}[h]
    \centering
    \includegraphics[width=0.9\textwidth]{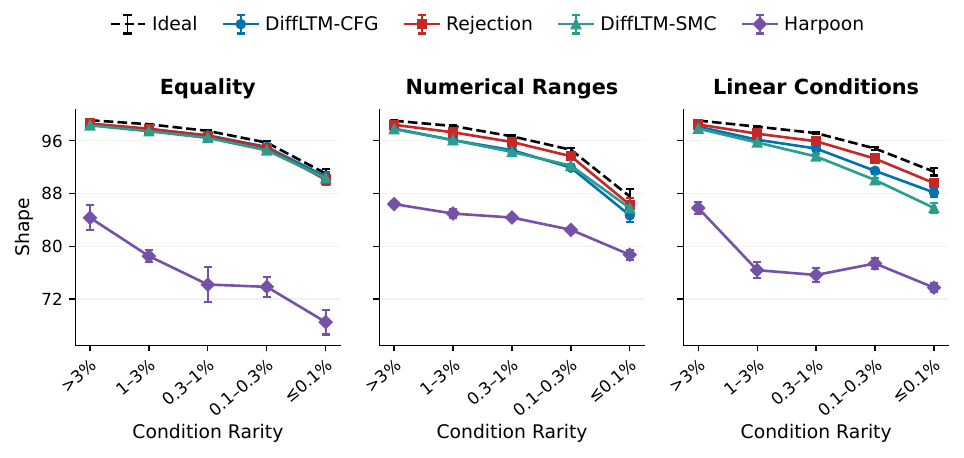}
    \caption{Shape vs rarity on ACSLabour Dataset.
    }
    \label{fig:cfg_acslabour_shape}
\end{figure}

\begin{figure}[h]
    \centering
    \includegraphics[width=0.9\textwidth]{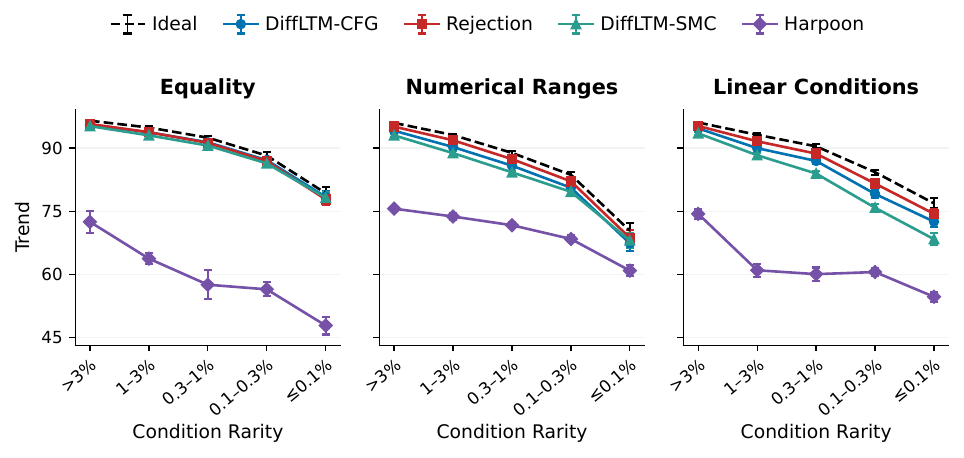}
    \caption{Trend vs rarity on ACSLabour Dataset.}
    \label{fig:cfg_acslabour_trend}
\end{figure}

\begin{figure}[h]
    \centering
    \includegraphics[width=0.9\textwidth]{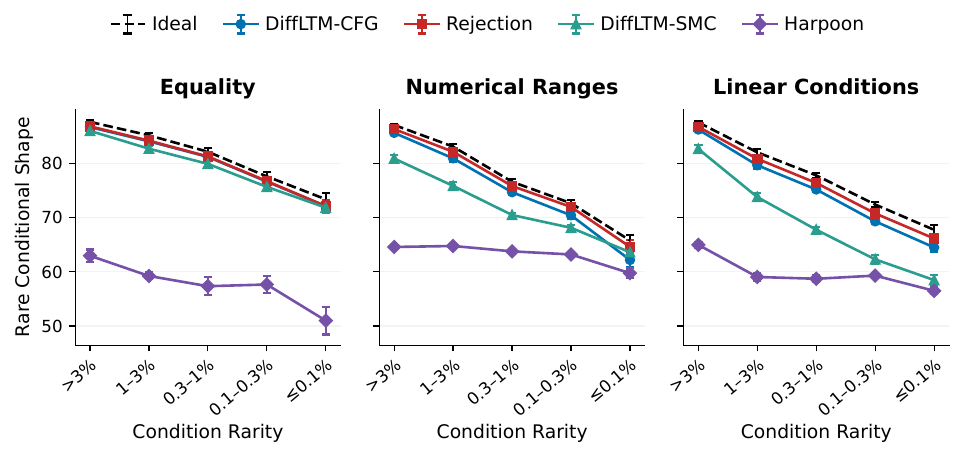}
    \caption{Conditional Shape (Rare) vs rarity on ACSLabour Dataset.}
    \label{fig:cfg_acslabour_cond_shape_rare}
\end{figure}

\begin{figure}[h]
    \centering
    \includegraphics[width=0.9\textwidth]{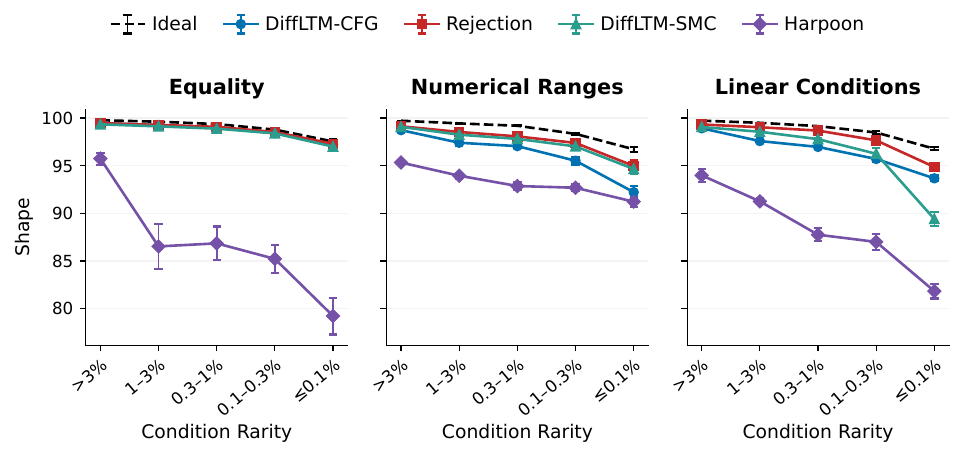}
    \caption{Shape vs rarity on ACSPubCov Dataset.}
    \label{fig:cfg_acspubcov_shape}
\end{figure}

\begin{figure}[h]
    \centering
    \includegraphics[width=0.9\textwidth]{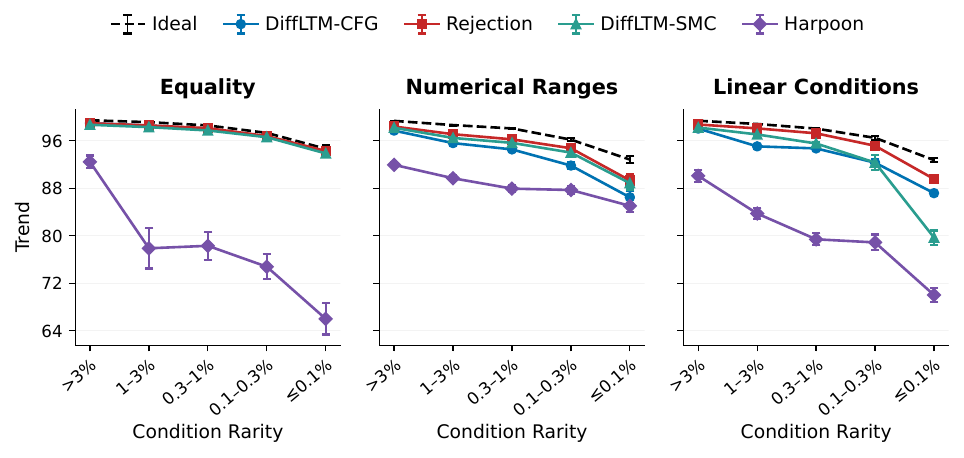}
    \caption{Trend vs rarity on ACSPubCov Dataset.}
    \label{fig:cfg_acspubcov_trend}
\end{figure}

\begin{figure}[h]
    \centering
    \includegraphics[width=0.9\textwidth]{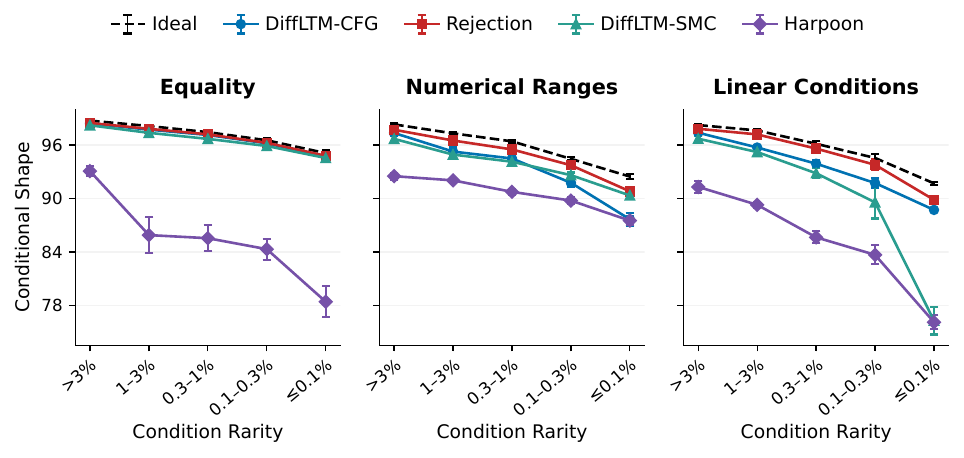}
    \caption{Conditional Shape vs rarity on ACSPubCov Dataset. }
    \label{fig:cfg_acspubcov_cond_shape_all}
\end{figure}

\begin{figure}[h]
    \centering
    \includegraphics[width=0.9\textwidth]{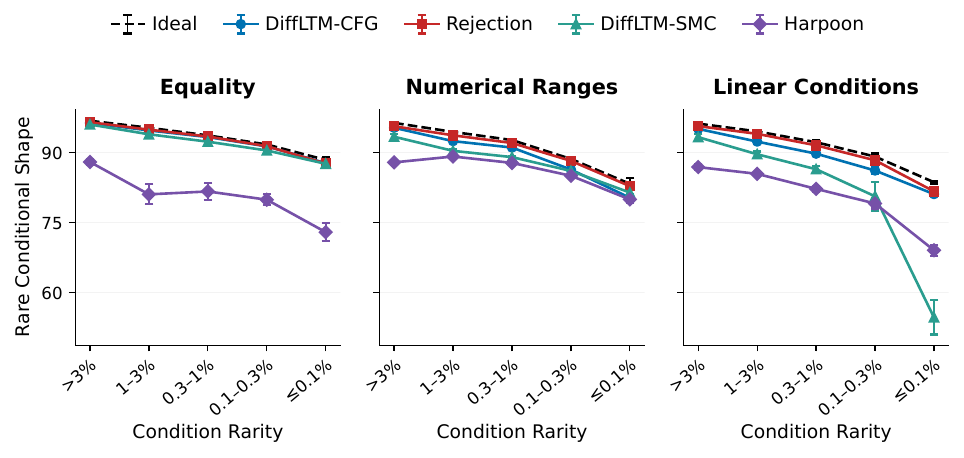}
    \caption{Conditional Shape (Rare) vs rarity on ACSPubCov Dataset. }
    \label{fig:cfg_acspubcov_cond_shape_rare}
\end{figure}

\begin{figure}[h]
    \centering
    \includegraphics[width=0.9\textwidth]{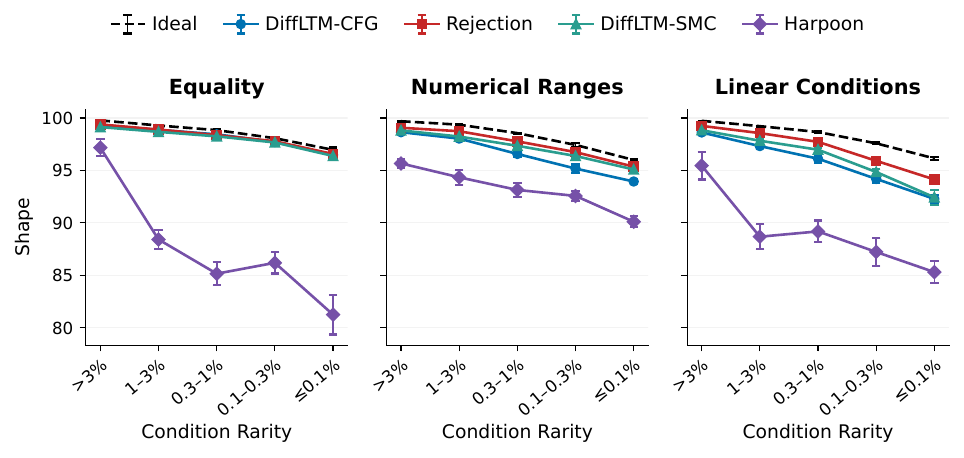}
    \caption{Shape vs rarity on ACSMobility Dataset. }
    \label{fig:cfg_acsmobility_shape}
\end{figure}

\begin{figure}[h]
    \centering
    \includegraphics[width=0.9\textwidth]{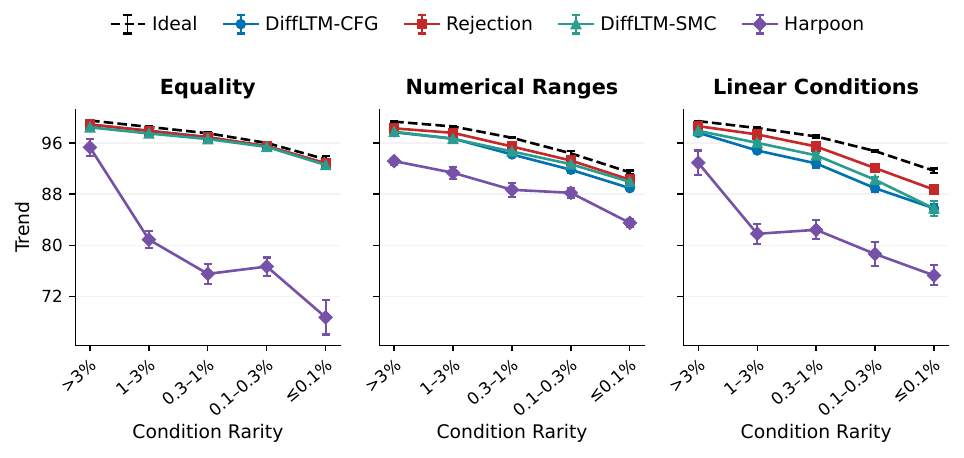}
    \caption{Trend vs rarity on ACSMobility Dataset. }
    \label{fig:cfg_acsmobility_trend}
\end{figure}

\begin{figure}[h]
    \centering
    \includegraphics[width=0.9\textwidth]{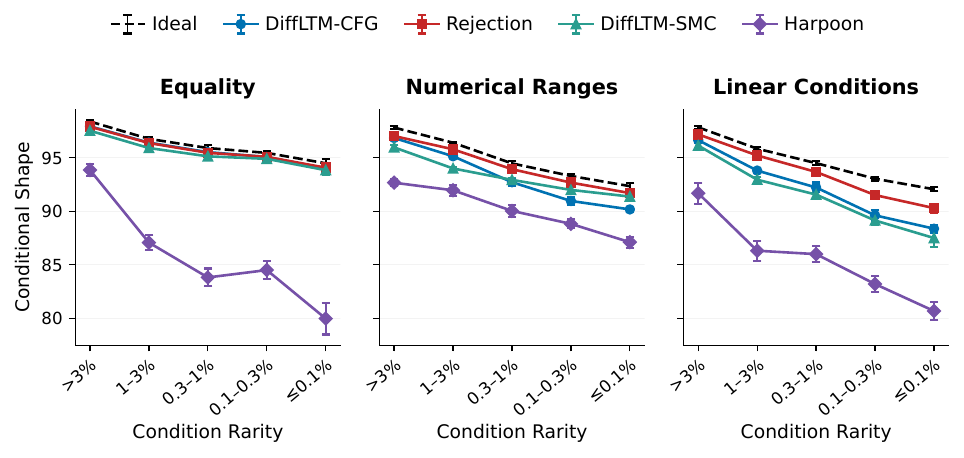}
    \caption{Conditional Shape vs rarity on ACSMobility Dataset. }
    \label{fig:cfg_acsmobility_cond_shape_all}
\end{figure}

\begin{figure}[h]
    \centering
    \includegraphics[width=0.9\textwidth]{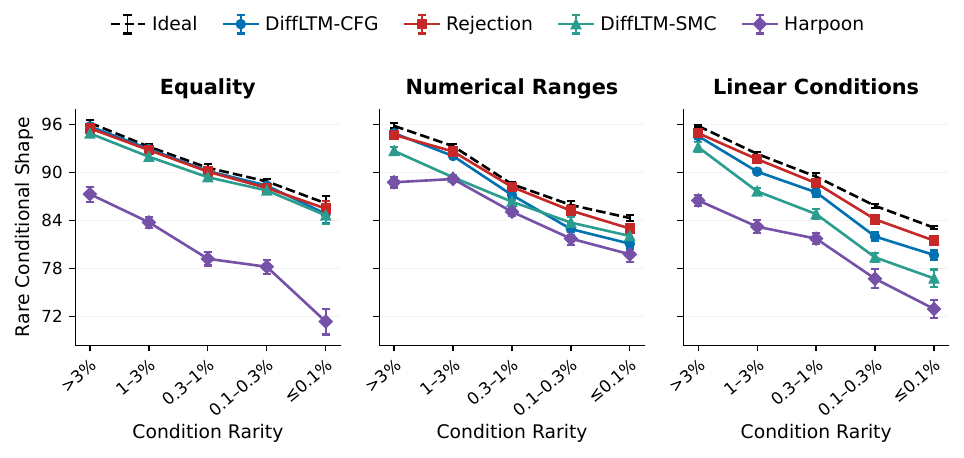}
    \caption{Conditional Shape (Rare) vs rarity on ACSMobility Dataset. }
    \label{fig:cfg_acsmobility_cond_shape_rare}
\end{figure}

We also use SMC to condition on arbitrary numerical conditions using the method described in Section~\ref{subsec:InfTimeMethods} and display the obtained shape score in Table~\ref{table:ArbitNumCondSMC} along with the ideal value. The displayed value is the mean over three training seeds. 





\begin{table}[t]
\centering
\small
\begin{tabular}{lrrr}
\hline
Slice & Rarity (\%) & Shape $\times 100$ (SMC) & Shape $\times 100$ (Ideal) \\
\hline
$\left(\frac{\mathrm{AGEP}-43}{10}\right)^2
 + \left(\frac{\texttt{log\_wage}-10.6}{0.5}\right)^2 \le 0.64$
& 9.48 & 97.56 & 99.09 \\

$0.25 \le
 \left(\frac{\mathrm{AGEP}-43}{10}\right)^2
 + \left(\frac{\texttt{log\_wage}-10.6}{0.5}\right)^2
 \le 0.64$
& 5.91 & 97.32 & 98.97 \\

$(\mathrm{AGEP}-43)(\mathrm{WKHP}-40)<-200$
& 6.90 & 96.75 & 99.08 \\

$60 \le \frac{\exp(\texttt{log\_wage})}{52\,\mathrm{WKHP}} \le 90$
& 3.63 & 96.70 & 98.73 \\
\hline
\end{tabular}
\caption{Conditional generation performance of SMC for arbitrary (nonlinear) numerical conditions on the ACSLabour dataset. Shape values displayed are means over three training seeds. Details on implementation and generation provided in Appendix~\ref{app:CondGenInfTime}.}
\label{table:ArbitNumCondSMC}
\end{table}
\section{Pseudocode} \label{app:pseudocode}

\begin{algorithm}[ht]
\DontPrintSemicolon
\caption{Training \contdiff}\label{alg:contdiff_training}
\KwIn{training data, denoiser $f_\theta$, embeddings $\{\phi_j\}$, noise schedule $\sigma_t$, categorical parameterization and loss type}
\Repeat{convergence}{
Sample a minibatch of rows $\boldsymbol{x}$ and construct their continuous representations $\boldsymbol{r}$ using (\ref{eq:ContRepDefn})\;
Sample $t \sim \mathrm{Unif}[0,1]$ and $\boldsymbol{z}\sim\mathcal{N}(\boldsymbol{0},I_D)$\;
Construct the noisy input $\boldsymbol{r}_t \leftarrow \boldsymbol{r}+\sigma_t\boldsymbol{z}$\;
Evaluate $f_\theta(\boldsymbol{r}_t,t)$\;
Compute $\mathcal{L}_{\mathrm{num}} \leftarrow \frac{1}{m_{\mathrm{num}}}\|\widehat{\boldsymbol{r}}_{\mathrm{num}}-\boldsymbol{r}_{\mathrm{num}}\|_2^2$\;
\If{categorical output is logits}{
Compute $\widehat{p}_{j,c}\leftarrow\mathrm{softmax}(\boldsymbol{\ell}_j)_c$ for each categorical column $j$\;
}
\If{categorical loss is MSE}{
If logits mode, construct $\widehat{\boldsymbol{r}}_{\mathrm{cat},j}\leftarrow\sum_{c=1}^{C_j}\widehat{p}_{j,c}\phi_j(c)$ using (\ref{eq:EmbeddingDefLogitMode})\;
Compute $\mathcal{L}_{\mathrm{cat}}\leftarrow\mathcal{L}_{\mathrm{cat}}^{\mathrm{MSE}}$ using (\ref{eq:MSELossCatDef})\;
}
\Else{
Compute $\mathcal{L}_{\mathrm{cat}}\leftarrow\mathcal{L}_{\mathrm{cat}}^{\mathrm{CE}}$ using (\ref{eq:MSELossCatDef})\;
}
Update $\theta$ and, for learned embeddings, $\{E_j\}$ using gradients of $\mathcal{L}_{\mathrm{num}}+\mathcal{L}_{\mathrm{cat}}$\;
}
\end{algorithm}

\begin{algorithm}[H]
\DontPrintSemicolon
\caption{Generation with \contdiff}\label{alg:contdiff_sampling}
\KwIn{trained denoiser $f_{\theta^*}$, embeddings $\{\phi_j\}$, noise schedule $\sigma_t$, grid $1=t_0>\cdots>t_T=0$}
Initialize $\boldsymbol{r}^{(0)}\sim\mathcal{N}(\boldsymbol{0},\sigma_{t_0}^2I_D)$\;
\For{$i=0,\ldots,T-1$}{
Evaluate the denoiser at $(\boldsymbol{r}^{(i)},t_i)$\;
\If{categorical output is logits}{
Compute $\widehat{p}_{j,c}^{(i)}\leftarrow\mathrm{softmax}(\boldsymbol{\ell}_j^{(i)})_c$ and $\widehat{\boldsymbol{r}}_{\mathrm{cat},j}^{(i)}\leftarrow\sum_{c=1}^{C_j}\widehat{p}_{j,c}^{(i)}\phi_j(c)$\;
}
Form the full denoised estimate $\widehat{\boldsymbol{r}}^{(i)}$\;
Update $\boldsymbol{r}^{(i+1)}\leftarrow\boldsymbol{r}^{(i)}+(\sigma_{t_{i+1}}-\sigma_{t_i})\frac{\boldsymbol{r}^{(i)}-\widehat{\boldsymbol{r}}^{(i)}}{\sigma_{t_i}}$\;
}
Return the final numerical coordinates directly\;
\If{categorical output is representation}{
Decode each categorical column as $\widehat{x}_j\leftarrow\arg\min_c\|\widehat{\boldsymbol{r}}_{\mathrm{cat},j}-\phi_j(c)\|_2$\;
}
\Else{
Decode each categorical column as $\widehat{x}_j\leftarrow\arg\max_c\widehat{p}_{j,c}$\;
}
\KwOut{generated mixed-type row $\boldsymbol{x}_{\text{syn}}$}
\end{algorithm}

\begin{algorithm}[H]
\DontPrintSemicolon
\caption{SMC for sampling $p(\boldsymbol{x}\mid\mathcal{C})$}
\label{alg:smc}
\KwIn{condition $\mathcal{C}$, $K$ particles, Noise schedule, $n_{\mathrm{syn}}$ output rows, effective sample size (ESS) threshold $\alpha K$} 
\For{$k\in[K]$}{
$\boldsymbol{r}_k^{(0)}\sim\mathcal{N}(\boldsymbol{0},\sigma_{t_0}^2I)$; $\log w_k\gets0$; $\log p_{\mathrm{prev},k}\gets0$\; \tcp*[r]{Initialize}
}
\For{$i=0,\ldots,T-1$}{
Estimate $p_k^{(i)}(\mathcal{C})$ \;
$\log w_k\gets\log w_k+\log p_k^{(i)}(\mathcal{C})-\log p_{\mathrm{prev},k}$ for all $k$ \tcp*[r]{reweight}
\If{$\operatorname{ESS}(\log\boldsymbol{w})<\alpha K$ and $i<T-1$}{
Systematically resample the pairs $\{(\boldsymbol{r}_k^{(i)},p_k^{(i)}(\mathcal{C}))\}_{k=1}^K$ according to $\boldsymbol{w}$; set $\log w_k\gets0$\tcp*[r]{resample}
}
$\log p_{\mathrm{prev},k}\gets\log p_k^{(i)}(\mathcal{C})$ for all $k$\;
$\boldsymbol{r}_k^{(i+1)}\gets\textsc{SdeStep}(\boldsymbol{r}_k^{(i)},t_i\rightarrow t_{i+1})$ for all $k$\tcp*[r]{move}
}
Let $\widehat{\boldsymbol{x}}_k$ be the candidate synthetic row produced by trajectory $\boldsymbol{r}_k^{(T)}$\;
$\log w_k\gets\log w_k+\log\mathbf{1}\{\widehat{\boldsymbol{x}}_k\text{ satisfies }\mathcal{C}\}-\log p_{\mathrm{prev},k}$ for all $k$\;
\KwOut{$n_{\mathrm{syn}}$ rows systematically resampled from $\{\widehat{\boldsymbol{x}}_k\}_{k=1}^K$ according to $\boldsymbol{w}$}
\end{algorithm}


\FloatBarrier
\newpage
\section{Optimality of CE loss for continuous diffusions:\texorpdfstring{\\}{ }
 Proofs for Section \ref{subsec:Optimality-CE-Loss}}
\label{app:Asymptotics}

In this appendix we give a formal presentation
of the theoretical results of Section \ref{sec:Theory}
establishing optimality of CE loss in the undersmoothed regime.
Section \ref{sec:setting-Gen} introduces the setting.
Section \ref{sec:N_Asymp-Gen} describes the asymptotics of parameters estimates. Section \ref{sec:Undersmooth} focuses on the undersmoothing regime $b\ll n^{-1/2}$ and Section \ref{sec:Density}
characterizes the optimal sampling density $\lambda$ in this setting.
(This is not needed for our main result, but is nevertheless a useful byproduct of the analysis.)
Finally, proofs of the various technical steps are given in 
Section \ref{sec:proofs-N-Asymp-Gen}.

\subsection{Setting and assumptions}\label{sec:setting-Gen}

Recall the setting of Section \ref{subsec:Optimality-CE-Loss}.
We assume the embeddings $\{\bx(\sigma): \sigma\in \{0,\dots,k-1\}\}$
are distinct and all class probabilities $q(\sigma)$ are strictly positive.
We parametrize the probabilities by log-weights $\{r(\sigma)\}_{\sigma=1}^{k-1}$ as follows
\begin{align}
q(\sigma) &= \frac{1}{Z_0}e^{-r(\sigma)},
\qquad \sigma\in \{1,\dots,k-1\},\\
q(0)&=\frac{1}{Z_0},\qquad\qquad Z_0:=1+\sum_{\sigma=1}^{k-1} e^{-r(\sigma)}\, .
\end{align}

The diffusion process $\{\bX_t:\, t\in[0,T]\}$ is given by 
Eq.~\eqref{eq:general-diffusion-0}, 
or equivalently by the SDE \eqref{eq:general-diffusion-SDE}.
The functions $a(t)$, $\tau_0(t)$, $s(t)$ entering these definitions
are continuously differentiable. 
The variance at time $t$ is denoted by $\tau(t)^2:=\tau_0(t)^2s(t)$.

The evolution of the distribution of $\bX_t$ is given by the following 
Fokker-Planck PDE
\begin{align}
\partial_t p(\bx,t) &=-\nabla_{\bx}\mathbin\cdot \left( \bff(\bx,t) p(\bx,t) \right) +
 \frac{1}{2} g(t) \Delta_{\bx}  p(\bx,t) \, ,\\
 p(\bx,0) &=
\frac{1}{(2\pi \tau(0)^2)^{d/2}}\, e^{-\|\bx\|^2/(2\tau(0)^2)}\, .
\end{align}
The one-time marginals of $\bX_t$ are the same as the ones of the probability flow ODE
\eqref{eq:ProbFlow-Gen}.
Tweedie's formula reads
\begin{align}
\bm(\bx,t)
&= \frac{1}{a(t)}
\left\{\bx+\tau_0(t)^2s(t)\,\nabla\log p(\bx,t)\right\}\, .
\end{align}
This yields
\begin{align}
\bh(\bz,t)
&= \left\{\dot{a}(t) - \frac{a(t)\dot{\tau}_0(t)}{\tau_0(t)}
-\frac{a(t)\dot{s}(t)}{2s(t)}\right\}\bm(\bz,t)
+\left\{\frac{\dot{\tau}_0(t)}{\tau_0(t)}
+\frac{\dot{s}(t)}{2s(t)}\right\}\bz\\
&= \frac{\dot{a}(t)}{a(t)}\bz
+\tau_0(t)^2s(t)\left\{\frac{\dot{a}(t)}{a(t)}
-\frac{\dot{\tau}_0(t)}{\tau_0(t)}
-\frac{\dot{s}(t)}{2s(t)}\right\}
\nabla\log p(\bz,t)\, .
\end{align}

We are given i.i.d. data triples
$\{(t_i,\sigma_i,\bY_i):i\le n\}$.  Their common law is specified by
$\sigma_i\in\{0,1,\dots,k-1\}$, $\prob(\sigma_i=\sigma)=q(\sigma)$,
$t_i\sim \lambda(t)\cdot\de t$, independently of $\sigma_i$, where $\lambda$ is a
probability density on $[0,T]$, and
\begin{align}
\left.\bY_i\right|_{(t_i,\sigma_i)}
\sim \normal\big(\bx(\sigma_i)\,a(t_i),\tau(t_i)^2\id_d\big),
 \quad \quad i\le n\, . 
\end{align}

The conditional model \eqref{eq:conditional-distribution-Gen}
is parametrized by
\[
\btheta_{\sigma}=(\theta_{\sigma,0},\btheta^{+}_{\sigma})\in\reals^{d+1}.
\]
We write
\[
\btheta=(\btheta_1,\dots,\btheta_{k-1})\in E,
\qquad E:=\reals^{(d+1)\times(k-1)}.
\]
In the following parameter derivatives are taken in the space $E$
with Frobenius inner product and norm.
Hessians and covariance matrices are understood as linear operators on
$E$; operator composition is denoted by $\circ$.
We will focus on the asymptotics $n\to\infty$ with $k,d$ fixed.

The function $\btheta(t)$ is a parameter learned via the
following weighted ERM:
\begin{align}
\hbtheta_b(t)
&:= \argmin\big\{\hR_{b,t}(\btheta) :\; \btheta\in K(t)\big\}\, ,\\
\hR_{b,t}(\btheta) &:= \frac{1}{n}\sum_{i=1}^{n}
w_b(t_i-t)\, \ell_{\sigma_i}(\bY_{i};\btheta) \, .
\end{align}
where $w_b(t) = w_1(t/b)$ is an unnormalized weight function and $b$ is a bandwidth parameter
and $t\mapsto K(t)\subseteq E$ specifies
the time-dependent optimization domain. All empirical and population
minimizations at time $t$ use the same set $K(t)$, subject to
Assumption~\ref{ass:ZeroBandwidth-Gen} below.

We define the population risk at time $t$ $R_t$ and its non-zero bandwidth version 
$R_{b,t}$ as:
\begin{align}
R_t(\btheta)&
:=\E\left[\ell_\Sigma\big(
 a(t)\bx(\Sigma)+\tau(t)\bG;\btheta\big)\right]\, ,\\
R_{b,t}(\btheta)
&:=\E\left[w_b(U-t)\ell_{\Sigma}
  \big(a(U)\bx(\Sigma)+\tau(U)\bG;\btheta\big)\right]\, ,
\label{eq:localized-population-risk-Gen}
\end{align}
where $U$ has density $\lambda$, $\Sigma\sim q$,
$\bG\sim\normal(0,\id_d)$, and $U$, $\Sigma$, and $\bG$ are independent.
Finally, we define the  corresponding population estimators as 
\begin{align}
  \btheta^{b=0}(t):=\argmin_{\btheta\in K(t)}R_{t}(\btheta),\;\;\;\;\;
  \btheta^b(t):=\argmin_{\btheta\in K(t)}R_{b,t}(\btheta), 
\end{align}
and the `target' parameters as  $\btheta^*(t)$, where, for $\sigma\in \{1,\dots,k-1\}$,
\begin{align} 
\theta^*_{\sigma,0}(t) :=-r(\sigma)\, ,
\;\;\;\;\;\;
\btheta^{*,+}_{\sigma}(t)
:=\frac{a(t)}{\tau(t)^2}\{\bx(\sigma)-\bx(0)\} \, .
\end{align}
(The model \eqref{eq:conditional-distribution-Gen}
matches the true conditional distribution for $\hbtheta(t) = \btheta^*(t)$.)

Our first assumption is that the loss function identifies the true parameters
at population level on the compact time neighborhood under consideration.
\begin{assumption}[Fisher consistency on $J$]\label{ass:MinimizerPopulation-Gen}
For every $t\in J$, the population risk $R_t$
has a unique minimizer over $K(t)$, and
\begin{align}
\btheta^*(t)
= \btheta^{b=0}(t)\in \interior\big(K(t)\big)\, .
\end{align}
\end{assumption}
We emphasize that we allow the sets $K(t)$ to expand with $t$. 

Our second assumption is that the generative process is given by the probability
flow ODE \eqref{eq:ProbFlow-Gen}, with learned parameters $\hbtheta_b$. 
\begin{assumption}[Generative process]\label{ass:GenerativeProcess-Gen}
  Assume $a$, $\tau_0$, and $s$ are smooth and $a(0)=0$, $a(T)=1$, $\tau_0(T)^2s(T)=0$.
  We consider the generation process given by the ODE:
  \begin{align}
  \de \bZ_t &= \hbh(\bZ_t,t)\,\de t \, ,\;\;\; \bZ_0 \sim\normal(0,\tau(0)^2\id_d)\, .
  \end{align}
  where
  \begin{align}
    \hbh(\bz,t) &:= c_1(t)\hbm(\bz,t;\hbtheta_b(t))+ c_2(t)\bz\, ,\\
    c_1(t)&:= \dot{a}(t) - \frac{a(t)\dot{\tau}_0(t)}{\tau_0(t)}
    -\frac{a(t)\dot{s}(t)}{2s(t)}\, ,\\
    c_2(t)&:=\frac{\dot{\tau}_0(t)}{\tau_0(t)}
    +\frac{\dot{s}(t)}{2s(t)}\, .
  \end{align}
  \end{assumption}

The next assumption is was already mentioned in the statement of Theorem \ref{thm:Overall-Limit-Gen-Informal} and discussed there. It ensures in particular
that $\hbtheta^b(t)$ asymptotically un biased.
\begin{assumption}[Undersmoothing]\label{ass:Bandwidth-Gen}
We assume $b=b_n$ and impose throughout the
undersmoothing conditions
\begin{align}
  n^{1/4}b_n\to 0, \qquad \qquad \frac{\sqrt n\,b_n}{\log(1/b_n)}\longrightarrow\infty.
\end{align}
\end{assumption}

We next state some standard conditions for the kernel.
\begin{assumption}[Kernel]\label{ass:Kernel-Gen}
The kernel $w_1$ is nonnegative, symmetric, supported on $[-1,1]$,
H\"older continuous on $\reals$ with exponent $\alpha\in(0,1]$, and
$w_1(0)>0$. Define
\begin{align}
\mu_0:=\int_{-1}^1w_1(v)\,\de v>0,\qquad
\mu_2:=\int_{-1}^1v^2w_1(v)\,\de v.
\label{eq:kernel-moments-Gen}
\end{align}
\end{assumption}

\begin{assumption}[Zero-bandwidth regularity on $I\subset[0,T)$]\label{ass:ZeroBandwidth-Gen}
Fix a compact interval $I\subset[0,T)$ and a compact 
neighborhood $J\subset[0,T)$ of $I$.
\begin{enumerate}
\item There is a compact convex  set $K_0\subset E$ (depending on $I$) with
nonempty interior, such that
\begin{align}
\inf_{t\in J}\operatorname{dist}
 \bigl(\btheta^*(t),E\setminus K_0\bigr)>0,
\label{eq:localization-envelope-margin}
\end{align}
and one of the following holds:
\begin{enumerate}[label=(\alph*)]
\item The sets $K(t)$ are nonempty, compact, and convex for $t\in J$;
their graph
\begin{align}
\mathcal K_J:=\{(t,\btheta):t\in J,\ \btheta\in K(t)\}
\label{eq:time-dependent-domain-graph}
\end{align}
is a compact subset of $J\times K_0$; and the target has a uniform interior
margin,
\begin{align}
\inf_{t\in J}\operatorname{dist}
 \bigl(\btheta^*(t),E\setminus K(t)\bigr)>0.
\label{eq:time-dependent-domain-margin}
\end{align}
\item $K(t)=E$ for every $t\in J$, and
$\ell_{\sigma}(\by;\,\cdot\,):E\to\reals$ is finite and convex for
every $(\sigma,\by)$.
\end{enumerate}
\item The losses and their first three parameter derivatives are jointly
continuous in $(\by,\btheta)$ on an open neighborhood of
$\reals^d\times K_0$. For finite $C,\kappa$,
\begin{align}
\max_{0\le j\le 3}\sup_{\btheta\in K_0}
 \|\nabla_{\btheta}^j\ell_\sigma(\by;\btheta)\|
 \le C(1+\|\by\|)^\kappa,
 \qquad 0\le\sigma\le k-1.
\label{eq:loss-growth-Gen}
\end{align}
Here $\nabla^0\ell=\ell$ and the norms for $j\ge2$ are tensor operator norms.
\item The functions $a,\tau,\lambda$ are continuous on $J$,
$\inf_J\lambda>0$, and
\begin{align}
\inf_{t\in J}\lambda_{\min}
 \big(\nabla_{\btheta}^2R_t(\btheta^{b=0}(t))\big)>0.
\label{eq:Hessian-positive-Gen}
\end{align}
\end{enumerate}
\end{assumption}

\begin{assumption}[Positive noise]\label{ass:PositiveNoise-Gen}
The schedule satisfies $\tau(t)>0$ for $0\le t<T$.
\end{assumption}
Since $\tau(t)>0$, for any $t< T$ the true density is given by
\begin{align}
p_t(\bz)=\sum_{\sigma=0}^{k-1}q(\sigma)
 \frac{\exp\{-\|\bz-a(t)\bx(\sigma)\|^2/(2\tau(t)^2)\}}
 {(2\pi\tau(t)^2)^{d/2}}.
\label{eq:mixture-density-Gen}
\end{align}

For $0\le s<t<T$, let
$\boldsymbol\Phi^*_{s,t}:\reals^d\to\reals^d$ denote the flow associated to the ODE
\eqref{eq:ProbFlow-Gen}.
Write $\boldsymbol\Phi^*_{t,s}:=(\boldsymbol\Phi^*_{s,t})^{-1}$.
For $0<t<T$, define  
the terminal
basins
\begin{align}
\cA_\sigma(t)
&:=\left\{\bz\in\reals^d:
 \lim_{u\uparrow T}\boldsymbol\Phi^*_{t,u}(\bz)=\bx(\sigma)\right\},
 \qquad 0\leq\sigma\leq k-1,
\label{eq:general-terminal-basins}
\end{align}
The next lemma gives some simple regularity properties of the basins
and is essentially extracted from \cite{wan2025elucidating}. 
For self-containedness, we give a proof in Appendix~\ref{proof:terminal-basins-Gen}.
\begin{lemma}[Terminal convergence and measurable basins \cite{wan2025elucidating}] 
\label{lem:terminal-basins-Gen}
Under Assumptions~\ref{ass:GenerativeProcess-Gen} and
\ref{ass:PositiveNoise-Gen}, the 
$0\le s<t<T$,
$\boldsymbol\Phi^*_{s,t}:\reals^d\to\reals^d$
is a smooth 
diffeomorphism.
For each $0\le s<T$, $\lim_{t\to T}\boldsymbol\Phi^*_{s,t}(\bz)$ exists and belongs to
$\{\bx(0),\ldots,\bx(k-1)\}$ for Lebesgue-almost every starting point $\bz$
at time $s$, hence the sets $\cA_\sigma(s)$ are Borel measurable and partition $\reals^d$ up to a null set.
For $0\le s<t<T$,
\begin{align}
\boldsymbol\Phi^*_{s,t}(\cA_\sigma(s))&=\cA_\sigma(t),
&
\int_{\cA_\sigma(t)}p_t(\bz)\,\de\bz&=q(\sigma).
\label{eq:general-measurable-basin-transport}
\end{align}
\end{lemma}

Score matching and cross-entropy are two popular losses
\begin{align}
\ell^{\rm SM}_\sigma(\by;\btheta)
 &=\|\bx(\sigma)-\hbm(\by;\btheta)\|^2,
 \label{eq:score-matching-loss-Gen}\\
\ell^{\rm CE}_\sigma(\by;\btheta)
 &=-\log\hq_\sigma(\by;\btheta).
 \label{eq:cross-entropy-loss-Gen}
\end{align}
\begin{proposition}[Verification for score matching and cross-entropy]
\label{prop:loss-assumptions-verification-Gen}
Suppose $a,\tau,\lambda$ are continuous on $J\subset[0,T)$ and
$\inf_J\tau>0$, $\inf_J\lambda>0$. Choose a compact convex parameter
box $K_0$ containing $\{\btheta^*(t):t\in J\}$ in its interior.
\begin{itemize}
\item[$(a)$] For cross-entropy, Assumptions~\ref{ass:MinimizerPopulation-Gen} and
\ref{ass:ZeroBandwidth-Gen} hold on $J$  with $K(t)=E$. 
\item[$(b)$] For score matching, Assumptions~\ref{ass:MinimizerPopulation-Gen} and
\ref{ass:ZeroBandwidth-Gen} hold for any compact family 
$K(t)$ whose graph is compact
in $J\times K_0$ and which satisfies
\eqref{eq:time-dependent-domain-margin}, provided that
the embeddings are affinely independent:
\begin{align}
\operatorname{rank}\big[\bx(1)-\bx(0),\ldots,
                         \bx(k-1)-\bx(0)\big]=k-1.
\label{eq:affine-independence-Gen}
\end{align}
\end{itemize}
If $a,\tau$ are $C^3$ on a neighborhood of $J$, then the risks are
jointly $C^3$ in $(t,\btheta)$ there.
\end{proposition}
The proof is in Appendix~\ref{proof:loss-assumptions-verification-Gen}.

\subsection{Asymptotics of parameter estimates}\label{sec:N_Asymp-Gen}

We next study the asymptotics of the empirical  risk minimization problem under
Assumption~\ref{ass:Bandwidth-Gen}.
Throughout  $I\subset(0,T)$ is compact. 

Define the loss gradient 
\begin{align}
\bpsi_\sigma(\by;\btheta)
&:=\nabla_{\btheta}\ell_\sigma(\by;\btheta)
 \in E,
\end{align}
Note that by Assumptions~\ref{ass:MinimizerPopulation-Gen} and
\ref{ass:ZeroBandwidth-Gen},
\begin{align}
  \E\big\{
 \bpsi_\Sigma\big(a(t)\bx(\Sigma)+\tau(t)\bG;\btheta^*(t)\big)\big\}=0\, ,
  \end{align}
  where $\Sigma\sim q$, $\bG\sim\normal(0,\id_d)$, and $\Sigma$ and $\bG$
  are independent.
We also define the Hessian and covariance
\begin{align}
\bH^*(t)
&:=\nabla_{\btheta}^2R_t(\btheta^*(t)),\label{eq:Hstar}\\
\bJ^*(t)
&:=\E\left[
 \bpsi_\Sigma\big(a(t)\bx(\Sigma)+\tau(t)\bG;\btheta^*(t)\big)
 \otimes
 \bpsi_\Sigma\big(a(t)\bx(\Sigma)+\tau(t)\bG;\btheta^*(t)\big)
 \right].
\label{eq:general-score-Hessian-moment}
\end{align}
 Here $\bH^*(t)$ is the self-adjoint linear operator
associated to the Hessian quadratic form, and for matrices
$\bA,\bB,\bC\in\reals^{(d+1)\times(k-1)}$,
\begin{align}
\nabla_{\btheta}^2R_t(\btheta^*(t))[\bA,\bB]
&=\left\langle\bA,\bH^*(t)[\bB]\right\rangle,\\
(\bA\otimes\bB)[\bC]
&:=\langle\bB,\bC\rangle\bA.
\label{eq:general-rank-one-operator}
\end{align}
Similarly $\bJ^*(t)$ is an operator on
$\reals^{(d+1)\times(k-1)}$.  

For the finite bandwidth empirical risk minimization problem, define
\begin{align}
\bH_b(t)&:=\nabla_{\btheta}^2R_{b,t}(\btheta^b(t)),\\
\bZ_{b,t}&:=w_b(U-t)\bpsi_\Sigma
 (a(U)\bx(\Sigma)+\tau(U)\bG;\btheta^b(t)),\\
\bOmega_b(t,s)&:=\E[\bZ_{b,t}\otimes\bZ_{b,s}],\qquad
\E\bZ_{b,t}=0,\label{eq:score-covariance-kernel-Gen}\\
\widehat{\bS}_n(t,\btheta)&:=\frac1n\sum_{i=1}^n
 w_b(t_i-t)\bpsi_{\sigma_i}(\bY_i;\btheta).
\label{eq:empirical-score-Gen}
\end{align}

\begin{proposition}[Localized population bias]
\label{prop:localized-population-second-order-bias-Gen}
Under Assumptions~\ref{ass:MinimizerPopulation-Gen},
\ref{ass:Kernel-Gen}, and \ref{ass:ZeroBandwidth-Gen}, let
$I\subset(0,T)$ be compact. Suppose $\lambda$ is $C^2$ on an open
neighborhood of $I$, and $(u,\btheta)\mapsto R_u(\btheta)$ is $C^3$
on an open neighborhood of $\{(t,\btheta^*(t)):t\in I\}$ in
$(0,T)\times E$. Then, uniformly on $I$,
\begin{align}
\btheta^b(t)-\btheta^*(t)&=b^2\boldsymbol\beta(t)+o(b^2),
 \label{eq:general-population-bias}\\
\boldsymbol\beta(t)&=-\frac{\mu_2}{2\mu_0\lambda(t)}
 \bH^*(t)^{-1}\left[
 \left.\partial_u^2\{\lambda(u)\bS(u,\btheta)\}
 \right|_{u=t,\,\btheta=\btheta^*(t)}\right],
 \qquad \bS(u,\btheta):=\nabla_{\btheta}R_u(\btheta),
 \label{eq:general-bias-coefficient}
\end{align}
where the derivatives $\partial_u$ in \eqref{eq:general-bias-coefficient} hold the parameter
$\btheta^*$ fixed.
\end{proposition}
The proof is in Appendix~\ref{proof:localized-population-second-order-bias-Gen}.

The next proposition characterize the effect of an error $\bDelta(t)$ on
the parameters estimates affects the sampling probabilities, using probability flow, 
to the linear order. 
\begin{proposition}[Linear response]\label{propo:Lin-Effect}
Under Assumptions~\ref{ass:GenerativeProcess-Gen} and
\ref{ass:PositiveNoise-Gen}, fix a compact interval $I\subset(0,T)$ and let $\bh_*$
be the probability-flow velocity obtained by setting
$\hbtheta_b=\btheta^*$.

Let
\begin{align}
\hbtheta_b(t)=\btheta^*(t)+\bDelta(t),
\qquad
\bDelta\in L^\infty(I;E),\qquad \bDelta(t)=0\quad(t\notin I),
\label{eq:general-parameter-perturbation}
\end{align}
and denote by $\rP_{\bDelta}$ the law of $\bZ_T$ under the corresponding
generation process.  For a column
$\bDelta_\rho(t)\in\reals^{d+1}$ of $\bDelta(t)$, set
\begin{align}
\overline{\bz}&:=(1,\bz^\top)^\top,\qquad
\pi_\rho^*(\bz,t):=\hq_\rho(\bz;\btheta^*(t)),\label{eq:general-posterior-notation}\\
\mathcal M_t(\bz)[\bDelta(t)]
&:=\sum_{\rho=1}^{k-1}\pi_\rho^*(\bz,t)
 \{\bx(\rho)-\bm(\bz,t)\}
 \left\langle\bDelta_\rho(t),\overline{\bz}\right\rangle.
\label{eq:general-denoiser-first-variation}
\end{align}
Then, as $\|\bDelta\|_\infty\to0$ over bounded measurable perturbations supported on $I$,
\begin{align}
\rP_{\bDelta}
&=\sum_{\sigma=0}^{k-1}
 \{q(\sigma)+\cL_\sigma(\bDelta)\}\delta_{\bx(\sigma)}
 +O(\|\bDelta\|_\infty^2),
\label{eq:general-terminal-law-linear-response}
\end{align}
where
\begin{align}
\cL_\sigma(\bDelta)
&=\int_I\left\langle\bL_\sigma(t),
 \bDelta(t)\right\rangle\,\de t,
\label{eq:general-linear-response-pairing}\\
  \langle\bL_\sigma(t),\bD\rangle &
   =-\int_{\cA_\sigma(t)}\nabla_{\bz}\mathbin\cdot
   \{c_1(t)p_t(\bz)\mathcal M_t(\bz)[\bD]\}\,\de\bz.
  \end{align}
or equivalently
\begin{align}
[\bL_\sigma(t)]_{j,\rho}
&=-c_1(t)\int_{\cA_\sigma(t)}
 \nabla_{\bz}\mathbin\cdot
 \left\{p_t(\bz)\pi_\rho^*(\bz,t)
 \{\bx(\rho)-\bm(\bz,t)\}\overline z_j\right\}\,\de\bz,
\label{eq:general-linear-response-kernel}
\end{align}
for $0\leq j\leq d$ and $1\leq\rho\leq k-1$, where
$\overline z_0=1$ and $\overline z_j=z_j$ for $j\geq1$.
The integrals are absolutely convergent, and
$\bL_\sigma\in C(I;E)$ is bounded. Further
\begin{align}
  \sum_{\sigma=0}^{k-1}\cL_\sigma(\bDelta)=0,\;\;\;\;\;\;
\sum_{\sigma=0}^{k-1}\bL_\sigma(t)=0,
\qquad t\in I.
\label{eq:general-response-kernel-conservation}
\end{align}
\end{proposition}

The proof is in Appendix~\ref{proof:Lin-Effect}.

\begin{example}[Binary variables]
For $d=1$, $k=2$, $\bx(0)=-1$, $\bx(1)=1$, and $q(1)=p$,
the denoiser is
\begin{align}
\hbm(z;\btheta)=\tanh\!\left(\frac{\theta_{1,0}+\theta_1^+z}{2}\right).
\label{eq:binary-reduction-Gen}
\end{align}
 If $q(+1) = p=1-q(-1)$, and $F_t$ is the scalar mixture distribution
function:
\begin{align}
F_t(x)= (1-p) \Phi\left(\frac{x+a(t)}{\sigma(t)}\right) + p \Phi\left(\frac{x-a(t)}{\tau(t)}\right)\, ,
\end{align}
the two basins are $\cA_-(t)=(-\infty,\gamma(t))$ and $\cA_+(t)=(\gamma(t),\infty)$ with
$\gamma(t)=F_t^{-1}(1-p)$.

Writing $\eta_*(z,t)=(\theta_{1,0}^*(t)+\theta_1^{*,+}(t)z)/2$,
integrating the spatial derivative in
\eqref{eq:general-linear-response-kernel} over this half-line gives
\begin{align}
  \bL_1(t)
 =\frac12 c_1(t)p_t(\gamma(t))\operatorname{sech}^2\Big((\theta_{1,0}^*(t)+\theta_1^{*,+}(t)\gamma(t))/2\Big)
 \begin{pmatrix}1\\\gamma(t)\end{pmatrix}.
\label{eq:binary-response-Gen}
\end{align}
\end{example}

\subsection{Undersmoothing regime}
\label{sec:Undersmooth}
Define the  fixed-time sandwich operator
\begin{align}
\bQ^*(t)&:=
 \bH^*(t)^{-1}\circ\bJ^*(t)\circ\bH^*(t)^{-1},
\label{eq:general-sandwich-covariance}
\end{align}
where $\bH^*(t)$, $\bJ^*(t)$ are the Hessian and covariance defined in 
Eqs.~\eqref{eq:Hstar} and \eqref{eq:general-score-Hessian-moment}.

\begin{proposition}[Shrinking-bandwidth empirical minimizers]
\label{prop:shrinking-bandwidth-erm-Gen}
Under Assumptions~\ref{ass:MinimizerPopulation-Gen},
\ref{ass:Kernel-Gen}, \ref{ass:ZeroBandwidth-Gen}, and
\ref{ass:Bandwidth-Gen}, let
$I\subset(0,T)$ be compact. Set $b=b_n$ and
$e_n=\sqrt{\log(1/b)/(nb)}$. With probability tending
to one, the empirical minimizers are unique and belong to
$\interior(K(t))$ simultaneously
on $I$, and, defining $\widehat{\bS}_n(t,\btheta)$ as in Eq.~\eqref{eq:empirical-score-Gen},
\begin{align}
\|\hbtheta_b-\btheta^b\|_\infty&=O_{\rP}(e_n),
 \label{eq:general-uniform-rate}\\
\hbtheta_b(t)-\btheta^b(t)
 &=-\bH_b(t)^{-1}[\widehat{\bS}_n(t,\btheta^b(t))]+\br_n(t),\\
\|\br_n\|_\infty&=O_{\rP}(e_n^2)=o_{\rP}(n^{-1/2}).
\label{eq:general-asymptotic-linearity}
\end{align}
\end{proposition}
The proof is in Appendix~\ref{proof:shrinking-bandwidth-erm-Gen}.

\begin{proposition}[Integrated parameter central limit theorem]
\label{propo:AsympGen}
Under Assumptions~\ref{ass:MinimizerPopulation-Gen},
\ref{ass:Bandwidth-Gen}, \ref{ass:Kernel-Gen}, and
\ref{ass:ZeroBandwidth-Gen}, let
$I\subset(0,T)$ be compact. Suppose $\lambda$ is $C^2$ on an open neighborhood
of $I$, and $(u,\btheta)\mapsto R_u(\btheta)$ is $C^3$ on an open
neighborhood of $\{(t,\btheta^*(t)):t\in I\}$ in $(0,T)\times E$.
For every
$\bphi\in C(I;\reals^{(d+1)\times(k-1)})$,
\begin{align}
&\sqrt n\int_I
 \left\langle\hbtheta_b(t)-\btheta^*(t),\bphi(t)\right\rangle
 \,\de t
\Rightarrow\normal\big(0,Q^*(\bphi)\big),
\label{eq:general-integrated-parameter-clt}\\
&Q^*(\bphi)
:=\int_I
 \frac1{\lambda(t)}
 \left\langle\bphi(t),\bQ^*(t)[\bphi(t)]\right\rangle\,\de t,
\label{eq:general-integrated-variance}
\end{align}
The conclusion holds jointly for every finite collection of test functions;
the limiting covariance associated with $\bphi$ and $\widetilde{\bphi}$ is
\begin{align}
Q^*(\bphi,\widetilde{\bphi})
&:=\int_I\frac1{\lambda(t)}\left\langle
 \bphi(t),\bQ^*(t)[\widetilde{\bphi}(t)]
 \right\rangle\,\de t.
\label{eq:general-integrated-cross-covariance}
\end{align}
\end{proposition}
The proof is in Appendix~\ref{proof:AsympGen}.

For convenience, recall that, with $\Sigma\sim q$ and
$\bG\sim\normal(0,\id_d)$ independent,
\begin{align}
\bH^*(t)
&:=\nabla_{\btheta}^2R_t(\btheta^*(t)),\\
\bJ^*(t)
&:=\E_{\Sigma,\bG}\left[
 \bpsi_\Sigma\big(a(t)\bx(\Sigma)+\tau(t)\bG;\btheta^*(t)\big)
 \otimes
 \bpsi_\Sigma\big(a(t)\bx(\Sigma)+\tau(t)\bG;\btheta^*(t)\big)
 \right],\\
\bQ^*(t)
&:=\bH^*(t)^{-1}\circ\bJ^*(t)\circ\bH^*(t)^{-1}.
\end{align}

We are now in position to state a formal version of Theorem 
\ref{thm:Overall-Limit-Gen-Informal}, given by the 
next theorem and corollary.
\begin{theorem}[Joint limit for generated class probabilities]
\label{thm:Overall-Limit-Gen}
Under Assumptions~\ref{ass:GenerativeProcess-Gen},
\ref{ass:Bandwidth-Gen}, \ref{ass:Kernel-Gen}, and
\ref{ass:PositiveNoise-Gen}, suppose
Assumptions~\ref{ass:MinimizerPopulation-Gen} and
\ref{ass:ZeroBandwidth-Gen} hold on every compact subinterval of
$(0,T)$ for the same domain map $t\mapsto K(t)$.
Suppose $\lambda$ is $C^2$ on $(0,T)$, and
$(u,\btheta)\mapsto R_u(\btheta)$ is $C^3$ on an open neighborhood of
$\{(t,\btheta^*(t)):t\in I\}$ for every compact $I\subset(0,T)$.
Let $\bL_\sigma(t)$ be the continuous  kernels
in Eq.~\eqref{eq:general-linear-response-kernel}.

Use $\hbtheta_b(t)$ on all of $(0,T)$ in the probability flow ODE \eqref{eq:ProbFlow-Gen}. Let
$\rP_n$ be its terminal law conditional on the training data, and set
$\hq_n(\sigma):=\rP_n(\{\bx(\sigma)\})$.
For $0<\varepsilon<T/2$, let $I_\varepsilon=[\varepsilon,T-\varepsilon]$
and let $\rP_{n,\varepsilon}$ be the terminal law obtained by using
$\hbtheta_b(t)$ on $I_\varepsilon$ and $\btheta^*(t)$ outside
$I_\varepsilon$. Assume that the endpoint contributions are negligible
at the central-limit scale: for every $\eta>0$,
\begin{align}
\lim_{\varepsilon\downarrow0}\limsup_{n\to\infty}
\prob\left(\sqrt n\,
 \|\rP_n-\rP_{n,\varepsilon}\|_{\rm TV}>\eta\right)=0.
\label{eq:general-endpoint-negligibility}
\end{align}
Assume also that, for every $0\leq\sigma\leq k-1$,
\begin{align}
\int_0^T\frac1{\lambda(t)}
 \langle\bL_\sigma(t),
      \bQ^*(t)[\bL_\sigma(t)]\rangle\,\de t<\infty.
\label{eq:general-full-interval-integrability}
\end{align}
For $0\le\sigma,\alpha\le k-1$, define
\begin{align}
V_{\sigma\alpha}(\lambda)
 &:=\int_0^T\frac1{\lambda(t)}
 \langle\bL_\sigma(t),
      \bQ^*(t)[\bL_\alpha(t)]\rangle\,\de t.
\label{eq:general-class-covariance}
\end{align}
Then
\begin{align}
\sqrt n(\hq_n-q)&\Rightarrow\normal_k(0,\bV(\lambda)).
 \label{eq:general-class-clt}
\end{align}
Moreover,
$\bV(\lambda)\boldsymbol1=0$.
\end{theorem}
The proof is in Appendix~\ref{proof:Overall-Limit-Gen}.

\begin{corollary}[Efficiency of cross-entropy]
\label{cor:general-class-clt}
For $t\in(0,T)$ and $\boldsymbol\vartheta$ in a neighborhood of
$\btheta^*(t)$, let $P_{t,\boldsymbol\vartheta}$ be the conditional
logistic model
\begin{align}
P_{t,\boldsymbol\vartheta}(\Sigma=\sigma\mid\bY=\by)
&=\hq_\sigma(\by;\boldsymbol\vartheta).
\label{eq:general-conditional-model}
\end{align}
Consider any loss $\ell$ satisfying the assumptions of
Theorem~\ref{thm:Overall-Limit-Gen} and the local Fisher-consistency
condition
\begin{align}
\E_{t,\boldsymbol\vartheta}\left[
 \nabla_{\btheta}\ell_\Sigma(\bY;\boldsymbol\vartheta)
 \right]&=0
\label{eq:general-loss-score-unbiased}
\end{align}
for every such $(t,\boldsymbol\vartheta)$, with differentiation under the
expectation permitted.  For fixed $\lambda$, let $\bV^\ell(\lambda)$ be
the covariance in the theorem for this loss and let
$\bV^{\rm CE}(\lambda)$ be the one for cross-entropy. Then
\begin{align}
\bV^{\rm CE}(\lambda)\preceq \bV^\ell(\lambda)\, .
\label{eq:general-ce-efficiency}
\end{align} 
\end{corollary}
The proof is in Appendix~\ref{proof:general-class-clt}.

\subsection{Choice of the time-sampling density}
\label{sec:Density}

\begin{proposition}[Optimal time-sampling density]
\label{prop:optimal-time-sampling-Gen}
Fix a positive semidefinite $k\times k$ matrix $\bW$. For an $C^2$ strictly positive
time-sampling density $\lambda$ and define
\begin{align}
C_{\sigma\alpha}(t)
&:=\left\langle\bL_\sigma(t),
 \bQ^*(t)[\bL_\alpha(t)]\right\rangle,
\label{eq:general-design-matrix}\\
h_W(t)&:=\operatorname{tr}\{\bW\bC(t)\},\\
\mathcal R_W(\lambda)
&:=\Tr\{\bW\bV(\lambda)\}
 =\int_0^T\frac{h_W(t)}{\lambda(t)}\,\de t.
\label{eq:general-design-risk}
\end{align}
Under the assumptions of Theorem~\ref{thm:Overall-Limit-Gen},
if $H_W:=\int_0^T\sqrt{h_W(t)}\,\de t<\infty$, then
\begin{align}
\inf_{\lambda\ge0:\,\int_0^T\lambda=1}\mathcal R_W(\lambda)
&=H_W^2.
\label{eq:LB_lambda_opt}
\end{align}
If $0<H_W<\infty$, the minimizing density is
\begin{align}
\lambda_{\rm opt}(t)&=
 \frac{\sqrt{h_W(t)}}{H_W}.
\label{eq:general-optimal-design}
\end{align}
The minimum is attained whenever $\lambda_{\rm opt}$ is admissible for
Theorem~\ref{thm:Overall-Limit-Gen}.
\end{proposition}

\begin{proposition}[SNR-reparametrization invariance]
\label{prop:Invariance}
Suppose the loss is cross-entropy or squared denoising loss and set
$\rho(t):=a(t)/\tau(t)$. If $\rho$ is $C^1$ and strictly monotone on
$(0,T)$, with $\dot\rho(t)\ne0$, then there is a positive semidefinite
matrix function $\widetilde{\bC}$ depending only on the normalized
signal-to-noise coordinate $r=\rho(t)$ such that
\begin{align}
\bC(t)&=\dot\rho(t)^2\widetilde{\bC}(\rho(t)).
\label{eq:general-design-snr-transform}
\end{align}
Consequently the optimal Gaussian-limit risk is
\begin{align}
\inf_\lambda\mathcal R_W(\lambda)
&=\left\{
 \int_{\rho((0,T))}
 \sqrt{\operatorname{tr}\{\bW\widetilde{\bC}(r)\}}\,\de r
 \right\}^2,
\label{eq:general-snr-invariance}
\end{align}
whenever the improper integral is finite. It therefore depends on the
schedule only through its SNR range and is invariant under monotone time
reparametrization.
\end{proposition}
The proofs are in Appendices~\ref{proof:optimal-time-sampling-Gen}
and~\ref{proof:Invariance-Gen}.

\subsection{Proofs for Sections~\ref{sec:setting-Gen} and \ref{sec:N_Asymp-Gen}}
\label{sec:proofs-N-Asymp-Gen}

\subsubsection{Proof of Lemma~\ref{lem:terminal-basins-Gen}}
\label{proof:terminal-basins-Gen}

Note that by construction 
$\bm(\bz,t)=\hbm(\bz;\btheta^*(t))$.
On every compact time interval below $T$, the velocity $\bh_*$
has linear growth and bounded spatial derivatives of every positive
order. Indeed, $\bh_*=c_1\bm+c_2\bz$, the embeddings are bounded,
the posterior is a softmax with bounded slope parameters, and
$\inf\tau>0$ on such an interval. Thus the flow and its inverse
are smooth spatial diffeomorphisms. The mixture density
\eqref{eq:mixture-density-Gen} solves the continuity equation for
$\bh_*$, so a trajectory initialized with density $p_s$ has density
$p_t$ at every $s\leq t<T$.

For each fixed $t$, represent the marginal as
$\bY_t=a(t)\bX+\tau(t)\bG_t$, where $\bG_t\sim\normal(0,\id_d)$ is
independent of $\bX$. Since
$c_2=\dot\tau/\tau$ and $c_1=\dot a-a\dot\tau/\tau$, we have
\begin{align}
\bh_*(\bz,t)
&=\dot a(t)\E[\bX\mid\bY_t=\bz]
 +\dot\tau(t)\E[\bG_t\mid\bY_t=\bz].
\label{eq:general-posterior-velocity}
\end{align}
Positive noise at zero implies $s(0)>0$. Since $s$ is non-decreasing,
$\tau=\tau_0\sqrt{s}$ is smooth on $[0,T]$, and both
$\dot a$ and $\dot\tau$ are integrable. For a true trajectory
initialized with density $p_s$, conditional Jensen and Fubini-Tonelli give
\begin{align}
\E\int_s^T\|\bh_*(\bZ_t,t)\|\,\de t
&\leq \E\|\bX\|\int_s^T|\dot a(t)|\,\de t
 +\E\|\bG_t\|\int_s^T|\dot\tau(t)|\,\de t<\infty.
\label{eq:general-finite-expected-path-length}
\end{align}
Therefore, almost every trajectory 
$t\mapsto \bZ_t$ has finite length and hence has a terminal
limit $\bZ_T=\lim_{t\uparrow T}\bZ_t$. Its law is 
$\bZ_T\sim\prob_q$, because $p_t(\bz)\,\de\bz$ converges weakly
to $\prob_q$ as $t\uparrow T$. Since $p_s$ is everywhere positive,
the exceptional set of starting points is also Lebesgue-null.

The limit conditions defining each basin are Borel because
the trajectories are continuous before $T$. The distinct embeddings
make the basins disjoint. Flow composition and invertibility imply
the set identity in \eqref{eq:general-measurable-basin-transport}.
Applying the terminal-law conclusion to a trajectory initialized
with $p_t$ gives the mass identity in that equation.

\subsubsection{Proof of Proposition~\ref{prop:loss-assumptions-verification-Gen}}
\label{proof:loss-assumptions-verification-Gen}

We begin with Assumption \ref{ass:MinimizerPopulation-Gen}.
Equal embedding norms imply the posterior log-odds
\begin{align}
\log\frac{\prob(\Sigma=\rho\mid\bX_t=\by)}
 {\prob(\Sigma=0\mid\bX_t=\by)}
 =-r(\rho)+\frac{a(t)}{\tau(t)^2}
 \langle\bx(\rho)-\bx(0),\by\rangle.
\label{eq:loss-posterior-Gen}
\end{align}
Thus $\pi^*(\by,t):=(\hq_\sigma(\by;\btheta^*(t)))_{\sigma=0}^{k-1}$
is the true posterior. Conditional expectation gives the excess risks
\begin{align}
R_t^{\rm CE}(\btheta)-R_t^{\rm CE}(\btheta^*)
 &=\E\left[\sum_{\sigma=0}^{k-1}\pi^*_\sigma(\bY,t)
 \log\frac{\pi^*_\sigma(\bY,t)}{\hq_\sigma(\bY;\btheta)}\right],\\
R_t^{\rm SM}(\btheta)-R_t^{\rm SM}(\btheta^*)
 &=\E\|\hbm(\bY;\btheta)-\bm(\bY,t)\|^2.
\label{eq:loss-excess-risks-Gen}
\end{align}
Hence, $R_t^{\rm L}(\btheta)-R_t^{\rm L}(\btheta^*)\ge 0$ for ${\rm L}\in\{{\rm CE}, {\rm SM}\}$. 

For cross-entropy $R_t^{\rm CE}(\btheta)-R_t^{\rm CE}(\btheta^*) = 0$
$\hq_\sigma(\bY;\btheta)=hq_\sigma(\bY;\btheta^*)$ for almost every $\bY,\sigma$,
which in turns implies  $\btheta=\btheta^*$.
For squared loss, $R_t^{\rm SM}(\btheta)-R_t^{\rm SM}(\btheta^*) = 0$
implies $\hbm(\bY;\btheta)=\bm(\bY,t)$ for almost every $\bY$,
which in turns implies  $\hq_\sigma(\bY;\btheta)=hq_\sigma(\bY;\btheta^*)$ 
because of Eq.~\eqref{eq:affine-independence-Gen},  whence $\btheta=\btheta^*$.

We now pass to  Assumption \ref{ass:ZeroBandwidth-Gen}. For $\bA\in E$, let $v_0=0$ and
$v_\rho=\langle\bA_\rho,\overline{\by}\rangle$,
$\overline{\by}=(1,\by^\top)^\top$. Let
$\mathcal M_t(\by)=D_{\btheta}\hbm(\by;\btheta^*(t))$.
The Hessian quadratic forms at the optimum are
\begin{align}
\langle\bA,\bH^{*,\rm CE}(t)[\bA]\rangle
 &=\E\left[\sum_{\rho=1}^{k-1}\pi_\rho^*v_\rho^2
             -\left(\sum_{\rho=1}^{k-1}\pi_\rho^*v_\rho\right)^2\right],\\
\langle\bA,\bH^{*,\rm SM}(t)[\bA]\rangle
 &=2\E\|\mathcal M_t(\bY)[\bA]\|^2.
\label{eq:loss-hessians-Gen}
\end{align}
All posterior coordinates are positive, whence the conditional variance in the
first line vanishes only if all $v_\rho$ equal $v_0=0$.

Under the affine independence assumption Eq.~\eqref{eq:affine-independence-Gen}  
the differential of the mean with respect to 
probability perturbations whose coordinates sum to zero has full rank $k-1$. Hence the second line
also vanishes only if all $v_\rho=0$ almost everywhere. Full support of
$\bY$ then gives $\bA=0$. Both Hessians are positive definite; continuity
and compactness of $J$ make this uniform. This proves point 3 
of Assumption \ref{ass:ZeroBandwidth-Gen}.

The unique population minimizer over $E$ is therefore also the unique
minimizer over every $K(t)$ containing $\btheta^*(t)$. The compact-graph
and uniform-margin hypotheses give precisely the constrained-domain
alternative in Assumption~\ref{ass:ZeroBandwidth-Gen}; cross-entropy
also satisfies its unconstrained convex alternative.
This proves point 1 of Assumption \ref{ass:ZeroBandwidth-Gen}.

To prove point 2, note that derivatives of softmax with respect to logits are bounded at every
fixed order. Each parameter differentiation supplies at most a factor
$\|\overline{\by}\|$. Consequently both losses and their first three
parameter derivatives have polynomial envelopes, with third derivative
bounded by $C(1+\|\by\|)^3$. The cross-entropy loss itself grows at most
linearly on $K_0$ and is convex on $E$; squared loss is bounded but may be
nonconvex, which explains the compact time-dependent domains. Gaussian
moments justify differentiation and continuity under expectation.
Differentiating $\by=a(t)\bx(\Sigma)+\tau(t)\bG$ also gives polynomial
Gaussian envelopes for three joint derivatives if $a,\tau$ are $C^3$.
This proves all the claims.

For completeness, the score operators used in the covariance are
particularly explicit for these losses:
\begin{align}
[\bpsi^{\rm CE}_\sigma(\by;\btheta)]_\rho
 &=(\hq_\rho(\by;\btheta)-\mathbf1_{\{\sigma=\rho\}})
 \overline{\by},\\
\bpsi^{\rm SM}_\sigma(\by;\btheta)
 &=2(D_{\btheta}\hbm(\by;\btheta))^*
 [\hbm(\by;\btheta)-\bx(\sigma)].
\label{eq:loss-scores-Gen}
\end{align}
Here the superscript $*$ on the differential denotes its adjoint.
For cross-entropy the conditional score covariance equals its Hessian,
so $\bJ^{*,\rm CE}=\bH^{*,\rm CE}$ and
$\bQ^{*,\rm CE}=(\bH^{*,\rm CE})^{-1}$.

\subsubsection{Zero-bandwidth limit of the population risk}

\begin{lemma}\label{lem:localized-population-geometry-Gen}
Under Assumptions~\ref{ass:MinimizerPopulation-Gen},
\ref{ass:Kernel-Gen}, and \ref{ass:ZeroBandwidth-Gen}, define
\begin{align}
d_b(t)&:=\int_0^T w_b(u-t)\lambda(u)\,\de u,
&\overline R_{b,t}&:=R_{b,t}/d_b(t).
\label{eq:normalized-localized-risk-Gen}
\end{align}
Then $c_-b\le d_b(t)\le c_+b$ uniformly on $I$ for all sufficiently
small $b$, and
\begin{align}
\lim_{b\to 0}\max_{0\le j\le2}\sup_{t\in I,\,\btheta\in K_0}
 \|D^j\overline R_{b,t}(\btheta)-D^jR_t(\btheta)\|=0.
\label{eq:zero-bandwidth-risk-C2-limit-Gen}
\end{align}
The minimizers $\btheta^b(t)$ are unique and belong to
$\interior(K(t))$ for all sufficiently small $b$, uniformly on $I$, and
\begin{align}
  \lim_{b\to 0}\|\btheta^b-\btheta^*\|_\infty&=0,\\
  \lim_{b\to 0}\sup_{t\in I}\left\|\frac{\bH_b(t)}{d_b(t)}-\bH^*(t)
 \right\|_{\rm op}&=0.
\label{eq:population-Hessian-convergence-Gen}
\end{align}
For every fixed neighborhood of the target curve, the population risk
has a uniform positive gap over feasible points outside that
neighborhood; in the unconstrained case the corresponding gap holds on
$K_0$. If $I\Subset(0,T)$, then also, uniformly over $I$,
\begin{align}
\bH_b(t)^{-1}
 &=\frac1{b\mu_0\lambda(t)}\bH^*(t)^{-1}+o(b^{-1})\, .
\label{eq:general-localized-Hessian-limit}
\end{align}
\end{lemma}

\begin{proof}
The Gaussian representation and the polynomial upper bounds in
Assumption~\ref{ass:ZeroBandwidth-Gen} imply joint continuity of
$D^jR_t(\btheta)$ on $J\times K_0$ for $j\le2$. The bounds on $d_b$
follow from local positivity of $\lambda$, compactness of $I$, and
$w_1(0)>0$. Since $\overline R_{b,t}$ is a normalized weighted average
of $R_u$ over $|u-t|\le b$, uniform continuity gives
\eqref{eq:zero-bandwidth-risk-C2-limit-Gen}, including when $0\in I$.

In the constrained case, compactness of $\mathcal K_J$
(as defined in Assumption \ref{ass:ZeroBandwidth-Gen}), continuity, and
uniqueness of the minimizer give, for every $\delta>0$, a positive gap
for $R_t(\btheta)-R_t(\btheta^*(t))$ uniformly over
\begin{align}
\{(t,\btheta):t\in I,\ \btheta\in K(t),\
  \|\btheta-\btheta^*(t)\|\ge\delta\}.
\end{align}
In the unconstrained case, the same compactness argument gives the gap
on $I\times K_0$. The $j=0$ part of
\eqref{eq:zero-bandwidth-risk-C2-limit-Gen} therefore places every
constrained minimizer, and every minimizer localized to $K_0$ in the
unconstrained case, in an arbitrarily small neighborhood of the target
curve.

The $j=2$ convergence, \eqref{eq:Hessian-positive-Gen}, and the uniform
interior margin give strict convexity on fixed-radius balls around the
target curve. Thus the constrained minimizers are unique and interior.
In the unconstrained case, convexity makes the localized interior
stationary point global, while local strict convexity makes it the
unique global minimizer. This proves uniform convergence of
$\btheta^b$ to $\btheta^*$; evaluating the uniform Hessian convergence
along these minimizers proves
\eqref{eq:population-Hessian-convergence-Gen}.

Finally, if $I\Subset(0,T)$, the kernel is not truncated by the temporal
boundary and $d_b(t)/b\to\mu_0\lambda(t)$ uniformly. Combining this with
\eqref{eq:population-Hessian-convergence-Gen} and continuity of matrix
inversion on uniformly positive definite operators proves
\eqref{eq:general-localized-Hessian-limit}.
\end{proof}

\subsubsection{Proof of Proposition~\ref{prop:localized-population-second-order-bias-Gen}}
\label{proof:localized-population-second-order-bias-Gen}

Lemma~\ref{lem:localized-population-geometry-Gen} gives unique localized
minimizers converging uniformly to $\btheta^*$. Since
$I\Subset(0,T)$,
the localized score equation can be written
\begin{align}
0&=b^{-1}\nabla R_{b,t}(\btheta^b(t))
   =\bF(b,t,\btheta^b(t)),\\
\bF(b,t,\btheta)
 &:=\int_{-1}^1w_1(v)\lambda(t+bv)\bS(t+bv,\btheta)\,\de v\, ,
\label{eq:general-localized-score}
\end{align}
where $\bS(t,\btheta)$ is defined in
Eq.~\eqref{eq:general-bias-coefficient}.
This is $C^2$ jointly near the zero-bandwidth curve, and
\[
D_{\btheta}\bF(0,t,\btheta^*(t))
 =\mu_0\lambda(t)\bH^*(t)
\]
is uniformly invertible.
The implicit function theorem gives a $C^2$ solution in $b$, uniformly
over $I$, identified with $\btheta^b$ for positive $b$ by uniqueness.
Kernel symmetry makes $\bF$ and its local solution even in
$b$, so $\partial_b\btheta^b|_{b=0}=0$. Differentiating twice gives
\begin{align}
\left.\partial_b^2\btheta^b(t)\right|_{b=0}
 =-\frac{\mu_2}{\mu_0\lambda(t)}\bH^*(t)^{-1}
 \left[\left.\partial_u^2(\lambda(u)\bS(u,\btheta))
 \right|_{u=t,\,\btheta=\btheta^*(t)}\right].
\end{align}
Taylor's theorem with a uniform remainder proves
\eqref{eq:general-population-bias}--\eqref{eq:general-bias-coefficient}.

\subsubsection{Proof of Proposition~\ref{propo:Lin-Effect}}
\label{proof:Lin-Effect}

As pointed out in previous proofs already, the posterior log-odds satisfy
\begin{align}
\log\frac{\prob(\Sigma=\rho\mid\bX_t=\bz)}
              {\prob(\Sigma=0\mid\bX_t=\bz)}
&=-r(\rho)+\frac{a(t)}{\tau(t)^2}
 \left\langle\bx(\rho)-\bx(0),\bz\right\rangle.
\label{eq:general-posterior-log-odds}
\end{align}
Thus $\pi_\rho^*(\bz,t)=\prob(\Sigma=\rho\mid\bX_t=\bz)$,
$\hbm(\bz,t;\btheta^*(t))=\bm(\bz,t)$, and $\bh_*$ is the true
probability-flow velocity with density $p_t$.

We next compute the differential of the denoiser.  For $\rho\geq1$, the perturbation of the $\rho$-th logit at
$(\bz,t)$ is
\begin{align}
\de u_\rho
=\left\langle\bDelta_\rho(t),\overline{\bz}\right\rangle,
\qquad \de u_0=0.
\end{align}
The softmax differential is
\begin{align}
D\pi_\alpha^*[\bDelta]
=\pi_\alpha^*
 \left\{\de u_\alpha-
 \sum_{\rho=1}^{k-1}\pi_\rho^*\de u_\rho\right\}.
\end{align}
Since $\bm=\sum_{\alpha=0}^{k-1}\pi_\alpha^*\bx(\alpha)$, summing the last
display against $\bx(\alpha)$ gives
\begin{align}
D_{\btheta}\hbm(\bz,t;\btheta^*(t))[\bDelta(t)]
&=\sum_{\rho=1}^{k-1}\pi_\rho^*(\bz,t)
 \{\bx(\rho)-\bm(\bz,t)\}
 \left\langle\bDelta_\rho(t),\overline{\bz}\right\rangle\nonumber\\
&=\mathcal M_t(\bz)[\bDelta(t)].
\label{eq:general-denoiser-differential-proof}
\end{align}
Write $I=[t_-,t_+]$, $\varepsilon=\|\bDelta\|_\infty$, and
$\bh_{\eta\bDelta}$ for the velocity with parameters
$\btheta^*+\eta\bDelta$, $0\leq\eta\leq1$. Outside $I$ all these
velocities equal $\bh_*$. The first velocity variation is
\begin{align}
\bv_t[\bDelta](\bz)
&:=\left.\partial_\eta\bh_{\eta\bDelta}(\bz,t)\right|_{\eta=0}
 =c_1(t)\mathcal M_t(\bz)[\bDelta(t)],\\
\bh_{\bDelta}(\bz,t)-\bh_*(\bz,t)
&=\bv_t[\bDelta](\bz)
 +O\bigl(\varepsilon^2(1+\|\bz\|)^2\bigr).
\label{eq:general-velocity-first-variation}
\end{align}

For sufficiently small $\varepsilon$, all parameters
$\btheta^*(t)+\eta\bDelta(t)$ lie in a fixed compact neighborhood of
the target curve. Write
\begin{align}
\bh_{\eta\bDelta}(\bz,t)
 &=c_2(t)\bz+\bb_{\eta,t}(\bz),
&
\bb_{\eta,t}(\bz)
 &:=c_1(t)\hbm(\bz,t;\btheta^*(t)+\eta\bDelta(t)).
\end{align}
The softmax probabilities lie in $[0,1]$, the embeddings are fixed, and
the slope parameters range over a compact set. Consequently, for every
fixed  order $r\geq0$,
\begin{align}
\sup_{t\in I,\,\eta\in[0,1],\,\bz\in\reals^d}
 \|D_{\bz}^r\bb_{\eta,t}(\bz)\|&\leq C_r,\label{eq:general-velocity-spatial-bounds}\\
\|\partial_\eta^jD_{\bz}^r\bb_{\eta,t}(\bz)\|
 &\leq C_r\varepsilon^j(1+\|\bz\|)^j,
 \qquad j=1,2.
\label{eq:general-velocity-parameter-bounds}
\end{align}
Thus the velocities are measurable in time, smooth in space, and
globally Lipschitz in space, with constants uniform in $t$ and $\eta$.

Let $\bPhi_{t_-,t}^\eta$ be the forward flow of
$\bh_{\eta\bDelta}$ and let
$\Psi_t^\eta:=(\bPhi_{t_-,t}^\eta)^{-1}$. If
$\bX_t^\eta=\bPhi_{t_-,t}^\eta(\bx)$, variation of constants gives
\begin{align}
\bX_t^\eta
 &=\alpha_t\bx+
   \alpha_t\int_{t_-}^t\alpha_u^{-1}
      \bb_{\eta,u}(\bX_u^\eta)\,\de u,
&
\alpha_t&:=\exp\left\{\int_{t_-}^t c_2(u)\,\de u\right\}
 =\frac{\tau(t)}{\tau(t_-)}.
\end{align}
The factors $\alpha_t$ are bounded above and away from zero on $I$, and
\eqref{eq:general-velocity-spatial-bounds} therefore yields
\begin{align}
\sup_{t\in I,\,\eta\in[0,1],\,\bz\in\reals^d}
 \|\Psi_t^\eta(\bz)-\alpha_t^{-1}\bz\|&\leq C,
&
1+\|\Psi_t^\eta(\bz)\|&\asymp1+\|\bz\|.
\label{eq:general-inverse-flow-growth}
\end{align}

For completeness, the spatial Jacobian
$\bA_t^\eta:=D_{\bx}\bX_t^\eta$ and the first two parameter variations
$\bY_t^\eta:=\partial_\eta\bX_t^\eta$ and
$\bW_t^\eta:=\partial_\eta^2\bX_t^\eta$ satisfy
\begin{align}
\dot{\bA}_t^\eta
 &=D_{\bz}\bh_{\eta\bDelta}(\bX_t^\eta,t)\bA_t^\eta,
& \bA_{t_-}^\eta&=\id_d,\\
\dot{\bY}_t^\eta
 &=D_{\bz}\bh_{\eta\bDelta}\bY_t^\eta
   +\partial_\eta\bh_{\eta\bDelta},
& \bY_{t_-}^\eta&=0,\\
\dot{\bW}_t^\eta
 &=D_{\bz}\bh_{\eta\bDelta}\bW_t^\eta
   +D_{\bz}^2\bh_{\eta\bDelta}[\bY_t^\eta,\bY_t^\eta]
   +2D_{\bz}\partial_\eta\bh_{\eta\bDelta}[\bY_t^\eta]
   +\partial_\eta^2\bh_{\eta\bDelta},
& \bW_{t_-}^\eta&=0,
\label{eq:general-flow-variational-equations}
\end{align}
where the coefficients without displayed arguments are evaluated at
$(\bX_t^\eta,t)$. Equations
\eqref{eq:general-velocity-spatial-bounds}--
\eqref{eq:general-velocity-parameter-bounds} and Gronwall's inequality
bound $\bA_t^\eta$ and its inverse uniformly and give, for $j=1,2$ and
$0\leq r\leq2$,
\begin{align}
\|\partial_\eta^jD_{\bz}^r\Psi_t^\eta(\bz)\|
 &\leq C\varepsilon^j(1+\|\bz\|)^N.
\label{eq:general-inverse-flow-parameter-bounds}
\end{align}
Here the inverse-flow bounds follow by differentiating
$\bPhi_{t_-,t}^\eta(\Psi_t^\eta(\bz))=\bz$ in $\bz$ and $\eta$.

Denote by $p_t^{\eta\bDelta}(\bz)$ the density at time $t$ for the
velocity $\bh_{\eta\bDelta}$. Since the perturbation vanishes before
$t_-$, its initial density is $p_{t_-}^{\eta\bDelta}=p_{t_-}$. The
change-of-variables formula therefore gives
\begin{align}
p_t^{\eta\bDelta}(\bz)
 &=p_{t_-}(\Psi_t^\eta(\bz))\,|\det D_{\bz}\Psi_t^\eta(\bz)|,\label{eq:DensityJacobian}
\end{align}
which implies the bounds
\begin{align}
\bigl|\partial_\eta^j p_t^{\eta\bDelta}(\bz)\bigr|
 &\leq C\varepsilon^j(1+\|\bz\|)^N e^{-c\|\bz\|^2},
 \qquad j=0,1,2.
\label{eq:general-density-parameter-bounds}
\end{align}
Here and below constants are uniform over
$\varepsilon\in(0,\varepsilon_0)$ (for $\eps_0>0$ sufficiently small), $t\in I$, and $\eta\in[0,1]$.
For $j=0$, the convention $\varepsilon^0=1$ means only that the density
itself is uniformly dominated; no smallness in $\varepsilon$ is claimed.
The same argument applied after one spatial differentiation gives
\begin{align}
\big\|\partial_\eta^j\nabla p_t^{\eta\bDelta}(\bz)\big\|
 &\leq C\varepsilon^j(1+\|\bz\|)^N e^{-c\|\bz\|^2},
 \qquad j=0,1,2.
\label{eq:general-density-parameter-bounds-gradient}
\end{align}
The Gaussian factor follows from
\eqref{eq:general-inverse-flow-growth} with the Gaussian decay of
$p_{t_-}$ and its derivatives; all remaining factors in the
differentiated density formula \eqref{eq:DensityJacobian} have polynomial growth.

Define the signed density variation
$\dot p_t[\bDelta]:=\left.\partial_\eta
p_t^{\eta\bDelta}\right|_{\eta=0}$, where the dot here denotes
variation in $\eta$, not time differentiation.
Taylor's formula and \eqref{eq:general-density-parameter-bounds} give
\begin{align}
|p_t^{\bDelta}(\bz)-p_t(\bz)|
 &\leq C\varepsilon(1+\|\bz\|)^N e^{-c\|\bz\|^2},
 \label{eq:general-density-weighted-bound}\\
\sup_{t\in I}
 \|p_t^{\bDelta}-p_t-\dot p_t[\bDelta]\|_{L^1(\reals^d)}
 &\leq C\varepsilon^2.
 \label{eq:general-density-L1-expansion}
\end{align}

\paragraph{Transport over basins.}
Let $\mathcal U_{s,t}$ be the transport operator on signed densities
for the true flow from $s$ to $t$. Explicitly, for $f\in L^1(\reals^d)$,
\begin{align}
\mathcal U_{s,t}[f](\bz)
 &:=f\bigl(\bPhi_{t,s}^*(\bz)\bigr)
   \left|\det D_{\bz}\bPhi_{t,s}^*(\bz)\right|,
 \qquad t_-\leq s\leq t\leq t_+,
\label{eq:general-true-transport-operator}
\end{align}
where $\bPhi_{t,s}^*=(\bPhi_{s,t}^*)^{-1}$. Thus
$\mathcal U_{s,t}[f]$ is the density of the pushforward of the signed
measure $f(\bz)\,\de\bz$ by $\bPhi_{s,t}^*$; in particular it preserves
total signed mass and is an $L^1$ contraction.
Differentiating the continuity equation in $\eta$ gives, in its
integrated or distributional sense,
\begin{align}
\partial_t\dot p_t[\bDelta]
 +\nabla_{\bz}\mathbin\cdot(\bh_*\dot p_t[\bDelta])
 &=-\nabla_{\bz}\mathbin\cdot(p_t\bv_t[\bDelta]),
\qquad
\dot p_{t_-}[\bDelta]=0.
\end{align}
Hence
\begin{align}
\dot p_v[\bDelta]
 &=-\int_{t_-}^v
 \mathcal U_{t,v}\!\left[
 \nabla_{\bz}\mathbin\cdot(p_t\bv_t[\bDelta])\right]\,\de t,
 \qquad v\in I.
\label{eq:general-linearized-density-transport}
\end{align}
The integrals are well-defined in $L^1$ by the Gaussian bounds
\eqref{eq:general-density-parameter-bounds}, \eqref{eq:general-density-parameter-bounds-gradient}.
The basin transport identity
\eqref{eq:general-measurable-basin-transport} implies, for every
integrable signed density $f$, and any $t<v<T$,
\begin{align}
\int_{\cA_\sigma(v)}\mathcal U_{t,v}[f](\bz)\,\de\bz
 &=\int_{\cA_\sigma(t)}f(\bz)\,\de\bz.
\label{eq:general-basin-transfer-identity}
\end{align}
This is a change of variables.

We also record an identity for later use. Put
$\bV_t^{\bDelta}:=\bh_{\bDelta}(\cdot,t)-\bh_*(\cdot,t)$.
Subtracting the two continuity equations gives the source
$-\nabla\mathbin\cdot(p_t^{\bDelta}\bV_t^{\bDelta})$ under the true
transport. Consequently,
\begin{align}
\int_{\cA_\sigma(v)}p_v^{\bDelta}(\bz)\,\de\bz-q(\sigma)
 &=-\int_{t_-}^v\int_{\cA_\sigma(t)}
 \nabla_{\bz}\mathbin\cdot
 \left\{p_t^{\bDelta}(\bz)\bV_t^{\bDelta}(\bz)\right\}
 \,\de\bz\,\de t,
 \qquad v\in I.
\label{eq:general-exact-basin-volume}
\end{align}
Further $\nabla_{\bz}\mathbin\cdot
\left\{p_t^{\bDelta}(\bz)\bV_t^{\bDelta}(\bz)\right\}$ is integrable uniformly in time, 
by the bounds \eqref{eq:general-density-parameter-bounds}, \eqref{eq:general-density-parameter-bounds-gradient}
 above.

After $t_+$ the perturbed flow equals the true flow.
Lemma~\ref{lem:terminal-basins-Gen} and absolute continuity of
$p_{t_+}^{\bDelta}$ therefore imply
\begin{align}
\rP_{\bDelta}(\{\bx(\sigma)\})
 &=\int_{\cA_\sigma(t_+)}p_{t_+}^{\bDelta}(\bz)\,\de\bz.
\label{eq:general-terminal-mass-volume}
\end{align}
In particular, the perturbed terminal law has no mass outside the
embeddings. Equations \eqref{eq:general-linearized-density-transport}
and \eqref{eq:general-basin-transfer-identity}, with $v=t_+$, identify
the derivative of \eqref{eq:general-terminal-mass-volume} with the
linear functional defined by
\eqref{eq:general-linear-response-pairing}--\eqref{eq:general-linear-response-kernel}.
Integration over the
basin partition is an $L^1$ contraction into the class-mass vector:
\begin{align}
\sum_{\sigma=0}^{k-1}
 \left|\int_{\cA_\sigma(t_+)}f(\bz)\,\de\bz\right|
 &\leq\|f\|_{L^1(\reals^d)}.
\label{eq:general-basin-L1-contraction}
\end{align}
Thus \eqref{eq:general-density-L1-expansion} gives the claimed
$O(\|\bDelta\|_\infty^2)$ remainder in total variation, uniformly over
bounded measurable perturbations.

Substituting the velocity variation
\eqref{eq:general-velocity-first-variation} into
\eqref{eq:general-linearized-density-transport} and using
\eqref{eq:general-basin-transfer-identity}, the coefficient of
$\Delta_{j,\rho}(t)$ in the terminal mass variation is precisely
\eqref{eq:general-linear-response-kernel}. This gives the pairing
\eqref{eq:general-linear-response-pairing}. The kernel integrands have
uniform polynomial-times-Gaussian bounds, so they are absolutely
integrable and the kernels are bounded.

To prove the continuity claim $\bL_\sigma\in C(I;E)$, fix $\bD\in E$ and define
\begin{align}
F_{t,\bD}(\bz)
 &:=\nabla_{\bz}\mathbin\cdot
 \{c_1(t)p_t(\bz)\mathcal M_t(\bz)[\bD]\},
&
\tbPsi_t&:=\boldsymbol\Phi^*_{t_+,t}.
\end{align}
The definition of the response kernel first gives
\begin{align}
\langle\bL_\sigma(t),\bD\rangle
 &=-\int_{\cA_\sigma(t)}F_{t,\bD}(\bz)\,\de\bz.
\end{align}
By \eqref{eq:general-measurable-basin-transport},
$\bPhi^*_{t,t_+}(\cA_\sigma(t))=\cA_\sigma(t_+)$. Hence the substitution
$\by=\bPhi^*_{t,t_+}(\bz)$, or equivalently
$\bz=\tbPsi_t(\by)$, changes the moving basin to the fixed measurable set
$\cA_\sigma(t_+)$ and yields
\begin{align}
\langle\bL_\sigma(t),\bD\rangle
 &=-\int_{\cA_\sigma(t_+)}
 F_{t,\bD}(\tbPsi_t(\by))\,|\det D\tbPsi_t(\by)|\,\de\by.
\label{eq:general-response-fixed-domain}
\end{align}
The integrand is continuous in $t$ and bounded in absolute value by
$C\|\bD\|_{\rm F}(1+\|\by\|)^N e^{-c\|\by\|^2}$ uniformly on $I$.
Dominated convergence, applied to a basis of the finite-dimensional
space $E$, proves $\bL_\sigma\in C(I;E)$.

Finally, to prove \eqref{eq:general-response-kernel-conservation}, fix $t\in I$, $0\leq j\leq d$, and
$1\leq\rho\leq k-1$, and set
\begin{align}
\bG_{t,j,\rho}(\bz)
 &:=p_t(\bz)\pi_\rho^*(\bz,t)
   \{\bx(\rho)-\bm(\bz,t)\}\overline z_j.
\end{align}
Because the measurable basins partition $\reals^d$ up to a null set and
the divergence is absolutely integrable,
\begin{align}
\sum_{\sigma=0}^{k-1}[\bL_\sigma(t)]_{j,\rho}
 &=-c_1(t)\int_{\reals^d}
   \nabla_{\bz}\mathbin\cdot\bG_{t,j,\rho}(\bz)\,\de\bz.
\label{eq:general-response-kernel-sum-proof}
\end{align}
To evaluate the last integral, choose smooth cutoffs $\chi_R$ that equal
one on $\{\|\bz\|\leq R\}$, vanish outside
$\{\|\bz\|\leq2R\}$, and satisfy $\|\nabla\chi_R\|\leq C/R$.
The polynomial-times-Gaussian bounds on $\bG_{t,j,\rho}$ and its
divergence allow dominated convergence and integration by parts, giving
\begin{align}
\int_{\reals^d}\nabla_{\bz}\mathbin\cdot\bG_{t,j,\rho}\,\de\bz
 &=\lim_{R\to\infty}\int_{\reals^d}
   \chi_R\nabla_{\bz}\mathbin\cdot\bG_{t,j,\rho}\,\de\bz\nonumber\\
 &=-\lim_{R\to\infty}\int_{\reals^d}
   \langle\nabla\chi_R,\bG_{t,j,\rho}\rangle\,\de\bz=0.
\end{align}
Equation~\eqref{eq:general-response-kernel-sum-proof} therefore proves
$\sum_\sigma\bL_\sigma(t)=0$. Substitution into
\eqref{eq:general-linear-response-pairing} gives
$\sum_\sigma\cL_\sigma(\bDelta)=0$, completing
\eqref{eq:general-response-kernel-conservation}.

\subsubsection{Proof of Proposition~\ref{prop:shrinking-bandwidth-erm-Gen}}
\label{proof:shrinking-bandwidth-erm-Gen}

Write $b=b_n$ and
$e_n=\sqrt{\log(1/b)/(nb)}$. The bandwidth conditions imply that,
eventually, $n^{-1/2}<b<n^{-1/4}$. Hence
$\log(1/b)\asymp\log n$, $nb>\sqrt n$, and $e_n\to0$.
All constants below may depend on $I,J,K_0$, but not on $n$.

\paragraph{Uniform empirical bounds.}
We first establish
\begin{align}
\max_{0\le j\le2}\sup_{t\in I,\,\btheta\in K_0}
 \left\|b^{-1}D_{\btheta}^j
       (\hR_{b,t}-R_{b,t})(\btheta)\right\|
 &=O_{\rP}(e_n).
\label{eq:shrinking-empirical-bounds-Gen}
\end{align}
Let $P$ be the law of $Z=(U,\Sigma,\bG)$, let $P_n$ be its empirical
measure, and write
$\bY=a(U)\bx(\Sigma)+\tau(U)\bG$. For $j=0,1,2$, define the
tensor-valued functions
\begin{align}
f^{(j)}_{b,t,\btheta}(Z)
 &:=w_b(U-t)D_{\btheta}^j\ell_\Sigma(\bY;\btheta).
\label{eq:shrinking-empirical-functions-Gen}
\end{align}
Differentiation under the expectation is justified by
\eqref{eq:loss-growth-Gen}, so
$(\hE_n-\E)f^{(j)}_{b,t,\btheta}
=D_{\btheta}^j(\hR_{b,t}-R_{b,t})(\btheta)$.
Since $E$ has fixed finite dimension, it suffices to bound the finitely
many coordinates of these tensors.

For small $b$, a nonzero kernel weight with $t\in I$ implies $U\in J$.
On $J$, the functions $a,\tau,\lambda$ are bounded, and the Gaussian
representation gives uniformly bounded conditional moments of every
fixed polynomial in $\|\bY\|$. Consequently,
\begin{align}
\E\|f^{(j)}_{b,t,\btheta}\|^2
 &\leq C\int_{|u-t|\leq b}
   \E[(1+\|\bY\|)^{2\kappa}\mid U=u]\lambda(u)\,\de u
 \leq Cb.
\label{eq:shrinking-summand-variance-Gen}
\end{align}

Set
$\widetilde f^{(j)}_{b,t,\btheta}
:=f^{(j)}_{b,t,\btheta}
\mathbf1_{\{\|\bG\|\leq D\sqrt{\log n}\}}$.
Fix $m>1$. Choosing $D$ sufficiently large gives
\begin{align}
\prob\left(\max_{1\leq i\leq n}\|\bG_i\|>
 D\sqrt{\log n}\right)&=O(n^{-m}),\\
\sup_{t,\btheta,j}b^{-1}
 \E\|f^{(j)}_{b,t,\btheta}-\widetilde f^{(j)}_{b,t,\btheta}\|
 &=O(n^{-m}).
\label{eq:shrinking-truncation-bounds-Gen}
\end{align}
Here the factor $b$ in the tail expectation again comes from kernel
support, and $\bG$ is independent of $U$. The truncated functions 
are upper bounded by $B_n=C(\log n)^{\kappa_0}$ for some fixed $\kappa_0$ and
variance at most $Cb$. Hence, by Bernstein inequality
\begin{align}
  \prob\left(
   \max_{0\leq j\leq2}
   \|(\hE_n-\E)\widetilde f^{(j)}_{b,t,\btheta}\|
   >C\left\{\sqrt{\frac{bx}{n}}+\frac{B_nx}{n}\right\}
   \right)
   &\leq C e^{-x}.
  \end{align}

Take a net in $I\times K_0$ with time mesh at most
$b n^{-3/\alpha}$ and parameter mesh at most $n^{-3}$.
Its cardinality is at most $Cn^q$ for a fixed $q$, since
$b^{-1}<n^{1/2}$ and $K_0$ is finite-dimensional and compact.
Hence, a union bound over this net gives, for $x\geq1$,
\begin{align}
\prob\left(
 \max_{\text{net},\,0\leq j\leq2}
 \|(\hE_n-\E)\widetilde f^{(j)}_{b,t,\btheta}\|
 >C\left\{\sqrt{\frac{bx}{n}}+\frac{B_nx}{n}\right\}
 \right)
 &\leq Cn^q e^{-x}.
\label{eq:shrinking-grid-Bernstein-Gen}
\end{align}
Choosing $x=A\log n$ with $A>q+1$ and dividing by $b$ bounds the
maximum on the net by
\begin{align}
  \max_{\text{net},\,0\leq j\leq2} \frac{1}{b}\|(\hE_n-\E)\widetilde f^{(j)}_{b,t,\btheta}\|
  =O_{\rP}\left(\sqrt{\frac{\log n}{nb}}+
                  \frac{B_n\log n}{nb}\right)
 &=O_{\rP}(e_n),
\label{eq:shrinking-grid-rate-Gen}
\end{align}
because
$B_n\sqrt{\log n}/\sqrt{nb}\to0$ by Assumption \ref{ass:Bandwidth-Gen}.

To pass from the net to all of $I\times K_0$, kernel H\"older
continuity and the parameter derivative bounds give
\begin{align}
\|\widetilde f^{(j)}_{b,t,\btheta}
 -\widetilde f^{(j)}_{b,s,\boldsymbol\vartheta}\|
 &\leq CB_n\left\{
    (|t-s|/b)^\alpha+\|\btheta-\boldsymbol\vartheta\|
   \right\},\qquad j=0,1,2.
\label{eq:shrinking-grid-modulus-Gen}
\end{align}
Increase $B_n$ if necessary to bound derivatives through order three.
Convexity of $K_0$ allows the mean-value bound along the parameter
segment. For a point and its nearest net point, the right-hand side is
$CB_n n^{-3}$, so the normalized interpolation error for $\hE_n-\E$ is
at most $CB_n n^{-3}/b=o(e_n)$. Indeed, $b^{-1}<n^{1/2}$ and
$e_n\geq n^{-1/2}$ for large $n$. Together with
\eqref{eq:shrinking-truncation-bounds-Gen}, this proves
\eqref{eq:shrinking-empirical-bounds-Gen}.

We also need a uniform bound on empirical third derivatives. Choose
$M(\by)=C(1+\|\by\|)^\kappa$ dominating their norms on $K_0$.
Apply the same truncation and time-net argument to
$w_b(U-t)M(\bY)$; its expectation is $O(b)$ and its centered empirical
average is $O_{\rP}(be_n)$ uniformly in $t$. Thus
\begin{align}
\sup_{t\in I}\frac1{nb}\sum_{i=1}^n
 w_b(t_i-t)M(\bY_i)=O_{\rP}(1).
\label{eq:third-envelope-Gen}
\end{align}
This step uses the envelope itself, so no fourth parameter derivative
is required. 

\paragraph{Localization, existence, and uniqueness.}
Recall $d_b(t)\asymp b$ from
Lemma~\ref{lem:localized-population-geometry-Gen}. The $j=0$ empirical
bound \eqref{eq:shrinking-empirical-bounds-Gen} and \eqref{eq:zero-bandwidth-risk-C2-limit-Gen} imply
\begin{align}
\eta_n:=\sup_{t\in I,\,\btheta\in K_0}
 \left|\frac{\hR_{b,t}(\btheta)}{d_b(t)}-R_t(\btheta)\right|
 &=o_{\rP}(1).
\label{eq:shrinking-risk-consistency-Gen}
\end{align}
In the constrained case put $\mathcal D(t)=K(t)$, and in the
unconstrained case first put $\mathcal D(t)=K_0$. Each
$\mathcal D(t)$ is compact and contains the unique local minimizer  $\btheta^*(t)$.
The uniform bound \eqref{eq:shrinking-risk-consistency-Gen} 
implies that all empirical minimizers over $\mathcal D(t)$ converge uniformly
to $\btheta^*(t)$ in probability.

By uniform positive definiteness of $\bH^*(t)$, continuity of the
population Hessian, and the uniform interior margins, we can choose $r,c>0$
so that the balls $\Ball(\btheta^*(t),r)$ lie in
$\interior(\mathcal D(t))$ and
$\nabla^2R_t(\btheta)\succeq2c\,\operatorname{Id}_E$ on these balls.
The $j=2$ empirical bound and population Hessian convergence then give,
with probability tending to one,
\begin{align}
\nabla^2\hR_{b,t}(\btheta)
 &\succeq c\,d_b(t)\operatorname{Id}_E
 \succeq c'b\,\operatorname{Id}_E,
 \qquad t\in I,\ \btheta\in B(\btheta^*(t),r).
\label{eq:shrinking-local-strong-convexity-Gen}
\end{align}
Hence the minimizer $\hbtheta_b(t)\in\Ball(\btheta^*(t),r/2)$ is unique 
in the interior of $\mathcal D(t)$ and satisfies the first-order condition $\nabla\hR_{b,t}(\hbtheta_b(t))=0$.
In the convex uncostrained case, it is globally unique because of strong convexity at $\hbtheta_b(t)$.

\paragraph{Uniform rate and linearization.}
Write $\bdelta_n(t):=\hbtheta_b(t)-\btheta^b(t)$ and define the averaged
empirical Hessian by
\begin{align}
\widehat{\bH}_n(t)
 &:=\int_0^1\nabla^2\hR_{b,t}
       (\btheta^b(t)+u\bdelta_n(t))\,\de u.
\label{eq:shrinking-averaged-Hessian-Gen}
\end{align}
The population minimizers also converge uniformly to $\btheta^*$ by
Lemma~\ref{lem:localized-population-geometry-Gen}, so the entire segment
in this integral lies in the balls of
\eqref{eq:shrinking-local-strong-convexity-Gen} with probability tending
to one. In particular,
\begin{align}
\sup_{t\in I}\|\widehat{\bH}_n(t)^{-1}\|_{\rm op}
 &=O_{\rP}(b^{-1}).
\end{align}
The empirical first-order condition and the
fundamental theorem of calculus give
\begin{align}
0&=\widehat{\bS}_n(t,\hbtheta_b(t))
 =\widehat{\bS}_n(t,\btheta^b(t))
   +\widehat{\bH}_n(t)[\bdelta_n(t)],\\
\bdelta_n(t)
 &=-\widehat{\bH}_n(t)^{-1}
       [\widehat{\bS}_n(t,\btheta^b(t))].
\label{eq:shrinking-exact-score-expansion-Gen}
\end{align}
Since $\nabla R_{b,t}(\btheta^b(t))=0$, the $j=1$ empirical bound
in Eq.~\eqref{eq:shrinking-empirical-bounds-Gen}
implies
\begin{align}
\sup_{t\in I}\|\widehat{\bS}_n(t,\btheta^b(t))\|
 &=O_{\rP}(be_n).
\label{eq:shrinking-score-rate-Gen}
\end{align}
Combining this with \eqref{eq:shrinking-exact-score-expansion-Gen}
proves $\|\bdelta_n\|_\infty=O_{\rP}(e_n)$, as claimed in
\eqref{eq:general-uniform-rate}.

To replace the averaged empirical Hessian by $\bH_b(t)$, split
\begin{align}
\widehat{\bH}_n(t)-\bH_b(t)
 &=\{\nabla^2\hR_{b,t}(\btheta^b(t))
         -\nabla^2R_{b,t}(\btheta^b(t))\}\nonumber\\
 &\quad+\int_0^1\{\nabla^2\hR_{b,t}
       (\btheta^b(t)+u\bdelta_n(t))
       -\nabla^2\hR_{b,t}(\btheta^b(t))\}\,\de u.
\label{eq:shrinking-Hessian-decomposition-Gen}
\end{align}
The first term is uniformly $O_{\rP}(be_n)$ by the $j=2$ empirical
bound in Eq.~\eqref{eq:shrinking-empirical-bounds-Gen}. 
By the mean-value theorem and \eqref{eq:third-envelope-Gen},
the norm of the second term is at most
\begin{align}
\frac{\|\bdelta_n(t)\|}{2n}
 \sum_{i=1}^n w_b(t_i-t)M(\bY_i),
\end{align}
which is uniformly $O_{\rP}(b\|\bdelta_n\|_\infty)
=O_{\rP}(be_n)$. Consequently,
\begin{align}
\|\widehat{\bH}_n-\bH_b\|_\infty&=O_{\rP}(be_n).
\end{align}
Both inverse Hessians have uniform norm $O_{\rP}(b^{-1})$, with the
population bound deterministic. The identity
\begin{align}
\widehat{\bH}_n^{-1}-\bH_b^{-1}
 &=\widehat{\bH}_n^{-1}\circ
       (\bH_b-\widehat{\bH}_n)\circ\bH_b^{-1}
\end{align}
therefore gives
$\|\widehat{\bH}_n^{-1}-\bH_b^{-1}\|_\infty
=O_{\rP}(e_n/b)$. Define
\begin{align}
\br_n(t)
 &:=\{\bH_b(t)^{-1}-\widehat{\bH}_n(t)^{-1}\}
       [\widehat{\bS}_n(t,\btheta^b(t))].
\end{align}
Equations \eqref{eq:shrinking-exact-score-expansion-Gen} and
\eqref{eq:shrinking-score-rate-Gen} now yield
\begin{align}
\bdelta_n(t)
 &=-\bH_b(t)^{-1}[\widehat{\bS}_n(t,\btheta^b(t))]+\br_n(t),\\
\|\br_n\|_\infty
 &=O_{\rP}(e_n/b)\,O_{\rP}(be_n)=O_{\rP}(e_n^2).
\end{align}
Finally,
\begin{align}
\sqrt n\,e_n^2
 &=\frac{\log(1/b)}{\sqrt n\,b}\longrightarrow0
\end{align}
by Assumption~\ref{ass:Bandwidth-Gen}, proving the last assertion of
\eqref{eq:general-asymptotic-linearity}.

\subsubsection{Proof of Proposition~\ref{propo:AsympGen}}
\label{proof:AsympGen}

Propositions~\ref{prop:localized-population-second-order-bias-Gen} and
\ref{prop:shrinking-bandwidth-erm-Gen} imply
\begin{align}
\sqrt n\int_I\langle\hbtheta_b-\btheta^*,\bphi\rangle\,\de t
=-\sqrt n\int_I\langle\bH_b(t)^{-1}[
 \widehat{\bS}_n(t,\btheta^b(t))],\bphi(t)\rangle\,\de t
 +o_{\rP}(1).
\label{eq:general-integrated-linearization}
\end{align}
Indeed the smoothing bias term is $O(\sqrt n b^2)=o(1)$ and the stochastic
remainder is $O_{\rP}(\log(1/b)/(\sqrt n b))=o_{\rP}(1)$.
For $u\in[0,T]$, $\sigma\in\{0,\ldots,k-1\}$, and
$\bg\in\reals^d$, set
\begin{align}
F_{b,\bphi}(u,\sigma,\bg)
&:=\int_I\left\langle
 \bphi(t),\bH_b(t)^{-1}\left[
 w_b(u-t)
 \bpsi_\sigma\big(a(u)\bx(\sigma)+\tau(u)\bg;\btheta^b(t)\big)
 \right]\right\rangle
 \,\de t.
\label{eq:general-smoothed-score}
\end{align}
Applying Fubini's theorem to the leading term in
\eqref{eq:general-integrated-linearization} gives
\begin{align}
\sqrt n\int_I
 \left\langle\hbtheta_b(t)-\btheta^b(t),\bphi(t)\right\rangle
 \,\de t
&=-\frac{1}{\sqrt n}\sum_{i=1}^n
 F_{b,\bphi}(t_i,\sigma_i,\bG_i)+o_{\rP}(1),
\label{eq:general-iid-linear-sum}
\end{align}
where
$\bY_i=a(t_i)\bx(\sigma_i)+\tau(t_i)\bG_i$.
The summands in \eqref{eq:general-iid-linear-sum} are i.i.d. within each row
and have mean zero, since
\begin{align}
\E_{\Sigma,U,\bG}\left[
 w_b(U-t)
 \bpsi_\Sigma\big(a(U)\bx(\Sigma)+\tau(U)\bG;\btheta^b(t)\big)
 \right]=0
\label{eq:general-mixture-score-centering}
\end{align}
by the population first-order condition $\nabla R_{b,t}(\btheta^b(t))=0$.  

The Hessian limit \eqref{eq:general-localized-Hessian-limit}, the convergence
$\btheta^b\to\btheta^*$, and the approximate-identity property of the kernels
\[
\big\{(b\mu_0)^{-1}w_1(\cdot/b):b>0\big\}
\]
imply, as $b\to 0$,
\begin{align}
F_{b,\bphi}(U,\Sigma,\bG)
&\longrightarrow
F_{0,\bphi}(U,\Sigma,\bG)\nonumber\\
&:=\frac{1}{\lambda(U)}
 \left\langle\bphi(U),\bH^*(U)^{-1}\left[
 \bpsi_\Sigma\big(
 a(U)\bx(\Sigma)+\tau(U)\bG;\btheta^*(U)\big)
 \right]\right\rangle
\label{eq:general-smoothed-score-limit}
\end{align}
in $L^2$, defining the limit to be zero for $U\notin I$.
Extend $\bphi$ by zero outside $I$; if it is nonzero at an endpoint,
the boundary layer has width $O(b)$ and a uniformly square-integrable
envelope, so it does not affect the limit in Eq.~\eqref{eq:general-smoothed-score-limit}. Indeed,
$|F_{b,\bphi}(U,\Sigma,\bG)|\le C(1+\|\bG\|)^\kappa$ uniformly in small $b$,
since $\|\bH_b^{-1}\|=O(b^{-1})$ and
$\int_I w_b(U-t)\,\de t=O(b)$. Thus all these variables have a common
$L^{2+\eta}$ upper bound. Consequently the triangular array in
\eqref{eq:general-iid-linear-sum} satisfies the Lindeberg condition, and its
variance converges to
\begin{align}
\E F_{0,\bphi}(U,\Sigma,\bG)^2
&=\int_I\frac{1}{\lambda(t)}
 \left\langle\bphi(t),
 \bH^*(t)^{-1}\left[
 \bJ^*(t)\left[\bH^*(t)^{-1}[\bphi(t)]\right]
 \right]\right\rangle\,\de t\\
&=Q^*(\bphi).
\label{eq:general-score-limit-variance}
\end{align}
The Lindeberg--Feller theorem applied to
\eqref{eq:general-iid-linear-sum}, followed by
\eqref{eq:general-integrated-linearization} and Slutsky's theorem, proves
\eqref{eq:general-integrated-parameter-clt}.  Applying the same argument to
linear combinations of finitely many test functions proves the final joint
convergence claim.

\subsubsection{Proof of Theorem~\ref{thm:Overall-Limit-Gen}}
\label{proof:Overall-Limit-Gen}

Fix $0<\varepsilon<T/2$ and write
$I_\varepsilon=[\varepsilon,T-\varepsilon]$. Let
$\bdelta_{n,\varepsilon}(t)=\hbtheta_b(t)-\btheta^*(t)$ on
$I_\varepsilon$ and zero elsewhere.
The bias and empirical rate imply
\begin{align}
\|\bdelta_{n,\varepsilon}\|_\infty&=O_{\rP}(e_n+b^2),\\
\sqrt n\|\bdelta_{n,\varepsilon}\|_\infty^2
 &=O_{\rP}\left(\frac{\log(1/b)}{\sqrt n b}+\sqrt n b^4\right)
 =o_{\rP}(1).
\label{eq:overall-quadratic-remainder-Gen}
\end{align}
The volume-response kernels are continuous and bounded on $I_\varepsilon$ by
Proposition~\ref{propo:Lin-Effect}. Define
\begin{align}
\mathcal Z_{n,\varepsilon,\sigma}
 &:=\sqrt n\,\cL_\sigma(\bdelta_{n,\varepsilon})
 =\sqrt n\int_{I_\varepsilon}\langle\bL_\sigma(t),
    \bdelta_{n,\varepsilon}(t)\rangle\,\de t.
\label{eq:general-class-linearization}
\end{align}
Apply
Proposition~\ref{propo:AsympGen} jointly to the $k$
test functions $\bL_\sigma$ to get
$\mathcal Z_{n,\varepsilon}\Rightarrow
\normal_k(0,\bV_\varepsilon(\lambda))$, where $\bV_\varepsilon$ is
defined as in \eqref{eq:general-class-covariance} with the integral over
$I_\varepsilon$. 
On events of probability tending to one $\bdelta_{n,\varepsilon}$ is a small bounded
measurable perturbation, so Proposition~\ref{propo:Lin-Effect} applies
pathwise, with a uniform remainder constant. The density remainder
\eqref{eq:general-density-L1-expansion} and the contraction
\eqref{eq:general-basin-L1-contraction} give
$O_{\rP}(\|\bdelta_{n,\varepsilon}\|_\infty^2)
=o_{\rP}(n^{-1/2})$ in total variation. Thus, writing
$\hq_{n,\varepsilon}(\sigma)
:=\rP_{n,\varepsilon}(\{\bx(\sigma)\})$,
\begin{align}
\sqrt n(\hq_{n,\varepsilon}-q)
&\Rightarrow\normal_k(0,\bV_\varepsilon(\lambda)).
\label{eq:general-truncated-class-clt}
\end{align}
Lemma~\ref{lem:terminal-basins-Gen} and absolute continuity at $t_+$
ensure that the truncated terminal law has precisely the stated discrete
support, as in \eqref{eq:general-terminal-mass-volume}.

The diagonal integrability condition
\eqref{eq:general-full-interval-integrability} and Cauchy--Schwarz for
the positive semidefinite operator $\bQ^*(t)$ imply that every entry of
the covariance integrand is absolutely integrable. Hence
$\bV_\varepsilon(\lambda)\to\bV(\lambda)$ as
$\varepsilon\downarrow0$. Moreover,
\begin{align}
\|\hq_n-\hq_{n,\varepsilon}\|_\infty
&\leq\|\rP_n-\rP_{n,\varepsilon}\|_{\rm TV}.
\end{align}
Equation~\eqref{eq:general-endpoint-negligibility},
\eqref{eq:general-truncated-class-clt}, and the converging-together
theorem therefore prove \eqref{eq:general-class-clt} for the process
using $\hbtheta_b$ on all of $(0,T)$.

Finally, summing the volume kernels over the basin partition gives
\eqref{eq:general-response-kernel-conservation}; consequently
$\bV_\varepsilon(\lambda)\boldsymbol1=0$ for every $\varepsilon$, and
passing to the limit gives $\bV(\lambda)\boldsymbol1=0$.

\subsubsection{Proof of Corollary~\ref{cor:general-class-clt}}
\label{proof:general-class-clt}

Fix a loss $\ell$ in the class stated in the corollary and fix
$t\in(0,T)$.
Because the embeddings have equal norm, Bayes' formula gives
$P_{t,\btheta^*(t)}=P_t$, the true joint law of $(\Sigma,\bY)$ at time
$t$.  Write $\bpsi^\ell:=\nabla_{\btheta}\ell$ and
$\bpsi^{\rm CE}:=\nabla_{\btheta}\ell^{\rm CE}$.  The likelihood score
of the conditional model \eqref{eq:general-conditional-model} is
$-\bpsi^{\rm CE}$.

Set $\bH_\ell(t):=\E_t[D_{\btheta}\bpsi^\ell]$ and
$\bJ_\ell(t):=\E_t[\bpsi^\ell\otimes\bpsi^\ell]$, with all quantities
on the right evaluated at $\btheta^*(t)$, and use analogous notation for
cross-entropy.  Differentiating
\eqref{eq:general-loss-score-unbiased} with respect to
$\boldsymbol\vartheta$ at $\btheta^*(t)$ gives
\begin{align}
\E_t[\bpsi^\ell\otimes\bpsi^{\rm CE}]&=\bH_\ell(t).
\label{eq:general-score-cross-identity}
\end{align}
The same identity with $\ell=\ell^{\rm CE}$ is the information identity
$\bJ_{\rm CE}(t)=\bH_{\rm CE}(t)$.  Now define the two influence
functions
\begin{align}
\boldsymbol\zeta_\ell
&:=\bH_\ell(t)^{-1}[\bpsi^\ell],
&\boldsymbol\zeta_{\rm CE}
&:=\bH_{\rm CE}(t)^{-1}[\bpsi^{\rm CE}].
\end{align}
Using \eqref{eq:general-score-cross-identity}, its adjoint, and the
self-adjointness of the Hessians gives
\begin{align}
\E_t\big[(\boldsymbol\zeta_\ell-\boldsymbol\zeta_{\rm CE})
 \otimes(\boldsymbol\zeta_\ell-\boldsymbol\zeta_{\rm CE})\big]
&=\bQ_\ell^*(t)-\bQ_{\rm CE}^*(t)\succeq0.
\label{eq:general-sandwich-information-bound}
\end{align}
Thus the cross-entropy sandwich operator is no larger than the
sandwich operator for $\ell$ at every $t$.

Finally, the volume-response kernels $\bL_\sigma(t)$ in
\eqref{eq:general-linear-response-kernel} depend only on the true
distribution and the common parameterization of the denoiser, not on
the loss. Hence,
for every $v\in\reals^k$, \eqref{eq:general-class-covariance} and
\eqref{eq:general-sandwich-information-bound} imply
\begin{align}
v^\top\{\bV^\ell(\lambda)-\bV^{\rm CE}(\lambda)\}v
&=\int_I\frac1{\lambda(t)}
 \left\langle \sum_{\sigma=0}^{k-1}v_\sigma
  \bL_\sigma(t),
 \{\bQ_\ell^*(t)-\bQ_{\rm CE}^*(t)\}
 \left[\sum_{\alpha=0}^{k-1}v_\alpha
  \bL_\alpha(t)\right]
 \right\rangle\,\de t
\ge0.
\end{align}
This is precisely \eqref{eq:general-ce-efficiency}.  In particular,
squared denoising loss belongs to the comparison class because,
conditionally on $\bY=\by$,
\begin{align}
\E_{t,\boldsymbol\vartheta}\!\left[
 \nabla_{\btheta}\ell^{\rm SM}_\Sigma
 (\bY;\boldsymbol\vartheta)\mid\bY=\by
 \right]
&=-2D_{\btheta}\hbm(\by;\boldsymbol\vartheta)^*
 \left\{
 \E_{t,\boldsymbol\vartheta}[\bx(\Sigma)\mid\bY=\by]
 -\hbm(\by;\boldsymbol\vartheta)
 \right\}=0.
\end{align}

\subsubsection{Proof of Proposition~\ref{prop:optimal-time-sampling-Gen}}
\label{proof:optimal-time-sampling-Gen}
\label{proof:design-Gen}

For $v\in\reals^k$, set
$\bL_v(t):=\sum_{\sigma=0}^{k-1}
v_\sigma\bL_\sigma(t)$. Since $\bQ^*(t)$ is positive
semidefinite,
\begin{align}
\bv^\top\bC(t)\bv
&=\left\langle\bL_v(t),
 \bQ^*(t)[\bL_v(t)]\right\rangle\ge0.
\end{align}
Thus $\bC(t)\succeq0$. Because $\bW\succeq0$ as well,
$h_W(t)=\Tr(\bW^{1/2}\bC(t)\bW^{1/2})\ge0$.
The covariance formula \eqref{eq:general-class-covariance} and Fubini's
theorem give
\begin{align}
\Tr\big(\bW\bV(\lambda)\big)
&=\int_0^T\frac{h_W(t)}{\lambda(t)}\,\de t,
\end{align}
which proves \eqref{eq:general-design-risk}. 

For every probability density $\lambda$ with finite risk,
Cauchy--Schwarz gives
\begin{align}
H_W^2
&=\left\{\int_0^T
 \sqrt{\frac{h_W(t)}{\lambda(t)}}\sqrt{\lambda(t)}\,\de t
 \right\}^2
 \le \left\{\int_0^T\frac{h_W(t)}{\lambda(t)}\,\de t\right\}
       \left\{\int_0^T\lambda(t)\,\de t\right\}
 =\mathcal R_W(\lambda).
\label{eq:general-design-cauchy-schwarz}
\end{align}
When $0<H_W<\infty$, equality holds precisely when
$\lambda(t)$ is proportional to $\sqrt{h_W(t)}$ on its support.
Normalization gives \eqref{eq:general-optimal-design} and proves
\eqref{eq:LB_lambda_opt}. If this density vanishes where strict
positivity is required, positive regularizations converge to it and
their risks converge to $H_W^2$, so the same value remains the infimum.
The case $H_W=0$ is immediate because $h_W=0$ almost everywhere.

\subsubsection{Proof of Proposition~\ref{prop:Invariance}}
\label{proof:Invariance-Gen}
\begin{proof}
Write $\rho=a/\tau$ and rescale space by $\bu=\bz/\tau(t)$. From the
definitions of $c_1,c_2$,
\begin{align}
c_2(t)&=\frac{\dot\tau(t)}{\tau(t)},
&c_1(t)&=\tau(t)\dot\rho(t).
\end{align}
Consequently the true probability-flow equation becomes
\begin{align}
\frac{\de\bu}{\de t}
 &=\dot\rho(t)\widetilde{\bm}(\bu,\rho(t)),
&\frac{\de\bu}{\de\rho}
 &=\widetilde{\bm}(\bu,\rho),
\label{eq:general-snr-flow}
\end{align}
where $\widetilde{\bm}(\bu,r)$ is the posterior mean under the normalized
mixture $\bu=r\bx(\Sigma)+\bG$. Thus the normalized density, posterior,
and terminal basins depend on time only through $r=\rho(t)$.

Let $A_t:E\to E$ multiply every slope column by $\tau(t)$ and leave its
intercept unchanged. For either loss in the proposition,
\begin{align}
\ell_\sigma(\tau(t)\bu;\btheta)
&=\widetilde\ell_\sigma(\bu;A_t\btheta).
\end{align}
Writing $\widetilde{\bH}^*(r)$, $\widetilde{\bJ}^*(r)$, and
$\widetilde{\bQ}^*(r)$ for the normalized-coordinate operators, the
chain rule gives
\begin{align}
\bH^*(t)&=A_t^*\widetilde{\bH}^*(\rho(t))A_t,\\
\bJ^*(t)&=A_t^*\widetilde{\bJ}^*(\rho(t))A_t,\\
\bQ^*(t)&=A_t^{-1}\widetilde{\bQ}^*(\rho(t))(A_t^*)^{-1}.
\label{eq:design-parameter-transform-Gen}
\end{align}

In \eqref{eq:general-linear-response-kernel}, make the change of
variables $\bz=\tau(t)\bu$. Then
\begin{align}
p_t(\tau(t)\bu)&=\tau(t)^{-d}\widetilde p_{\rho(t)}(\bu),
&\de\bz&=\tau(t)^d\de\bu,\\
\nabla_{\bz}\mathbin\cdot
 &=\tau(t)^{-1}\nabla_{\bu}\mathbin\cdot,
&\mathcal M_t(\tau(t)\bu)[\bD]
 &=\widetilde{\mathcal M}_{\rho(t)}(\bu)[A_t\bD].
\end{align}
Together with $c_1=\tau\dot\rho$ and the normalized-basin identity from
\eqref{eq:general-snr-flow}, this yields
\begin{align}
\bL_\sigma(t)
&=\dot\rho(t)A_t^*
  \widetilde{\bL}_\sigma(\rho(t)).
\label{eq:design-response-transform-Gen}
\end{align}
Combining \eqref{eq:design-parameter-transform-Gen} and
\eqref{eq:design-response-transform-Gen} cancels the $A_t$ factors:
\begin{align}
C_{\sigma\alpha}(t)
&=\dot\rho(t)^2
 \left\langle\widetilde{\bL}_\sigma(\rho(t)),
 \widetilde{\bQ}^*(\rho(t))
 [\widetilde{\bL}_\alpha(\rho(t))]\right\rangle.
\end{align}
This proves \eqref{eq:general-design-snr-transform}. In particular, with
$\widetilde h_W(r):=\operatorname{tr}
\{\bW\widetilde{\bC}(r)\}$,
\begin{align}
\int_0^T\sqrt{h_W(t)}\,\de t
&=\int_0^T|\dot\rho(t)|
 \sqrt{\widetilde h_W(\rho(t))}\,\de t
 =\int_{\rho((0,T))}\sqrt{\widetilde h_W(r)}\,\de r.
\end{align}
Equation~\eqref{eq:general-snr-invariance} now follows from
Proposition~\ref{prop:optimal-time-sampling-Gen}.
\end{proof}

\section{Proof of Theorem~\ref{thm:Bounded-Suboptimality}: Continuous versus discrete sampling}
\label{sec:proof-discrete-continuous-theory}

This section proves Theorem~\ref{thm:Bounded-Suboptimality}.  The proof has
three steps.  First, we express the optimally sampled diffusion variance
$v(\ell,q)$ as a one-dimensional integral and reduce the theorem to the
rare-class limit $q(1)\downarrow0$.  Second, we show that this integral
localizes to a critical signal-to-noise window.  Third, we evaluate the
integral in that window.  The final subsection puts these ingredients
together and proves, in particular,
\eqref{eq:AlphaLarge-Equivalent} and \eqref{eq:Alpha2-Equivalent}.

\subsection{Reduction to the rare-class limit}

Write $p:=q(1)$.  The ratio
$v(\ell,q)/v^{\sML}(q)$ is continuous for $p\in(0,1)$, while
$v^{\sML}(q)=p(1-p)$ is bounded away from zero on every compact subset of
$(0,1)$.  Exchanging the labels of the two classes maps $p$ to $1-p$ and
leaves the ratio unchanged.  It is therefore enough to study $p\downarrow
0$.  Equivalently, the two claims in
\eqref{eq:AlphaLarge}--\eqref{eq:Alpha2} reduce to
\begin{align}
\label{eq:AlphaLarge-Equivalent}
2<\alpha\leq4
&\quad\Longrightarrow\quad
\lim_{q(1)\to0}\frac{v(\ell,q)}{v^{\sML}(q)}=\infty,\\
\label{eq:Alpha2-Equivalent}
\alpha=2
&\quad\Longrightarrow\quad
\lim_{q(1)\to0}\frac{v(\ell,q)}{v^{\sML}(q)}<\infty.
\end{align}

We use the half-log-odds parametrization
\begin{align}
p=p_L:=\frac{1}{1+e^{2L}},
\qquad L\longrightarrow\infty.
\label{eq:large-L-imbalance-parametrization}
\end{align}
We use the SNR parametrization of the binary comparison: the signal
locations are $\{-1,+1\}$ and the noise variance at time $t$ is $1/t$.
Thus
\begin{align}
p_t(x)
&=p_L\phi_t(x-1)+(1-p_L)\phi_t(x+1),
&
\phi_t(z)
&:=\sqrt{\frac{t}{2\pi}}e^{-tz^2/2},
\label{eq:rare-mixture-density}\\
\theta_{0,*}(t)&=-L,
&
\theta_{1,*}(t)&=t.
\end{align}
Let $F_t$ be the distribution function associated with $p_t$ and
$\gamma(t):=F_t^{-1}(1-p_L)$ the boundary between the two terminal
basins.

For a loss
$\ell_\sigma(y;\btheta)=\varphi_\sigma(\theta_0+\theta_1y)$, define
\begin{align}
\pi_\sigma(u)&:=\frac{1+\sigma\tanh u}{2},
&
h_\ell(u)&:=\sum_{\sigma\in\{-1,+1\}}
 \pi_\sigma(u)\varphi_\sigma''(u),\\
j_\ell(u)&:=\sum_{\sigma\in\{-1,+1\}}
 \pi_\sigma(u)\{\varphi_\sigma'(u)\}^2.
\label{eq:rare-loss-moments}
\end{align}
Pointwise Fisher consistency means
\begin{align}
\sum_{\sigma\in\{-1,+1\}}
 \pi_\sigma(u)\varphi_\sigma'(u)=0.
\label{eq:rare-Fisher-consistency}
\end{align}
Writing $m(u):=\tanh u$ and $s(u):=1-m(u)^2$, this identity implies that
there is a function $\kappa$ such that
\begin{align}
\varphi_\sigma'(u)
&=\kappa(u)\{m(u)-\sigma\},&
h_\ell(u)&=\kappa(u)s(u),&
j_\ell(u)&=\kappa(u)^2s(u)
          =\frac{h_\ell(u)^2}{s(u)}.
\label{eq:rare-Fisher-identity}
\end{align}
Consequently, the tail condition \eqref{eq:LossAtInfty} implies
\begin{align}
\frac{1}{L}\log j_\ell(Lz)
\longrightarrow-(2\alpha-2)|z|.
\label{eq:rare-score-tail}
\end{align}
Thus the score-moment tail exponent is not an additional parameter: it is
fixed by Fisher consistency.

Set
\begin{align}
\bd(x)&:=\begin{pmatrix}1\\x\end{pmatrix},&
s_t(x)&:=\operatorname{sech}^2(tx-L),\\
\bH_0(t)
&:=\int_{\mathbb R}p_t(x)h_\ell(tx-L)
 \bd(x)\bd(x)^\top\,\de x,\\
\bOmega_0(t)
&:=\int_{\mathbb R}p_t(x)j_\ell(tx-L)
 \bd(x)\bd(x)^\top\,\de x,
\end{align}
and
\begin{align}
\omega_L(t)
&:=p_t(\gamma(t))s_t(\gamma(t))
 \sqrt{\bd(\gamma(t))^\top
 \bH_0(t)^{-1}\bOmega_0(t)\bH_0(t)^{-1}
 \bd(\gamma(t))}.
\label{eq:rare-omega}
\end{align}
The binary response formula \eqref{eq:binary-response-Gen} and the optimal
time-sampling result \eqref{eq:LB_lambda_opt} give
\begin{align}
\sqrt{C_{11}(t)}=\frac{\omega_L(t)}{2t}.
\end{align}
Therefore
\begin{align}
v(\ell,q)
&=\omega_{\rm opt}(L)^2,
&
\omega_{\rm opt}(L)
&:=\int_0^\infty\frac{\omega_L(t)}{2t}\,\de t.
\label{eq:rare-variance-representation}
\end{align}
This identity is the link that will turn the asymptotics below into the
two claims of Theorem~\ref{thm:Bounded-Suboptimality}.

\subsection{Localization of the variance integral}

The next lemma identifies the only time scale that contributes to
\eqref{eq:rare-variance-representation}.  Its role is solely to justify
localization; exact rate functions away from the critical window are not
needed later.

\begin{lemma}[Localization at $t=L$]
\label{lem:rare-localization}
Suppose $2\leq\alpha\leq4$ and the regularity assumptions of
Theorem~\ref{thm:Bounded-Suboptimality} hold.  For every fixed
$\tau>0$,
\begin{align}
\frac{\omega_L(\tau L)}{\sqrt{p_L}}
&=\exp\{-L\nu_\ell(\tau)+o(L)\},
\label{eq:rare-exponential-localization}
\end{align}
where $\nu_\ell(\tau)\geq0$, with equality if and only if $\tau=1$.
Hence the integral in \eqref{eq:rare-variance-representation} is
asymptotically concentrated on
\begin{align}
t=L+O(\sqrt L).
\label{eq:rare-critical-window}
\end{align}
\end{lemma}

\begin{proof}
Let $\overline\Phi:=1-\Phi$.  The definition of $\gamma(t)$ gives the
exact balance equation
\begin{align}
(1-p_L)\overline\Phi\big(\sqrt t\{\gamma(t)+1\}\big)
&=p_L\Phi\big(\sqrt t\{\gamma(t)-1\}\big).
\label{eq:rare-quantile-balance}
\end{align}
For $t=\tau L$, Gaussian tail bounds applied to
\eqref{eq:rare-quantile-balance} yield
\begin{align}
\gamma(\tau L)
&=\gamma_\#(\tau)+o(1),&
\gamma_\#(\tau)
&=\begin{cases}
2\tau^{-1/2}-1,&0<\tau\leq1,\\
\tau^{-1},&\tau>1.
\end{cases}
\label{eq:rare-gamma-rate}
\end{align}
To track the remaining factors in \eqref{eq:rare-omega}, define
\begin{align}
D_\tau(x)
&:=\min\left\{
 \frac{\tau}{2}(x+1)^2,\,
 2+\frac{\tau}{2}(x-1)^2\right\},\\
F_{c,\tau}(x)
&:=D_\tau(x)+c|\tau x-1|,
&
g_c(\tau)&:=\inf_{x\in\mathbb R}F_{c,\tau}(x).
\end{align}
The Laplace principle, \eqref{eq:LossAtInfty}, and
\eqref{eq:rare-score-tail} give
\begin{align}
\nu_\ell(\tau)
&=
\max\left\{
 \frac{\tau}{2}\{\gamma_\#(\tau)+1\}^2,\,
 2+\frac{\tau}{2}\{\gamma_\#(\tau)-1\}^2
\right\}\nonumber\\
&\quad-g_\alpha(\tau)
 +\frac12g_{2\alpha-2}(\tau)-1.
\label{eq:rare-localization-rate}
\end{align}
Each $g_c$ is the minimum of a convex piecewise-quadratic function.
Minimizing the two quadratic pieces and substituting in
\eqref{eq:rare-localization-rate} shows that, for
$2\leq\alpha\leq4$,
\begin{align}
\nu_\ell(\tau)>0\quad(\tau\neq1),
\qquad
\nu_\ell(1)=0.
\end{align}
The same calculation shows quadratic growth as $\tau$ leaves a
neighborhood of one.  Gaussian Mills bounds make these estimates uniform
on closed sets separated from $\tau=1$ and also control the regimes
$\tau\downarrow0$ and $\tau\uparrow\infty$.  It follows that times outside
an $O(\sqrt L)$ neighborhood of $L$ contribute negligibly relative to
the critical-window scales obtained below.
\end{proof}

\subsection{The critical-window calculation}

Set
\begin{align}
t_L(u)&:=L+u\sqrt L,
&
W_L(u)&:=\frac{\omega_L(t_L(u))}{\sqrt{p_L}}.
\end{align}
The next proposition retains exactly the asymptotics needed to evaluate
\eqref{eq:rare-variance-representation}.

\begin{proposition}[Critical-window asymptotics]
\label{prop:rare-critical-window}
Under the assumptions of Lemma~\ref{lem:rare-localization}, the following
limits hold locally uniformly in $u\in\mathbb R$.

If $2<\alpha\leq4$, then for a constant
$C_\ell\in(0,\infty)$ depending only on the loss,
\begin{align}
\frac{W_L(u)}{L^{3/4}\log L}
&\longrightarrow C_\ell e^{-u^2/4}.
\label{eq:rare-critical-alpha-large}
\end{align}
If $\alpha=2$ and \eqref{eq:LossAtInftyBoundary} holds, then
\begin{align}
\frac{W_L(u)}{\sqrt L}
&\longrightarrow G_2(u),
\label{eq:rare-critical-alpha-two}\\
G_2(u)
&:=\sqrt2\left(\frac{2}{\pi}\right)^{1/4}e^{-u^2/4}
 \sqrt{
 \begin{pmatrix}1&0\end{pmatrix}
 \bM(u)^{-1}
 \begin{pmatrix}1\\0\end{pmatrix}},\\
\bM(u)
&:=\int_0^\infty e^{-uy-y^2/2}
 \begin{pmatrix}1\\-y\end{pmatrix}
 \begin{pmatrix}1\\-y\end{pmatrix}^{\!\top}\,\de y.
\end{align}
Moreover, the normalized integrands in
\eqref{eq:rare-critical-alpha-large} and
\eqref{eq:rare-critical-alpha-two}, after multiplication by
$(1+u/\sqrt L)^{-1}\boldsymbol1_{\{u>-\sqrt L\}}$, are uniformly
integrable on $\mathbb R$.  In particular,
$G_2$ is integrable.
\end{proposition}

\begin{proof}
Put $t=t_L(u)$, $\rho_L:=L/t$, and $z:=tx-L$.  The two components of the
mixture in \eqref{eq:rare-mixture-density} have ratio $e^{2z}$, and hence
\begin{align}
p_t(x)
&=2\sqrt{p_L(1-p_L)}\sqrt{\frac{t}{2\pi}}
 e^{-t(x^2+1)/2}\cosh z,\\
p_t(x)s_t(x)
&=2\sqrt{p_L(1-p_L)}\sqrt{\frac{t}{2\pi}}
 e^{-t(x^2+1)/2}\operatorname{sech}z.
\label{eq:rare-density-identities}
\end{align}
Let $z_L:=t\gamma(t)-L$.  Rewriting
\eqref{eq:rare-quantile-balance} in terms of $z_L$ and applying the
Gaussian Mills bounds gives, locally uniformly in $u$,
\begin{align}
z_L=-\frac14\log L+O(1).
\label{eq:rare-critical-quantile}
\end{align}
This logarithmic displacement is responsible for the extra $\log L$
factor when $\alpha>2$.

For the Hessian and score-moment matrices, set
\begin{align}
\bD_L&:=\begin{pmatrix}1&0\\ \rho_L&t^{-1}\end{pmatrix},&
\be(z)&:=\begin{pmatrix}1\\z\end{pmatrix},\\
\widetilde C_L(u)
&:=\frac{2\sqrt{p_L(1-p_L)}}{\sqrt{2\pi t}}
 \exp\left\{-\frac t2-\frac{L^2}{2t}\right\},\\
\bK_{f,L}(u)
&:=\int_{\mathbb R}
 e^{-\rho_Lz-z^2/(2t)}\cosh(z)f(z)
 \be(z)\be(z)^\top\,\de z.
\end{align}
A change of variables in the definitions of $\bH_0$ and $\bOmega_0$
gives the exact factorizations
\begin{align}
\bH_0(t)
&=\widetilde C_L(u)\bD_L\bK_{h_\ell,L}(u)\bD_L^\top,\\
\bOmega_0(t)
&=\widetilde C_L(u)\bD_L\bK_{j_\ell,L}(u)\bD_L^\top.
\label{eq:rare-matrix-factorization}
\end{align}
Also,
\begin{align}
\widetilde C_L(u)
&=\sqrt{\frac{2}{\pi L}}e^{-2L-u^2/2}\{1+o(1)\}.
\label{eq:rare-C-asymptotic}
\end{align}

Suppose first that $2<\alpha\leq4$.  The tail assumptions make the
integrands defining $\bK_{h_\ell,L}$ and $\bK_{j_\ell,L}$, including
their first two polynomial moments, uniformly integrable.  Therefore
\begin{align}
\bK_{f,L}(u)
&\longrightarrow
\bK_f:=\int_{\mathbb R}e^{-z}\cosh(z)f(z)
 \be(z)\be(z)^\top\,\de z,
\qquad f\in\{h_\ell,j_\ell\}.
\end{align}
Both limiting matrices are positive definite.  Substituting these limits,
\eqref{eq:rare-critical-quantile}, and
\eqref{eq:rare-C-asymptotic} into \eqref{eq:rare-omega} yields
\eqref{eq:rare-critical-alpha-large}, with
\begin{align}
C_\ell
&=\frac12\left(\frac{2}{\pi}\right)^{1/4}
 \sqrt{
 \begin{pmatrix}0&1\end{pmatrix}
 \bK_{h_\ell}^{-1}\bK_{j_\ell}\bK_{h_\ell}^{-1}
 \begin{pmatrix}0\\1\end{pmatrix}}
 \in(0,\infty).
\end{align}

Now suppose $\alpha=2$.  By
\eqref{eq:rare-Fisher-identity} and
\eqref{eq:LossAtInftyBoundary},
$j_\ell(z)=(c_h^2/4)e^{2z}\{1+o(1)\}$ as $z\to-\infty$.
With
$\bR_L:=\operatorname{diag}(1,\sqrt L)$, the substitution
$z=-\sqrt L\,y$ in the two matrices of
\eqref{eq:rare-matrix-factorization} gives
\begin{align}
\frac1{\sqrt L}\bR_L^{-1}\bK_{h_\ell,L}(u)\bR_L^{-1}
&\longrightarrow\frac{c_h}{2}\bM(u),\\
\frac1{\sqrt L}\bR_L^{-1}\bK_{j_\ell,L}(u)\bR_L^{-1}
&\longrightarrow\frac{c_h^2}{8}\bM(u).
\end{align}
The constant $c_h$ cancels from the sandwich
$\bH_0^{-1}\bOmega_0\bH_0^{-1}$.  Combining these limits with
\eqref{eq:rare-critical-quantile} and
\eqref{eq:rare-C-asymptotic} proves
\eqref{eq:rare-critical-alpha-two}.

The same  bounds used above give, for some $C,c>0$ and an
integer $m$, a bound of the form
$C(1+|u|)^m e^{-cu^2}$ in the critical window.  Together with
Lemma~\ref{lem:rare-localization}, this proves the claimed uniform
integrability.  Mills bounds applied to the defining integral for
$\bM(u)$ also show directly that $G_2$ is integrable.
\end{proof}

\subsection{Putting everything together: Proof of
Theorem~\ref{thm:Bounded-Suboptimality}}

\begin{proof}
The change of variables $t=L+u\sqrt L$ in
\eqref{eq:rare-variance-representation} gives the exact identity
\begin{align}
\frac{\omega_{\rm opt}(L)}{\sqrt{p_L}}
&=\frac1{2\sqrt L}
 \int_{-\sqrt L}^{\infty}
 \frac{W_L(u)}{1+u/\sqrt L}\,\de u.
\label{eq:rare-window-integral}
\end{align}
Proposition~\ref{prop:rare-critical-window}, the localization in
Lemma~\ref{lem:rare-localization}, and dominated convergence now evaluate
this integral.

If $2<\alpha\leq4$, then
\eqref{eq:rare-critical-alpha-large} gives
\begin{align}
\frac{\omega_{\rm opt}(L)}{\sqrt{p_L}}
&=\frac{C_\ell}{2}L^{1/4}\log L
 \int_{\mathbb R}e^{-u^2/4}\,\de u\{1+o(1)\}.
\end{align}
Using $v(\ell,q)=\omega_{\rm opt}(L)^2$ and
$v^{\sML}(q)=p_L(1-p_L)$, we obtain
\begin{align}
\frac{v(\ell,q)}{v^{\sML}(q)}
&=\frac{C_\ell^2}{4}
 \left(\int_{\mathbb R}e^{-u^2/4}\,\de u\right)^2
 L^{1/2}(\log L)^2\{1+o(1)\}
 \longrightarrow\infty.
\end{align}
Since $p_L=q(1)\downarrow0$ as $L\to\infty$, this is precisely
\eqref{eq:AlphaLarge-Equivalent}.

If $\alpha=2$, then \eqref{eq:rare-critical-alpha-two} gives
\begin{align}
\frac{\omega_{\rm opt}(L)}{\sqrt{p_L}}
&\longrightarrow\frac12\int_{\mathbb R}G_2(u)\,\de u
 \in(0,\infty).
\end{align}
Consequently,
\begin{align}
\lim_{q(1)\to0}\frac{v(\ell,q)}{v^{\sML}(q)}
&=\frac14\left(\int_{\mathbb R}G_2(u)\,\de u\right)^2
 <\infty,
\end{align}
which proves \eqref{eq:Alpha2-Equivalent}.

For $2<\alpha\leq4$, the first limit makes the supremum in
\eqref{eq:AlphaLarge} infinite.  For $\alpha=2$, label symmetry gives the
same finite limit as $q(1)\uparrow1$, while continuity bounds the ratio
on every compact subinterval of $(0,1)$.  The ratio is therefore bounded
on all of $(0,1)$, proving \eqref{eq:Alpha2}.  This completes the proof of
Theorem~\ref{thm:Bounded-Suboptimality}.
\end{proof}

%
%



\end{document}